\documentclass[twoside]{article}

\usepackage[accepted]{aistats2022}

\usepackage{algorithmic}
\usepackage{verbatim}
\usepackage{amsmath,amsthm,amsfonts,amssymb}
\usepackage{amsmath}
\usepackage[dvipsnames]{xcolor,colortbl}
\definecolor{pearThree}{HTML}{E74C3C}
\definecolor{pearcomp}{HTML}{B97E29}
\definecolor{pearDark}{HTML}{2980B9}
\definecolor{pearDarker}{HTML}{1D2DEC}
\usepackage{hyperref}
\hypersetup{
	colorlinks,
	citecolor=pearDark,
	linkcolor=pearThree,
	breaklinks=true,
	urlcolor=black}
\usepackage{cleveref}

\usepackage{mathrsfs}
\usepackage{enumitem}
\usepackage{bm}
\usepackage{multirow}
\usepackage{tikz}
\usetikzlibrary{positioning,angles,quotes}
\usepackage{booktabs}
\usepackage{array, makecell} 
\usepackage{graphicx}
\usepackage{subfigure}
\usepackage{caption}
\usepackage{thmtools}
\usepackage{thm-restate}
\usepackage{dsfont}
\usepackage{xspace}
\usepackage{empheq}
\usepackage{pifont}
\usepackage{bbding}
\usepackage{wrapfig}

\usepackage{wasysym}
\usepackage{cancel}

\usepackage{scalerel}

\include{notations}
\usepackage[title]{appendix}
\allowdisplaybreaks[4]

\newtheorem{theorem}{Theorem}
\newtheorem{lemma}[theorem]{Lemma}
\newtheorem{corollary}[theorem]{Corollary}

\newtheorem{definition}[theorem]{Definition}

\newtheorem{assumption}[theorem]{Assumption}

\theoremstyle{definition}
\newtheorem{remark}{Remark}

\renewcommand{\P}{\mathbb{P}}

\newcommand{\E}{\mathcal{E}}
\newcommand{\D}{\mathcal{D}}
\newcommand{\Q}{\mathcal{Q}}

\usepackage[utf8]{inputenc}
\usepackage[T1]{fontenc}    
\usepackage{hyperref}       
\usepackage{url}            
\usepackage{booktabs}       
\usepackage{amsfonts}       
\usepackage{nicefrac}       
\usepackage{microtype}      

\usepackage{standalone}

\usepackage{amsmath,amsthm}
\usepackage{mathtools}
\usepackage[linesnumbered,ruled,vlined]{algorithm2e}
\usepackage{amssymb}
\usepackage{bbm}
\usepackage{bm}
\usepackage{enumitem}
\usepackage{dirtytalk}
\usepackage{pifont}
\usepackage{tikz}
\usetikzlibrary{arrows.meta}
\usepackage{float}
\usepackage{cleveref}
\usepackage{hyperref}
\usepackage{soul}
\usepackage{dsfont}
\usepackage{amssymb}
\usepackage[thinc]{esdiff}
\usepackage{fontawesome}

\definecolor{vert1}{RGB}{133,146,66} 
\definecolor{vert3}{RGB}{91,140,90} 
\definecolor{vert2}{RGB}{157,193,7} 
\definecolor{vert4}{RGB}{20,200,20} 

\newcommand{\SL}{\mathcal{G}_L}
\newcommand{\SLepsilon}{\mathcal{G}_{L+\epsilon}}
\newcommand{\SLepsilonk}{\mathcal{G}_{L+\epsilon_k}}

\newcommand{\SOL}{\mathcal{G}_{O(L)}}

\newcommand{\SLarr}{\mathcal{S}_L^{\rightarrow}}

\newcommand{\areset}{a_{\text{\textup{reset}}}}

\newcommand{\UCRLVTR}{{\small\textsc{UCRL-VTR}}\xspace}
\newcommand{\UCRLVTRp}{{\small\textsc{UCRL-VTR$^+$}}\xspace}

\newcommand\myeqii{\mathrel{\stackrel{\makebox[0pt]{\mbox{\normalfont\tiny (ii)}}}{\,=\,}}}

\newcommand\myineqi{\mathrel{\stackrel{\makebox[0pt]{\mbox{\normalfont\tiny (i)}}}{\,\leq\,}}}
\newcommand\myineqii{\mathrel{\stackrel{\makebox[0pt]{\mbox{\normalfont\tiny (ii)}}}{\,\leq\,}}}
\newcommand\myineqiii{\mathrel{\stackrel{\makebox[0pt]{\mbox{\normalfont\tiny (iii)}}}{\,\leq\,}}}

\DeclareMathOperator*{\argmax}{arg\,max}
\DeclareMathOperator*{\argmin}{arg\,min}
\DeclareMathOperator*{\clip}{\mathrm{clip}}

\newcommand{\Xk}{\mathcal{X}_k}

\newcommand{\cG}{\mathcal{G}}
\newcommand{\cS}{\mathcal{S}}
\newcommand{\cA}{\mathcal{A}}

\newcommand{\cU}{\mathcal{U}}

\makeatletter
\newcommand\footnoteref[1]{\protected@xdef\@thefnmark{\ref{#1}}\@footnotemark}
\makeatother

\newcommand*\widefbox[1]{\fbox{\hspace{2em}#1\hspace{2em}}}

\newcommand{\cC}{\mathcal{C}}

\newcommand{\cE}{\mathcal{E}}

\newcommand{\cJ}{\mathcal{J}}

\newcommand{\cM}{\mathcal{M}}

\newcommand{\cQ}{\mathcal{Q}}

\newcommand{\cX}{\mathcal{X}}
\newcommand{\X}{\cX}

\newcommand{\bcM}{\overline{\cM}}
\newcommand{\bV}{\overline{V}}

\newcommand{\bp}{\overline{p}}
\newcommand{\bq}{\overline{q}}
\newcommand{\bpi}{\overline{\pi}}
\newcommand{\bpistar}{\overline{\pi}^\star}
\newcommand{\bVstar}{\overline{V}^\star}

\newcommand{\ovU}{\ov{U}}
\newcommand{\ovV}{\ov{V}}

\newcommand{\tpi}{\widetilde{\pi}}

\newcommand{\tq}{\widetilde{q}}

\newcommand{\wt}[1]{\widetilde{#1}}

\newcommand{\wh}[1]{\widehat{#1}}
\newcommand{\ov}[1]{\overline{#1}}

\DeclarePairedDelimiter\abs{\lvert}{\rvert}%

\definecolor{LightCyan}{rgb}{0.7,1,1}
\definecolor{LightRed}{rgb}{1,0.7,0.7}

\let\originalleft\left
\let\originalright\right
\renewcommand{\left}{\mathopen{}\mathclose\bgroup\originalleft}
\renewcommand{\right}{\aftergroup\egroup\originalright}

\usepackage{tcolorbox}
\newtcolorbox{redbox}{colback=red!5!white,colframe=red!75!black}
\newtcolorbox{bluebox}{colback=blue!5!white,colframe=blue!75!black}
\newtcolorbox{yellowbox}{colback=yellow!5!white,colframe=yellow!75!black}

\newcount\Comments  
\definecolor{darkgreen}{rgb}{0,0.5,0}
\definecolor{darkred}{rgb}{0.7,0,0}
\definecolor{teal}{rgb}{0.3,0.8,0.8}

\newcommand{\red}[1]{\noindent{\textcolor{red}{#1}}}

\newcommand{\tab}[1]{\noindent{\textcolor{blue}{#1}}}
\newcommand{\lm}[1]{\noindent{\textcolor{purple}{#1}}}

\usetikzlibrary{arrows,automata}
\usetikzlibrary{positioning}
\usetikzlibrary{matrix, positioning, arrows.meta}
\newlength{\myheight}
\tikzset{labels/.style={font=\sffamily\scriptsize},
    circuit/.style={draw,minimum width=2cm,minimum height=\myheight,very thick,inner sep=1mm,outer sep=0pt,cap=round,font=\sffamily\bfseries}
}

\makeatletter
\def\@fnsymbol#1{\ensuremath{\ifcase#1\or \dagger\or \ddagger\or
  \mathsection\or \mathparagraph\or \|\or \diamond \or **\or \dagger\dagger
  \or \ddagger\ddagger \else\@ctrerr\fi}}
\makeatother

\makeatletter
\newcommand{\printfnsymbol}[1]{%
  \textsuperscript{\@fnsymbol{#1}}%
}
\makeatother

\usepackage{tabularx}
\usepackage{footnote}
\usepackage{hhline}
\usepackage{multirow}

\definecolor{HighlightColor}{gray}{0.97}

\definecolor{colorm}{rgb}{0.2, 0.2, 0.6}
\definecolor{colorn}{rgb}{0.77, 0.12, 0.23}

\newcommand{\PreserveBackslash}[1]{\let\temp=\\#1\let\\=\temp}
\newcolumntype{C}[1]{>{\PreserveBackslash\centering}p{#1}}

\newcommand{\hp}{\widehat{p}}

\newcommand{\betastar}{\beta^\star}

\newcommand{\Vstar}{V^\star}

\usepackage{accents}
\newcommand{\utQ}{\widetilde{\ov{Q}}}
\newcommand{\ltQ}{\underaccent{\tilde}{\ov{Q}}}
\newcommand{\utV}{\widetilde{\ov{V}}}
\newcommand{\ltV}{\underaccent{\tilde}{\ov{V}}}

\newcommand{\Var}{\mathrm{Var}}

\newcommand{\rQ}{\mathring{\ov{Q}}}
\newcommand{\rV}{\mathring{\ov{V}}}

\newcommand{\AdaGoalnormal}{{\textsc{AdaGoal}}\xspace}

\newcommand{\AdaGoal}{{\small\textsc{AdaGoal}}\xspace}
\newcommand{\UniGoal}{{\small\textsc{UniGoal}}\xspace}
\newcommand{\RareGoal}{{\small\textsc{RareGoal}}\xspace}

\newcommand{\uni}{{\small\textsc{uni}}\xspace}
\newcommand{\rare}{{\small\textsc{rare}}\xspace}

\newcommand{\AXstar}{{\small\textsc{AX$^{\star}$}}\xspace}

\newcommand{\MGE}{{\small\textsc{MGE}}\xspace}

\newcommand{\ALGO}{{\small\textsc{AdaGoal-UCBVI}}\xspace}

\newcommand{\TAE}{{\small\textsc{TAE}}\xspace}
\newcommand{\RFE}{{\small\textsc{RFE}}\xspace}

\newcommand{\UniGoalUCBVI}{{\small\textsc{UniGoal-UCBVI}}\xspace}
\newcommand{\RareGoalUCBVI}{{\small\textsc{RareGoal-UCBVI}}\xspace}

\newcommand{\VDS}{{\small\textsc{VDS}}\xspace}
\newcommand{\HER}{{\small\textsc{HER}}\xspace}

\newcommand{\ALGOscs}{{\scriptsize\textsc{AdaGoal-UCBVI}}\xspace}
\newcommand{\ALGOfns}{{\footnotesize\textsc{AdaGoal-UCBVI}}\xspace}
\newcommand{\UniGoalUCBVIfns}{{\footnotesize\textsc{UniGoal-UCBVI}}\xspace}
\newcommand{\RareGoalUCBVIfns}{{\footnotesize\textsc{RareGoal-UCBVI}}\xspace}
\newcommand{\AdaGoalfns}{{\footnotesize\textsc{AdaGoal}}\xspace}

\newcommand{\RareGoalfns}{{\footnotesize\textsc{RareGoal}}\xspace}

\newcommand{\ALGOLM}{{\small\textsc{AdaGoal-UCRL$\cdot$VTR}}\xspace}

\newcommand{\ALGOLMnormal}{{\textsc{AdaGoal-UCRL$\cdot$VTR}}\xspace}

\newcommand{\ALGOnormal}{{\textsc{AdaGoal-UCBVI}}\xspace}

\newcommand{\BPIUCBVI}{{\small\textsc{BPI-UCBVI}}\xspace}

\newcommand{\UcbExplore}{{\small\textsc{UcbExplore}}\xspace}
\newcommand{\DisCo}{{\small\textsc{DisCo}}\xspace}
\newcommand{\GOSPRL}{{\small\textsc{GOSPRL}}\xspace}

\newcommand{\N}{\mathbb{N}}

\newcommand{\cnt}{\mathrm{cnt}}

\DeclareMathOperator{\KL}{KL}

\newcommand{\CommaBin}{\mathbin{\raisebox{0.5ex}{,}}}

\newcommand{\bn}{\bar{n}}
\renewcommand{\bp}{\bar{p}}
\renewcommand{\epsilon}{\varepsilon}

\newcommand{\btheta}{\bm{\theta}}

\newcommand{\bphi}{\bm{\phi}}

\newcommand{\bpsi}{\bm{\psi}}

\newcommand{\bSigma}{\bm{\Sigma}}

\newcommand{\bbb}{\bm{b}}
\newcommand{\zero}{\bm{0}}

\newcommand{\vvalue}{V}
\newcommand{\qvalue}{Q}

\newcommand{\la}{\langle}
\newcommand{\ra}{\rangle}

\newcommand{\Ib}{\bm{I}}

\newcommand{\pnorm}{B}

\usepackage[round]{natbib}

\usepackage{minitoc}
\begin{document}

%

%
\runningauthor{Tarbouriech, Domingues, Ménard, Pirotta, Valko, Lazaric}

\twocolumn[

\aistatstitle{Adaptive Multi-Goal Exploration}

\aistatsauthor{ Jean Tarbouriech \And Omar Darwiche Domingues \And  Pierre Ménard}

\aistatsaddress{ Meta AI  \& Inria Scool \And  Inria Scool \And OvGU Magdeburg }

\vspace{-0.1in}

\aistatsauthor{ Matteo Pirotta \And Michal Valko \And  Alessandro Lazaric }

\aistatsaddress{ Meta AI \And DeepMind \And Meta AI }

\vspace{0.1in}
]

\begin{abstract} We introduce a generic strategy for provably efficient multi-goal exploration. It relies on \AdaGoal, a novel goal selection scheme that leverages a measure of uncertainty in reaching states to adaptively target goals that are neither too difficult nor too easy. We show how \AdaGoal can be used to tackle the objective of learning an $\epsilon$-optimal goal-conditioned policy for the (initially unknown) set of goal states that are reachable within $L$ steps in expectation from a reference state $s_0$ in a reward-free Markov decision process. In the tabular case with $S$ states and $A$ actions, our algorithm requires $\widetilde{O}(L^3 S A \epsilon^{-2})$ exploration steps, which is nearly minimax optimal. We also readily instantiate \AdaGoal in linear mixture Markov decision processes, yielding the first goal-oriented PAC guarantee with linear function approximation. Beyond its strong theoretical guarantees, we anchor \AdaGoal in goal-conditioned deep reinforcement learning, both conceptually and empirically, by connecting its idea of selecting ``uncertain'' goals to maximizing value ensemble disagreement. 
\end{abstract}

\doparttoc 
\faketableofcontents 

\renewcommand{\red}[1]{\noindent{\textcolor{black}{#1}}}

\section{INTRODUCTION} \label{sect_intro}

\newcommand{\GS}{\textbf{(GS)}\xspace}
\newcommand{\PE}{\textbf{(PE)}\xspace}

When the extrinsic reward signal is absent or non-informative, a reinforcement learning (RL) agent needs to explore the environment driven by objectives other than reward maximization \citep[see e.g.,][]{chentanez2005intrinsically,oudeyer2009intrinsic, singh2010intrinsically}. This general setting is often called \textit{unsupervised} RL, self-supervised RL or intrinsically motivated RL. A noteworthy instance of it is unsupervised \textit{goal-conditioned} RL (GC-RL, \citealp[see e.g.,][for a survey]{colas2020intrinsically}). In this framework, the agent must learn a goal-conditioned policy, which learns a distribution over actions conditioned not only on the current state but also on a goal state that it must reach as quickly as possible (in expectation). For the goal-conditioned policy to be able to reach a variety of goals in the unknown environment, the agent must autonomously set its own goals (via e.g., a curriculum) and learn to effectively reach them. 

\paragraph{Deep GC-RL.} Recently, GC-RL has been extensively studied in the context of deep RL~\citep[e.g.,][]{schaul2015universal, andrychowicz2017hindsight, florensa2018automatic, warde2018unsupervised, nair2018visual, colas2019curious, zhao2019maximum, hartikainen2019dynamical, ecoffet2020first, pong2020skew,zhang2020automatic,pitis2020maximum}. GC-RL has notably been shown to be a powerful framework to tackle navigation problems \citep{florensa2018automatic}, game playing \citep[][on Montezuma's Revenge]{ecoffet2020first} or real-world robotic manipulation tasks \citep{pong2020skew}. At a high-level, these deep GC-RL methods follow the same algorithmic structure that alternates between:
\begin{enumerate}[leftmargin=.5in,topsep=-1pt,itemsep=2pt,partopsep=0pt, parsep=0pt]
    \item[\GS] \textit{Goal Selection:} select a candidate goal state;
    \item[\PE] \textit{Policy Execution:} deploy an explorative policy conditioned on this goal for a fixed number of steps.
\end{enumerate}
Given this simple and unifying algorithmic structure, the core differences between the methods lie in the goal sampling distribution --- e.g., uniform \citep{kaelbling1993learning, schaul2015universal, andrychowicz2017hindsight}, proportional to rarity \citep{pong2020skew}, or targeting goals of increasing complexity \citep{florensa2018automatic, colas2019curious,zhang2020automatic} --- and in ways to take advantage of each policy execution as much as possible to speed up the learning --- e.g., goal relabeling \citep{andrychowicz2017hindsight} or encoding goal states in lower-dimensional representations \citep{pong2020skew}. The specific choice of \GS and \PE steps directly influence the learning speed as well as the quality of the goal-conditioned policy returned by the algorithm.

While these GC-RL approaches effectively leverage deep RL techniques and are able to achieve impressive results in complex domains, they often lack substantial theoretical understanding and guarantees, even when restricted to the tabular setting.

\paragraph{Theory of GC-RL.} On the other hand, GC-RL has been rather sparsely analyzed through a theoretical lens. Notably, \citet{lim2012autonomous, tarbouriech2020improved} study the \textit{incremental autonomous exploration} setting, where the agent is restricted to identifying and learning to efficiently reach only the goal states that are \textit{incrementally} attainable from a reference initial state. Meanwhile, \citet{tarbouriech2021provably} study the problem of \textit{goal-free cost-free exploration} (i.e., finding a near-optimal goal-reaching policy for any starting state, goal state and cost function) under the assumption that the Markov decision process (MDP) is \textit{communicating} (i.e., any state is reachable from any other state).

While the existing theoretical GC-RL approaches come with rigorous guarantees on the learning performance, their algorithmic designs are quite involved, far from the interpretable alternation of \GS and \PE steps (see Appendix \ref{app_examples} for details, where we identify three specific shortcomings). This makes it hard to extract relevant practical insights and to make them amenable to a deep RL implementation.

\paragraph{Objective.} 
In light of these conceptual and algorithmic discrepancies between the two branches of theoretical and deep GC-RL, we seek~to: 
\begin{center}
    \textit{Design a strategy for unsupervised goal-conditioned RL with both \\ \textbf{1)}~a simple and interpretable algorithmic structure whose conceptual idea can be adapted to deep RL, 
    and \textbf{2)}~strong learning guarantees under suitable assumptions on the environment (e.g., tabular, linear).}   
\end{center}

\paragraph{Contributions.} Our main contributions can be summarized as follows:
\begin{itemize}[leftmargin=.15in,topsep=-1.2pt,itemsep=1.4pt,partopsep=0pt, parsep=0pt]
        \item We formalize the \textit{multi-goal exploration} (\MGE) objective of minimizing the number of exploration steps (i.e., the sample complexity) required to learn a near-optimal goal-conditioned policy for all the goal states that are reachable within a given number of steps in expectation from the initial state. We formally motivate the need of an available reset action to the initial state so as to solve the objective in a reasonable number of exploration steps, and we derive a lower bound on the sample complexity.
        \item We introduce \AdaGoal, a novel \textit{goal selection scheme} for unsupervised goal-conditioned RL. It relies on a simple optimization problem that adaptively targets goal states of intermediate difficulty. It also provides an algorithmic \textit{stopping rule} and a \textit{set of candidate goal states} that the agent is confident it can reliably reach. We identify some generic conditions that make \AdaGoal theoretically sound. 
        \item We design \ALGO, an instantiation of \AdaGoal in tabular MDPs, and prove that it achieves nearly minimax optimal sample complexity, improving over existing bounds that only hold in special cases.
        \item Leveraging a similar algorithmic and proof structure, we readily design \ALGOLM for linear-mixture MDPs. This yields the first goal-oriented PAC\footnote{Recall that the probably approximately correct (PAC) learning setting provides sample complexity guarantees to find a near-optimal policy at the fixed initial state.} guarantee with linear function approximation.  
        \item We anchor \AdaGoal in deep GC-RL, both conceptually and empirically, by connecting its idea of selecting ``uncertain'' goals to a practical approximation of maximizing value ensemble disagreement \citep{zhang2020automatic}.
\end{itemize}

\subsection{Additional Related Work} 

\paragraph{Theory of unsupervised exploration.}  A recent line of work theoretically analyzes unsupervised exploration in RL (in other settings than GC-RL). Some approaches \citep[e.g.,][]{hazan2019provably, tarbouriech2019active, cheung2019regret, tarbouriech2020active} learn to induce a desired state(-action) distribution (e.g., maximally entropic), yet they do not analyze how this can be used to solve downstream tasks. In finite-horizon MDPs, \citet{jin2020reward} propose a paradigm composed of: an \textit{exploration phase}, where the agent interacts with the environment without the supervision of reward signals; followed by a \textit{planning phase}, where the agent is 
required to compute a near-optimal policy for some given reward function. 
If the reward function can be chosen arbitrarily (including adversarially), the problem is called \textit{reward-free exploration} (\RFE) \citep{jin2020reward, kaufmann2021adaptive, menard2021fast,zanette2020provably, wang2020reward, chen2021near,zhang2021reward,zhang2020nearly}. If there is only a finite number of rewards that are fixed a priori yet unknown during exploration, 
it is called \textit{task-agnostic exploration} (\TAE) \citep{zhang2020task,wu2020accommodating, wu2021gap}. Our \MGE problem bears some resemblance to \TAE in the sense that we also consider a finite number of unknown tasks to solve (specifically, finding a near-optimal goal-conditioned policy for the unknown set of reliably reachable goal states). However, \TAE (as well as \RFE) is limited to the \textit{finite-horizon} setting and does not extend to goal-reaching tasks, which belong to the category of stochastic shortest path problems.

\paragraph{Single-goal exploration (a.k.a.\,SSP).}  A special case of multi-goal exploration is when there is a single goal state that is extrinsically fixed throughout learning (i.e., there is no \GS). This is known as the stochastic shortest path (SSP) problem \citep{bertsekas1995dynamic}. A recent line of work has studied online learning in SSP in the context of regret minimization \citep{tarbouriech2019no,rosenberg2020near,cohen2021minimax,tarbouriech2021stochastic,jafarnia2021online,chen2021implicit,vial2021regret,min2021learning}. However, as argued by \citet{tarbouriech2021sample}, a standard regret-to-PAC conversion does not hold in general in SSP, which implies that these techniques cannot be directly applied to our case. See Section~\ref{section_analysis} for further discussion.

\section{MULTI-GOAL EXPLORATION}  \label{section2_MGE}

The agent interacts with an environment modeled by a Markov decision process (MDP) $\mathcal{M}$ that has no extrinsic reward (i.e., it is \textit{reward-free}), a finite state space $\mathcal{S}$ with $S \triangleq \abs{\mathcal{S}}$ states and a finite action space $\mathcal{A}$ with $A \triangleq \abs{\mathcal{A}}$ actions. Let $s_0 \in \mathcal{S}$ be a designated initial state. We denote by $p(s' \vert s,a)$ the probability transition from state~$s$ to state~$s'$ by taking action~$a$. 

A deterministic stationary policy $\pi: \mathcal{S} \rightarrow \mathcal{A}$ is a mapping between states to actions and we denote by~$\Pi$ the set of all possible policies. 
We measure the performance of a policy in navigating the MDP and define the shortest-path distance as follows.

\begin{definition} \label{def_VSSP}
For any policy $\pi \in \Pi$ and a pair of states $(s, s') \in \mathcal{S}^2$, let $V^{\pi}(s \rightarrow s') \in [0, + \infty]$ be the expected number of steps it takes to reach $s'$ starting from $s$ when executing policy $\pi$, i.e.,\footnote{Note that $V^{\pi}(s \rightarrow s')$ corresponds to the value function (a.k.a.\,expected cost-to-go) of policy $\pi$ in a stochastic shortest-path setting (SSP,~\citealp{bertsekas1995dynamic}) with initial state $s$, goal state $s'$ and unit cost function. If the policy $\pi$ reaches $s'$ from $s$ with probability 1, it is said to be \textit{proper}, otherwise it holds that $V^{\pi}(s \rightarrow s') = + \infty$.} 
\begin{align*}
    V^{\pi}(s \rightarrow s') &\triangleq \mathbb{E}\Big[ \omega_{\pi}(s \rightarrow s') \Big], \\
    \omega_{\pi}(s \rightarrow s') &\triangleq \inf \Big\{ i \geq 0: s_{i+1} = s' \,\big\vert\, s_1 = s, \pi \Big\},
\end{align*}
where the expectation is w.r.t.\,the random sequence of states generated by executing $\pi$ in $\mathcal{M}$ starting from state~$s$. For any state $g \in \cS$, let $V^{\star}(s_0 \rightarrow g) \in [0, + \infty]$ be the shortest-path distance from $s_0$ to $g$, i.e., $$V^{\star}(s_0 \rightarrow g) \triangleq \min_{\pi \in \Pi} V^{\pi}(s_0 \rightarrow g).$$ 
Let $D_0 \in [0, + \infty]$, $D \in [0, + \infty]$ be the MDP's (possibly infinite) $s_0$-diameter and diameter, respectively, i.e., 
\begin{align*}
    D_0 \triangleq \max_{g \in \cS} V^{\star}(s_0 \rightarrow g), \qquad D \triangleq \max_{s, g} V^{\star}(s \rightarrow g).
\end{align*}
\end{definition}

We denote by $\cG \subseteq \cS$ the \textit{goal space}, which corresponds to the set of goal states that the agent may condition its goal-conditioned policy on (in the absence of prior knowledge on the goal space we simply set $\cG = \cS$). In environments with arbitrary dynamics, there may be some goal states in $\cG$ that are \textit{too} difficult for the agent to reliably reach in a reasonable number of exploration steps, or even completely unreachable from $s_0$. Consequently, we consider the high-level objective of \textit{learning an accurate goal-conditioned policy for all the goal states that are reliably reachable} from $s_0$. 

\begin{definition}[Reliably $L$-reachable goal states $\SL$] 
For any threshold $L \geq 1$, we define a goal state~$g \in \cG$ to be reliably $L$-reachable if \mbox{$V^{\star}(s_0 \rightarrow g) \leq L$}, and we denote by $\SL$ the set of such goal states, i.e.,
    \begin{align*}
        \SL \triangleq \Big\{ g \in \cG : V^{\star}(s_0 \rightarrow g) \leq L \Big\}.
    \end{align*}
\end{definition}

We thus seek to learn a goal-conditioned policy that is accurate in reaching the goals in $\SL$. A challenge in solving this objective comes from the fact that the set of goals of interest $\SL$ is initially \textit{unknown} and it has to be discovered online at the same time as learning their corresponding optimal policy. The threshold~$L$ can be interpreted as the user's exploration radius of interest around~$s_0$. In the absence of a pre-specified threshold, the agent can build its own curriculum for~$L$ to guide its learning process.  

Since in environments with arbitrary dynamics the agent may get stuck in a state without being able to return to~$s_0$, we introduce the following ``reset'' assumption.\footnote{This setting should be contrasted with the finite-horizon setting, where each policy resets automatically after $H$ steps, or assumptions on the MDP dynamics such as ergodicity or bounded diameter, which guarantee that it is always possible to find a policy navigating between any two states.} In Lemma~\ref{lemma_exp_sep} we will formally motivate its role in solving our learning objective.

\begin{assumption}
	The action space contains a known action $\areset \in \cA$ such that $p(s_0| s, \areset) = 1$ for any state $s \in \mathcal{S}$.
	\label{assumption_reset}
\end{assumption}

Consider as input an exploration radius $L \geq 1$, an accuracy level $\epsilon \in (0, 1]$ and a confidence level $\delta \in (0,1)$. We now formally define our exploration objective.

\begin{definition}[Multi-Goal Exploration --- \MGE] An algorithm is said to be $(\epsilon, \delta, L, \cG)$-PAC for \MGE if 
\begin{itemize}[leftmargin=.15in,topsep=-1.2pt,itemsep=1.4pt,partopsep=0pt, parsep=0pt]
    \item it stops after some (possibly random) number of exploration steps $\tau$ that is less than some polynomial in the relevant quantities ($S, A, L, \epsilon^{-1}, \log \delta^{-1}$) with probability at least $1 - \delta$,
    \item it returns a set of goal states $\cX$ and a set of policies $\{\wh{\pi}_g\}_{g \in \cX}$ such that $\mathbb{P}\big( \cC_1  \cap \cC_2 \big) \geq 1 - \delta,$ where we define the conditions
\end{itemize}
\begin{align*}
    \cC_1 &\triangleq \Big\{ \forall g \in \cX, ~ V^{\wh{\pi}_g}(s_0 \rightarrow g) - V^{\star}(s_0 \rightarrow g) \leq \epsilon \Big\}, \\
    \cC_2 &\triangleq \Big\{ \SL \subseteq \mathcal{X} \subseteq \SLepsilon \Big\}.
\end{align*}
The objective is to build an $(\epsilon, \delta, L, \cG)$-PAC algorithm for which the \MGE sample complexity, that is the number of exploration steps $\tau$, is as small as possible.
\label{def_MGE}
\end{definition}

\begin{remark} Since the goal set $\SL$ is \textit{unknown}, it may not be possible to exactly \textit{identify} it within a reasonable number of exploration steps. Thus we allow the learner to output a larger set $\mathcal{X}$ of candidate goals and policies. Nonetheless, we constrain an $(\epsilon, \delta, L, \cG)$-PAC algorithm for \MGE to return a set $\cX$ that is at most contained in the slightly larger set $\SLepsilon$ (i.e., $\mathcal{X} \subseteq \SLepsilon$).
\end{remark}

\begin{remark}
      Consider that $\cG = \cS$. Then the inclusion $\SL \subseteq \cS$ is an equality if $\mathcal{M}$ is communicating (i.e., $D < + \infty$) and if the (unknown) $D_0$ is lower or equal to $L$ (note that under Assumption~\ref{assumption_reset}, $D \leq D_0 +1$). Moreover, denote by $\SLarr$ the set of incrementally $L$-reachable states defined by \citet[][Definition~5]{lim2012autonomous}. Then in the special case where $\SLarr = \SL$, the \MGE objective of Definition~\ref{def_MGE} is equivalent to the \AXstar objective proposed by \citet[][Definition~5]{tarbouriech2020improved} in the incremental autonomous exploration setting introduced by \citet{lim2012autonomous}.
\end{remark}

\paragraph{\MGE vs.\,reset-free \MGE.} Lemma~\ref{lemma_exp_sep} establishes an exponential separation between \MGE and reset-free \MGE (i.e., \MGE without Assumption~\ref{assumption_reset}). This motivates the use of Assumption \ref{assumption_reset} to solve our learning objective in a reasonable number of exploration steps.

\begin{lemma} \MGE can be solved in $\textrm{poly}(S, L, \epsilon^{-1}, A)$ steps. On the other hand, there exists an MDP and a goal space where any algorithm requires at least $\Omega(D)$ steps to solve reset-free \MGE, where $D$ is exponentially larger than $L, S, A, \epsilon^{-1}$. 
\label{lemma_exp_sep}
\end{lemma}

\paragraph{\MGE lower bound.} We now give a worst-case lower bound on the \MGE problem (details in Appendix \ref{app_proofs}).

\begin{lemma}
For any algorithm that is $(\epsilon, \delta, L, \cG)$-PAC for \MGE for any MDP and goal space $\cG$, there exists an MDP and a goal space where the algorithm requires, in expectation, at least $\Omega(L^3 S A \epsilon^{-2})$ exploration steps to stop.
\label{lemma_LB}
\end{lemma}

\begin{remark} We can relate the dependencies in Lemma~\ref{lemma_LB} with the lower bound of the time steps needed to identify an $\epsilon$-optimal policy in both $\gamma$-discounted MDPs with a generative model --- i.e., $\Omega((1-\gamma)^{-3} S A \epsilon^{-2})$ \citep{azar2013minimax} --- and online stationary finite-horizon MDPs --- i.e., $\Omega(H^3 S A \epsilon^{-2})$ \citep{domingues2021episodic}. This correspondence is not surprising as $L$ captures the ``range'' of the \MGE problem, similar to the effective horizon $1/(1-\gamma)$ or the horizon $H$. Also recall that both discounted MDPs and finite-horizon MDPs are subclasses of goal-oriented MDPs, i.e., SSP-MDPs \citep{bertsekas1995dynamic}.
\end{remark}

\paragraph{\MGE under linear function approximation.} Lemma~\ref{lemma_LB} shows that the \MGE sample complexity must scale with $S A$ in the worst case, which may be prohibitive in the case of large state-action spaces. This motivates us to further analyze \MGE under \textit{linear function approximation}. In particular, we focus on the \textit{linear mixture} MDP setting \citep{ayoub2020model, zhou2021nearly}, which assumes that the transition probability is a linear mixture of $d$ signed basis measures. 

\begin{definition}[Linear Mixture MDP, \citealp{ayoub2020model, zhou2021nearly}]\label{def_linearmixture}
The unknown transition probability $p$ is a linear combination of $d$ signed basis measures $\phi_i(s'|s,a)$, i.e., $p(s'|s,a) \triangleq \sum_{i=1}^d \phi_i(s'|s,a)\theta^{\star}_i$. Meanwhile, for any $V: \cS \rightarrow [0,1]$, $i \in [d], (s,a) \in \cS \times \cA$, the summation $\sum_{s' \in \cS} \phi_i(s'|s,a) V(s')$ is computable. For simplicity, let $\bphi \triangleq [\phi_1, \dots, \phi_d]^\top$, $\btheta^{\star} \triangleq [\theta^{\star}_1,\dots, \theta^{\star}_d]^\top$ and $\bpsi_V(s, a) \triangleq \sum_{s' \in \cS} \bphi(s'| s, a)V(s)$. Without loss of generality, we assume $\|\btheta^{\star}\|_2 \le B, \|\bpsi_V(s, a)\|_2 \le 1$ for all $V: \cS \rightarrow [0,1]$ and $(s,a) \in \cS \times \cA$.
\end{definition}

\section{OUR \AdaGoalnormal APPROACH}
\label{sect_overview}

In Algorithm~\ref{algo}, we introduce the common algorithmic structure based on \AdaGoal. We use it to design \ALGO that tackles the \MGE problem in tabular MDPs, and \ALGOLM that tackles the \MGE problem in linear mixture MDPs. Both follow the goal-conditioned structure described in Section~\ref{sect_intro}. The agent sets a horizon of $H = \Omega(L \log L \epsilon^{-1})$ and splits its learning interaction in algorithmic episodes of length $H$. At the beginning of each algorithmic episode, it \GS selects a candidate goal state and \PE deploys an explorative (i.e., optimistic) policy conditioned on this goal for $H$ steps before resetting to $s_0$. It alternates     between these two steps until an adaptive stopping rule is met, at which point the algorithm terminates.\footnote{Indeed recall from Definition~\ref{def_MGE} that the algorithm must adaptively decide when to terminate its learning interaction.}

\SetKwInput{KwData}{Input}
\SetKwInput{KwOracles}{Oracles}
\SetKwInput{KwResult}{Output}
\SetKwInput{KwExample}{(\textit{Example}}
\SetKwComment{Comment}{$\triangleright$\ }{}

\let\oldnl\nl
\newcommand{\nonl}{\renewcommand{\nl}{\let\nl\oldnl}}

\begin{algorithm}[t!]
\SetAlgoLined
 \textbf{Input}: Exploration radius $L \geq 1$, accuracy level $\epsilon \in (0,1]$, confidence level $\delta \in (0,1)$. \\
 \textbf{Input}: \tab{Number of states $S$, number of actions $A$.} \\
 \textbf{Input}: \lm{Dimension of feature mapping $d$, bound $\pnorm$ on $\ell_2$-norm of $\btheta^{\star}$.} \\
 \textbf{Input}: Goal space $\cG \subseteq \cS$ (otherwise set $\cG = \cS$). \\
 Set as horizon $H \triangleq \lceil 5(L+\red{2})\log\big(\red{10}(L+\red{2})/\epsilon\big)/\log(2) \rceil.$ \label{line_H}\\
 \textbf{Initialize}: algorithmic episode index $k = 1$, distance estimates $\D_1(g) = \mathds{1}[g \neq s_0]$, error estimates $\E_1(g) = H \mathds{1}[g \neq s_0]$, goal-conditioned finite-horizon $Q$-values $\Q_{1,h}(s,a,g) = \mathds{1}[s \neq g]$, for all $(g, s, a, h) \in \cG \times \cS \times \cA \times [H]$.  \\
 \vspace{0.05in}
 \While{\textup{stopping rule \eqref{stopping_rule} is not met}}{
 \textbf{\raisebox{.5pt}{\textcircled{\raisebox{-.8pt} {i}}} Goal selection rule:} \\
 Select as goal state
 \nonl
  \vspace{-0.04in}
 \begin{align*}
g_k \in ~ \argmax_{g \in \cG}      \quad \quad &\E_k(g)        
\\
\textrm{subject to: } \quad& \D_k(g) \leq L. 
\end{align*}

\vspace{0.05in}
\textbf{\raisebox{.5pt}{\textcircled{\raisebox{-.8pt} {ii}}} Policy execution rule:} \\
For a duration of $H$ steps, run the optimistic goal-conditioned policy $\pi^k_{g_k}$ such that at step~$h$, $\pi^k_{g_k,h}(s) \in \argmin_{a \in \cA} \mathcal{Q}_{k,h}(s, a, g_k)$ (note that $\D_{k}(g_k) = \min_{a \in \cA} \mathcal{Q}_{k,1}(s_0, a, g_k)$). \\
Then execute action $\areset$ and increment episode index $k \mathrel{+}= 1$. 

\vspace{0.05in}
 \textbf{\raisebox{.5pt}{\textcircled{\raisebox{-.8pt} {iii}}} Update and check stopping rule:} \\
 \tab{Update estimates $\Q_k,\,\D_k,\,\E_k$ according to \eqref{eq_Qt},\,\eqref{eq_Dt},\,\eqref{eq_Et} using samples collected so far.} \label{line_estimates_tabular} \\
 \lm{Update estimates $\Q_k,\,\D_k,\,\E_k$ according to \eqref{eq_Qt_LFA},\,\eqref{eq_Dt_LFA},\,\eqref{eq_Et_LFA} using samples collected so far.} \label{line_estimates_LFA} \\
 Stop the algorithm if
 \nonl
  \vspace{-0.07in}
\begin{align}\label{stopping_rule}
\max_{g \in \cG: \, \D_k(g) \leq L} ~\E_k(g) ~\leq ~ \epsilon.
\end{align}
\vspace{-0.1in}
}

Let $\kappa \triangleq \inf \big\{ k \in \mathbb{N} : \max_{g \in \cG: \, \D_k(g) \leq L} \E_k(g) \leq \epsilon \big\}$. \\
\vspace{0.05in}
 \textbf{Output}: Goal states $\cX_{\kappa} \triangleq \{ g \in \cG: \D_{\kappa}(g) \leq L \}$, and for every $g \in \cX_{\kappa}$, a deterministic, non-stationary 
 policy $\wh{\pi}_g$ that at time step $i$ and state $s$ selects action $a$ according to:
 \nonl
 \vspace{-0.07in}
 \begin{align*}
        \wh{\pi}_g(a \vert s, i)
        \triangleq \left\{
    \begin{array}{ll}
        \argmin_{a \in \cA} \mathcal{Q}_{\kappa,h}(s, a, g) \\
        \quad \mbox{if } i \equiv h\ (\textrm{mod}\ H+1) ~\mbox{for } h \in [H], \\
        \areset \\
        \quad \mbox{if } i \equiv 0\ (\textrm{mod}\ H+1).
    \end{array}
    \right.
    \end{align*}
\caption{\AdaGoal-based algorithmic structure. \tab{Blue} text denotes \ALGO specific steps and \lm{purple} text denotes \ALGOLM specific steps.}
 \label{algo}
\end{algorithm}

\paragraph{$\Box$ \,\PE step.} Goal-conditioned finite-horizon $\cQ$-functions \citep{kaelbling1993learning, schaul2015universal} are maintained optimistically. At each episode $k$ and episode step $h \in [H]$, $\cQ_{k,h}(s,a,g)$ approximates (from below) the number of (expected) steps required to reach any goal~$g \in \cG$ starting from any state-action pair~$(s,a) \in \cS \times \cA$ and executing the optimal goal-reaching policy for $H-h$ steps. Intuitively, the $\cQ$-functions will gradually increase, more so for goal states that the agent struggles to reach. This is essentially done by initializing the $\cQ$-functions optimistically (i.e., at $0$), considering that the cost (i.e., negative reward) is $+1$ (resp.\,$0$) per time step if the conditioned goal is not reached (resp.\,reached), and carefully subtracting an exploration bonus to maintain optimism. Given a goal $g_k \in \cG$ selected at the beginning of episode $k$, the \PE step simply amounts to deploying an explorative policy conditioned on $g_k$, that is, a policy $\pi_{k,h}$ that greedily minimizes the current $\cQ$-functions, i.e., $\pi_{k,h}(s) \in \argmin_{a \in \cA} \cQ_{k,h}(s,a,g_k)$. 

\newcommand{\DHstar}{\D^{\star}_{\scaleto{\! H}{4pt}}}

\paragraph{$\Box$ \,\GS step.} To elect a relevant sequence of candidate goals $(g_k)_{k \geq 1}$, we introduce \AdaGoal, an adaptive goal selection scheme based on a simple constrained optimization problem. It relies on the agent's ability to compute two types of goal-conditioned quantities for any goal $g \in \cG$ and episode $k \geq 1$: 
\begin{itemize}[leftmargin=.12in,topsep=-1.5pt,itemsep=1pt,partopsep=0pt, parsep=0pt]
    \item a distance estimate from $s_0$ to $g$, denoted by $\D_k(g)$;
    \item an error of estimating this distance, denoted by $\E_k(g)$.
\end{itemize}
Conveniently, the distance estimates can be simply instantiated as $\D_k(g) \triangleq \min_{a \in \cA} \Q_{k,1}(s_0,a,g)$. Formally, we require the following properties of $\D$ and $\E$ to hold (with high probability):
\begin{itemize}[leftmargin=.12in,topsep=-1.5pt,itemsep=1.4pt,partopsep=0pt, parsep=0pt]
    \item \textbf{Property 1:} $\D$ is an \textit{optimistic distance estimate}, i.e., 
    \begin{align*}
        \D_k(g) \leq \DHstar(g), \quad \forall k \geq 1, \forall g \in \cG,
    \end{align*}
where $\DHstar(g) \triangleq \min_{\pi} \mathbb{E}\big[ \omega_{\pi}(s_0 \rightarrow g) \wedge H \big]$ denotes the shortest-path distance from $s_0$ to $g$ truncated at $H$ steps. Note that $\DHstar(g) \in [0, H]$ and that $\lim_{H \rightarrow + \infty} \DHstar(g) = V^{\star}(s_0 \rightarrow g)$.
    \item \textbf{Property 2:} $\E$ is an \textit{upper bound on the prediction error}, i.e.,
    \begin{align*}
        \abs{\DHstar(g) - \D_k(g)} \leq \E_k(g), \quad \forall k \geq 1, \forall g \in \cG.
    \end{align*}
\end{itemize}
Given these two goal-conditioned quantities, \AdaGoal selects at episode $k$ a candidate goal that solves the following constrained optimization problem:
\begin{subequations}
\begin{empheq}[box=\widefbox]{align}
    g_k \in ~ \argmax_{g \in \cG}      \quad \quad &\E_k(g)  \label{opt_pbm_constrained}      \\
    \textrm{subject to: } \quad& \D_k(g) \leq L. 
    \label{opt_pbm_constrained_constraint}
\end{empheq}
\end{subequations}

\begin{figure*}[t!]
\centering
\vspace{-0.02in}
\begin{minipage}{0.6\linewidth}
\includegraphics[width=1.45in,trim =20 20 20 5,clip]{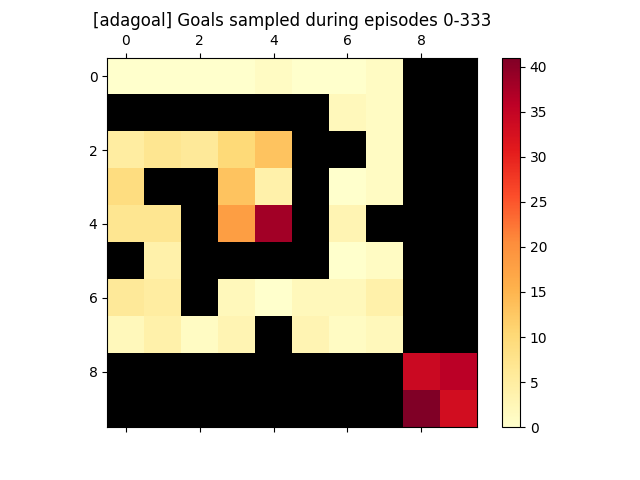} \hspace{-0.25in}
\includegraphics[width=1.45in,trim =20 20 20 5,clip]{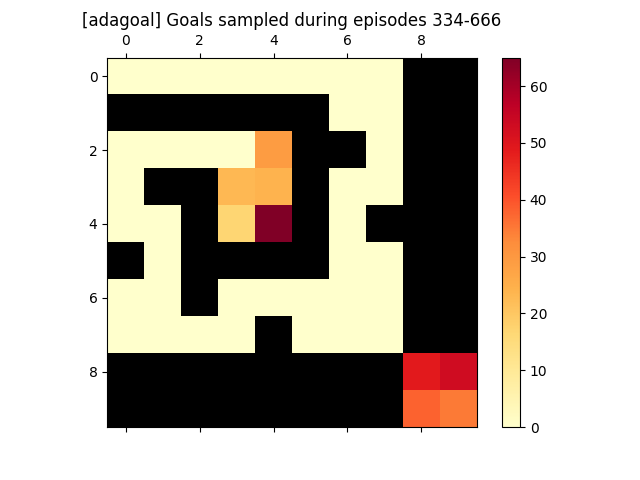} \hspace{-0.25in}
\includegraphics[width=1.45in,trim =20 20 20 5,clip]{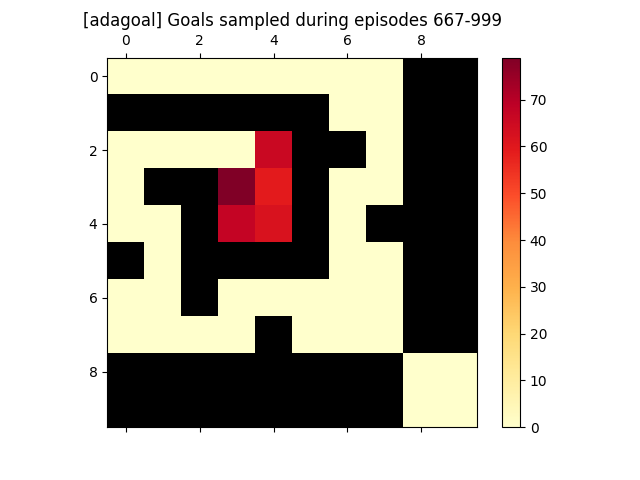}
\end{minipage}%
\hfill%
\begin{minipage}{0.40\linewidth}
\caption{Goal sampling frequency of \ALGOscs over $1000$ episodes of length $H=50$ (with $L=40$). The grid-world has $S=52$ states, starting state $s_0 = (0, 0)$ (top left), $A=5$ actions ($4$ cardinal ones and $\areset$). The $4$ states of the bottom right room can only be accessed from $s_0$ by any cardinal action with probability $\eta = 0.001$ (their associated $V^{\star}(s_0 \rightarrow \cdot)$ thus scale with $\eta^{-1}$).}
\label{fig:gridworld_goals_sampled_main}
\end{minipage}%
\vspace{-0.1in}
\end{figure*}

\paragraph{Interpretation.} The agent sequentially selects a goal state with highest prediction error~$\E$ among those whose distance estimate~$\D$ is not too large. If the agent is confident that a goal $g$ is either too easy or too hard to reach, it will assign a low prediction error $\E(g)$. As a result, the objective function in \eqref{opt_pbm_constrained} adaptively samples goal states on the frontier of the learning process. The constraint in \eqref{opt_pbm_constrained_constraint}, although it is not required for the final sample complexity result, further tightens the goal selection process. Indeed, for any $k \geq 1$, let
    \begin{align*}
         \mathcal{X}_k ~ \triangleq ~ \big\{ g \in \cG: \D_k(g) \leq L \big\}, \qquad \epsilon_k ~ \triangleq ~  \max_{g \in \mathcal{X}_k} ~ \E_k(g).
\end{align*}
Then, if as a warm-up we take the limit $H \rightarrow + \infty$, injecting Properties 1 and 2 in (\ref{opt_pbm_constrained}-\ref{opt_pbm_constrained_constraint}) entails that 
\begin{align} \label{eq_sandwich_Xk}
        \SL \subseteq \mathcal{X}_k \subseteq \SLepsilonk.
    \end{align}
We thus see that the constraint in (\ref{opt_pbm_constrained_constraint}) does not remove valid goals in $\SL$ from the set $\mathcal{X}_k$ of candidate goal states to sample. Second, decreasing $\epsilon_k$ has the dual impact of making the set of candidates goals $\mathcal{X}_k$ closer to $\SL$ and improving their distance estimates, which motivates the goal selection scheme in (\ref{opt_pbm_constrained}). In practice, we consider a \textit{finite} truncation $H$ (line~\ref{line_H} in Algorithm~\ref{algo}), thus we need to account for the bias $\rho_g \triangleq V^{\star}(s_0 \rightarrow g) - \DHstar(g)$, which can be arbitrarily large for goals $g$ that are hard or impossible to reach. Fortunately, our \AdaGoal strategy will be able to gradually discard such states, hence the final \MGE sample complexity will not pay for such terms. In fact, we will later see that the choice of horizon $H = \Omega(L \log L \epsilon^{-1})$ ensures that $\rho_g = O(\epsilon)$ for all the (unknown) goal states of interest $g \in \SL$.

\paragraph{Choice of $\Q, \D, \E$.} A key algorithmic design is how to build and update the goal-conditioned $\Q$-functions and the estimates~$\D$ and~$\E$. In Appendices~\ref{app_TAB} and~\ref{app_LFA}, we will carefully construct them with exact bonus-based estimates for both tabular MDPs and linear mixture MDPs. As suggested by the algorithms' names, the estimates of the former are inspired from \BPIUCBVI \citep{menard2021fast}, an algorithm for best policy identification in finite-horizon tabular MDPs, while those of the latter are inspired from \UCRLVTR \citep{ayoub2020model}, an algorithm for regret minimization in finite-horizon linear mixture MDPs. Since all our estimates are optimistic, we see that Algorithm~\ref{algo} relies on the principle of optimism in the face of uncertainty \textit{both} for the goal selection and the policy execution. Finally, in Section~\ref{section_AdaGoal_deepRL}, we propose a way to instantiate $\Q, \D, \E$ for a practical implementation in deep RL.

\paragraph{$\Box$ \,Adaptive stopping rule.} 

At the end of each algorithmic episode, the estimates are updated using the samples collected so far, and the algorithm checks whether a stopping rule \eqref{stopping_rule} based on \AdaGoal is triggered, in which case it terminates. This occurs when the prediction errors $\E$ of all the goal states that meet the \AdaGoal constraint (\ref{opt_pbm_constrained_constraint}) are below the prescribed accuracy level~$\epsilon$. These states then form the set of \textit{candidate goal states} output by Algorithm~\ref{algo}, along with their associated optimistic goal-reaching policies.

\paragraph{Empirical validation.} In Figure~\ref{fig:gridworld_goals_sampled_main} (see Appendix~\ref{app_exp} for details), we empirically study the sequence of goals selected by \ALGO during learning. We design a two-room grid-world with a very small probability of reaching the second room. We see that \AdaGoal is able to discard the states from the second room and target as goals the states in the first room that are ``furthest'' away from $s_0$, which effectively correspond to the fringe of what the agent can reliably reach.

\section{SAMPLE COMPLEXITY GUARANTEES}
\label{sect_theory}


\paragraph{$\Box$ \,Guarantee for \ALGO.} We first bound the \MGE sample complexity of \ALGO. For simplicity we consider that $\cG = \cS$, i.e., the goal space spans the entire state space (the results trivially extend to any $\cG \subseteq \cS$).

\newcommand{\thmMGE}{
\ALGO is $(\epsilon, \delta, L, \cS)$-PAC for \MGE and, with probability at least $1-\delta$, for $\epsilon\in (0,1/S]$ its \MGE sample complexity is of order\footnote{ The notation $\wt{O}$ in Theorem \ref{lemma_UB} hides poly-log terms in $\epsilon^{-1},S,A,L,\delta^{-1}$. See Lemma~\ref{lemmaAPP_algo_sample_complexity_for_MGE} in Appendix \ref{app_proof_UB} for a more detailed bound that includes the poly-log terms.} $\wt{O}(L^3 S A \epsilon^{-2})$.
}

\begin{theorem}
  \label{lemma_UB}
  \thmMGE
\end{theorem}

Lemma~\ref{lemma_LB} and Theorem~\ref{lemma_UB} imply that the \MGE sample complexity of \ALGO is nearly minimax optimal for small enough $\epsilon$ and up to logarithmic terms.

In the absence of a pre-specified exploration radius~$L$, the agent can build its own curriculum for~$L$ (i.e., design a sequence of increasing~$L$'s) to guide its learning. In this case, the total sample complexity is (up to a logarithmic factor) the same of \ALGO run with the final value of~$L$, as stated below.

\begin{corollary} \label{cor_UB_increasingL}
The successive execution of \ALGO for an increasing sequence $L \in \{ 2, 2^2, . . . , 2^f\}$ with $f \in \mathbb{N}^*$ is $(\epsilon, \delta, L_{\text{f}}, \cS)$-PAC for \MGE and, with probability at least $1-\delta$, for $\epsilon\in (0,1/S]$ its \MGE sample complexity is of order $\wt{O}(L_{\text{f}}^3 S A \epsilon^{-2})$, where \mbox{$L_{\text{f}} = 2^f$}. 
\end{corollary} 

\newcommand{\wtD}{{\scalebox{0.92}{$\wt D$}}}

Finally, we can investigate the special case where $\mathcal{M}$ is communicating and the objective is to learn an $\epsilon$-optimal goal-conditioned policy for \textit{every} goal state. Since the $s_0$-diameter $D_0$ is unknown, we can use as an initial subroutine the \GOSPRL algorithm \citep{tarbouriech2021provably} to compute an estimate $\wtD$ such that $D_0 \leq \wtD \leq 2 D_0$ in $\wt{O}(D_0^3 S^2 A)$ time steps. Then we can execute \ALGO with $L = \wtD$, which leads to the following guarantee. 

\begin{corollary} \label{cor_UB_comm}
Assume that the MDP $\mathcal{M}$ has a finite and unknown $s_0$-diameter $D_0$. Then the above strategy is $(\epsilon, \delta, D_0, \cS)$-PAC for \MGE and, with probability at least $1-\delta$, for $\epsilon\in (0, 1/S]$, its \MGE sample complexity is of order $\wt{O}(D_0^3 S A \epsilon^{-2})$. 
\end{corollary}

\paragraph{Comparison with existing bounds.} A different although related problem is considered in \citet{tarbouriech2021provably} in \textit{communicating} MDPs, where the agent must find an $\epsilon$-optimal goal-conditioned policy for any arbitrary starting state, goal state and positive cost function. For small enough $\epsilon$, their bound in the unit-cost case scales as $\wt{O}(D^4 \Gamma S A \epsilon^{-2})$, where~$\Gamma$ is the branching factor which in the worst case is~$S$. We see that Corollary~\ref{cor_UB_comm} improves over \citet{tarbouriech2021provably} by a factor~$D_0$ as well as a factor $\Gamma \leq S$. Although the latter considers a more demanding cost-free objective (i.e., for any positive cost function), it is unable to avoid its superlinear dependence in~$S$ when instantiated in our scenario of unit cost functions, since the algorithm's design is to estimate uniformly well the transition kernel.~Finally, \DisCo \citep{tarbouriech2020improved} designed for incremental autonomous exploration can solve \MGE in the special case where $\SLarr = \SL$, where $\SLarr$ denotes the set of incrementally $L$-reachable states (see \citealp[][Definition~5]{lim2012autonomous}). However, the bound of \DisCo would translate to $\wt{O}(L^5 \Gamma S A \epsilon^{-2})$, which is also worse than Theorem \ref{lemma_UB}. 

\paragraph{$\Box$ \,Guarantee for \ALGOLM.} We now bound the \MGE sample complexity of Algorithm~\ref{algo} in linear mixture MDPs (Definition~\ref{def_linearmixture}). Since the state space $\cS$ may be large, we consider that the known goal space is in all generality a subset of it, i.e., $\cG \subseteq \cS$, where $G \triangleq \abs{\cG}$ denotes the cardinality of the goal space. 

\begin{theorem} \label{SC_algolm} In linear mixture MDPs, for $\epsilon\in (0,1]$, \ALGOLM is $(\epsilon, \delta, L, \cG)$-PAC for \MGE and, with probability at least $1-\delta$, its \MGE sample complexity is of order\footnote{The notation $\wt{O}$ in Theorem \ref{SC_algolm} hides poly-log terms in $\epsilon^{-1},G,d,L,\delta^{-1}$.} $\wt{O}\left( L^4 d^2 \epsilon^{-2} \right)$, where $d$ is the dimension of the feature mapping.
\end{theorem}

To the best of our knowledge, Theorem \ref{SC_algolm} yields the first goal-oriented PAC guarantee with linear function approximation. The algorithm's choice of $\E,\,\D,\,\Q$ relies on two regression-based goal-conditioned estimators, one standard ``value-targeted'' estimator inspired from \UCRLVTR \citep{ayoub2020model} and one novel ``error-targeted'' estimator, see Appendix~\ref{app_LFA}. We expect that the bound of Theorem \ref{SC_algolm} can be refined using tighter Bernstein-based estimates, for instance inspired from \UCRLVTRp \citep{zhou2021nearly}, which we leave as future work. Note that the $\wt{O}$ notation in Theorem \ref{SC_algolm} contains a $\log(G)$ factor (which appears when performing a union bound argument over all goals $g \in \cG$), and the computational complexity of \ALGOLM scales with $G$ (since the algorithm maintains goal-conditioned estimates). We point out that here we still consider that the MDP has a finite number of states~$S$. This is to be expected given the way a goal is currently modeled (at the granular level of states), independently of how large the state space is, where it may be very hard to visit specific states. Learning in single- or multi-goal RL beyond a finite state space is an interesting direction of future investigation. Note that existing works on single-goal exploration (i.e., SSP regret minimization) with linear function approximation \citep{vial2021regret, min2021learning} also assume that the state space is finite.

\section{ANALYSIS OVERVIEW} \label{section_analysis}

The proofs of Theorems~\ref{lemma_UB} and~\ref{SC_algolm} (see Appendices~\ref{app_TAB} and~\ref{app_LFA}) are decomposed in the same following key steps.

\paragraph{$\triangleright$~\textit{Key step \ding{172}: Optimism and gap bounds.}}

We prove that Properties 1 and 2 hold with high probability, i.e., \textbf{(i)} the quantities $\D$ are optimistic estimates of the optimal goal-conditioned finite-horizon value functions $\DHstar$ and \textbf{(ii)} the quantities $\E$ are valid upper bounds to the goal-conditioned finite-horizon gaps.

\paragraph{$\triangleright$~\textit{Key step \ding{173}: Bounding the cumulative gap bounds.}}

We make explicit a function $f_{\mathcal{M}}$ (depending on the MDP $\mathcal{M}$) that is strictly decreasing in $K$ with $f_{\mathcal{M}}(K) \rightarrow_{K \rightarrow \infty}\,0$, such that with high probability, 
\begin{align} \label{eq_PE}
    \sum_{k=1}^K \DHstar(g_k) - \D_k(g_k) \leq  \sum_{k=1}^K \E_k(g_k) \leq K \cdot f_{\mathcal{M}}(K),
\end{align}
for any number of algorithmic episodes $K$. Specifically, we establish that $f_{\mathcal{M}}(K)= \wt{O}\big( \sqrt{K H^2 S A} + H^2 S^2 A \big)$ for \ALGO, and $f_{\mathcal{M}}(K)= \wt{O}\big( \sqrt{ K H^{3} d^2} + H^2 d^{3/2} \big)$ for \ALGOLM. 
\eqref{eq_PE} resembles a \textit{no-regret} property of the exploration algorithm that receives as input the sequence of goals $(g_k)_{k \geq 1}$ prescribed by \AdaGoal and performs the~\PE step. Indeed, intuitively, the aim of the \PE step is to improve the estimation of $\D_k(g_k)$ and make it closer to $\DHstar(g_k)$, i.e., to decrease the prediction error $\E_k(g_k)$.

\paragraph{$\triangleright$~\textit{Key step \ding{174}: Bounding the sample complexity.}} To bound $\kappa$ the episode index at which Algorithm~\ref{algo} terminates, we combine \eqref{eq_PE} and the termination condition \eqref{stopping_rule} to simultaneously lower and upper bound (with high probability) the cumulative errors $\E$ as 
\begin{align*}
    \epsilon \cdot (\kappa -1) \leq  \sum_{k=1}^{\kappa-1} \E_k(g_k) \leq (\kappa - 1) \cdot f_{\mathcal{M}}(\kappa - 1).
\end{align*}
Inverting this functional inequality in $\kappa$ yields that $\kappa$ is finite and bounded as 
\begin{align} \label{eq_kappa_bound}
    \kappa \leq f^{-1}_{\mathcal{M}}(\epsilon) + 2.
\end{align}
The sample complexity is $(H+1) \cdot \kappa$ with $H = \wt{O}(L)$, thus \ALGO (resp.\,\ALGOLM) stops in $\wt{O}(L^3 S A \epsilon^{-2} + L^3 S^2 A \epsilon^{-1})$ (resp.\,$\wt{O}(L^4 d^2 \epsilon^{-2})$) time steps, with high probability.

\paragraph{$\triangleright$~\textit{Key step \ding{175}: Connecting to the original \MGE objective.}}  The key remaining step is to prove that the \MGE objective is indeed fulfilled. 

\begin{remark} Algorithm~\ref{algo} relies on a finite-horizon construction, with algorithmic episodes of length~$H$. This relates to the reduction of SSP to finite-horizon studied in some SSP regret minimization works \citep{cohen2021minimax, chen2021finding}, which rely on the idea that an SSP problem can be approximated by a finite-horizon problem if the horizon is large enough w.r.t.\,$T_{\star}$, the optimal expected hitting time to the goal starting from any state. Two main differences arise in our \MGE setting: (i) first, in these works, the goal state is fixed throughout learning and $T_{\star}$ is assumed known, whereas we need to deal with goal selection and find the relevant goals of interest while having to discard those that are poorly reachable or unreachable. (ii) Second, these works ensure that the \textit{empirical} goal-reaching performance of the algorithm's non-stationary policy \textit{over the whole learning interaction} is good enough (by definition of the regret objective in SSP). As such, they do not show that the \textit{expected} performance of some \textit{candidate} policy is good enough (i.e., the SSP value function is small enough) -- in fact, they do not even explicitly prove that the executed policies are proper. The latter property may actually not be possible to obtain since standard regret-to-PAC conversion may not work in SSP \citep{tarbouriech2021sample}. In our \MGE objective, the key difference lies in the availability of the reset action (Assumption~\ref{assumption_reset}), as we will now see.
\end{remark}

Our analysis builds on the following reasoning: given a goal state in $\SLepsilon$, we can find a candidate policy with near-optimal goal-reaching behavior (i.e., SSP value function) by: (i) first computing a near-optimal policy $\wt{\pi}$ in the finite-horizon reduction (using the stopping rule \eqref{stopping_rule}), (ii) and then expanding $\wt{\pi}$ into an infinite-horizon policy via the reset action every $H$ time steps to get our desired candidate policy. 

Now, importantly, the above reasoning \textit{only holds} for the goals in $\SLepsilon$, which is an \textit{unknown} set. This is where our \AdaGoal strategy comes into the picture, as it provides a simple and computable \textit{sufficient condition} for a goal to belong to $\SLepsilon$. 

\newcommand{\lemEandDgivesgSLeps}{
With probability at least $1-\delta$, if a goal state $g \in \cG$ satisfies $\D_k(g) \leq L$ and $\E_k(g) \leq \epsilon$ for an episode $k \geq 1$, then $g \in \SLepsilon$. 
}

\begin{lemma}
  \label{lemma_EandD_gives_g_SLeps}
  \lemEandDgivesgSLeps
\end{lemma}

We are now ready to put everything together and prove that \ALGO and \ALGOLM are $(\epsilon, \delta, L)$-PAC for \MGE. The candidate goal states are $\cX_{\kappa} \triangleq \{ g \in \cS: \D_{\kappa}(g) \leq L \}$, with candidate policies $\wh{\pi}_g \triangleq (\pi^{\kappa + 1}_g)^{|H}$. In what follows we reason with high probability. Property 1 ensures that $\SL \subseteq \cX_{\kappa}$, while Lemma~\ref{lemma_EandD_gives_g_SLeps} entails that $\cX_{\kappa} \subseteq \SLepsilon$. Finally, for any $g \in \cX_{\kappa}$, combining Property 2 and the termination condition \eqref{stopping_rule} gives that $\pi_{g}^{\kappa + 1}$ is $\epsilon/9$-optimal in $\bcM_{g,H}$. As a result, the translation from the finite-horizon to goal-oriented objective (which holds since $g \in \SLepsilon$ and by choice of the horizon $H = \Omega(L \log L \epsilon^{-1})$, see Lemma~\ref{lem:link_FH_SSP}) yields that $V_{g}^{\wh{\pi}_g}(s_0) \leq V^{\star}_{g}(s_0) + \epsilon$, i.e., $\wh{\pi}_g$ is $\epsilon$-optimal for the original SSP objective. This concludes the proofs of Theorems \ref{lemma_UB} and \ref{SC_algolm}.

\begin{figure*}[t!]
\centering
\vspace{-0.08in}
\includegraphics[width=1.6in,trim=0 -4cm 0 0]{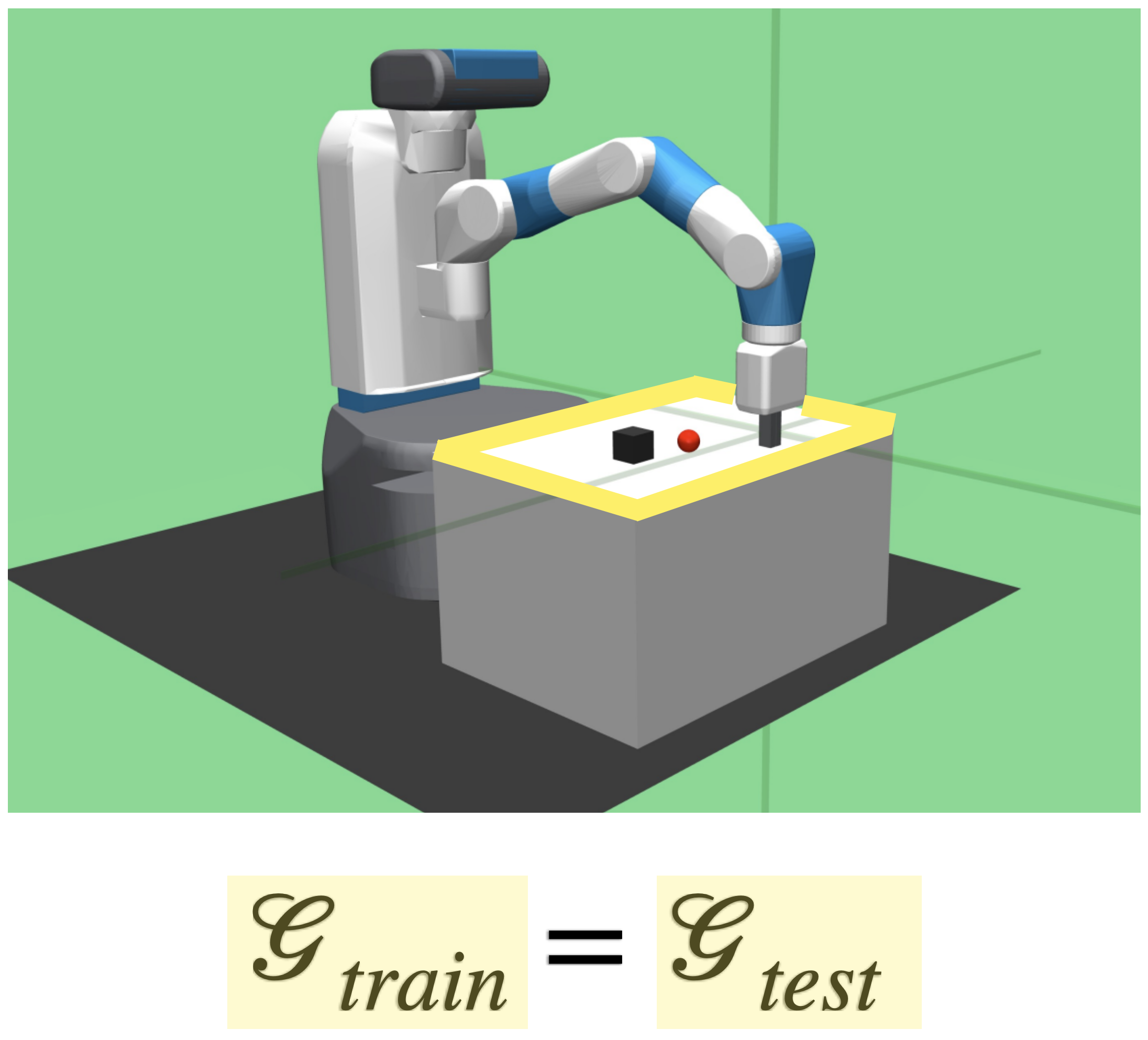}\hspace{0.1in} 
\includegraphics[width=4in]{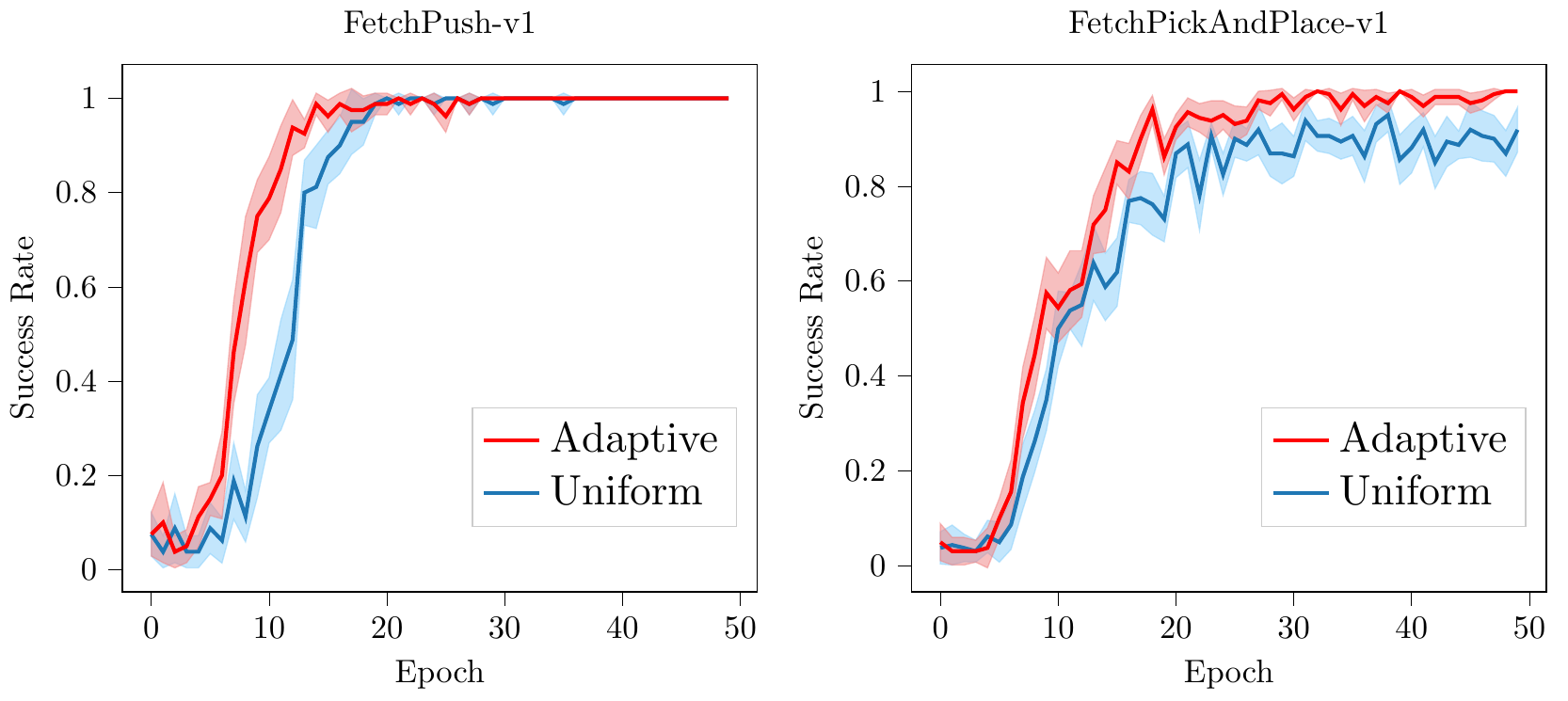} \\
\hspace*{0.16in}~\includegraphics[width=1.6in,trim=0 -4cm 0 0]{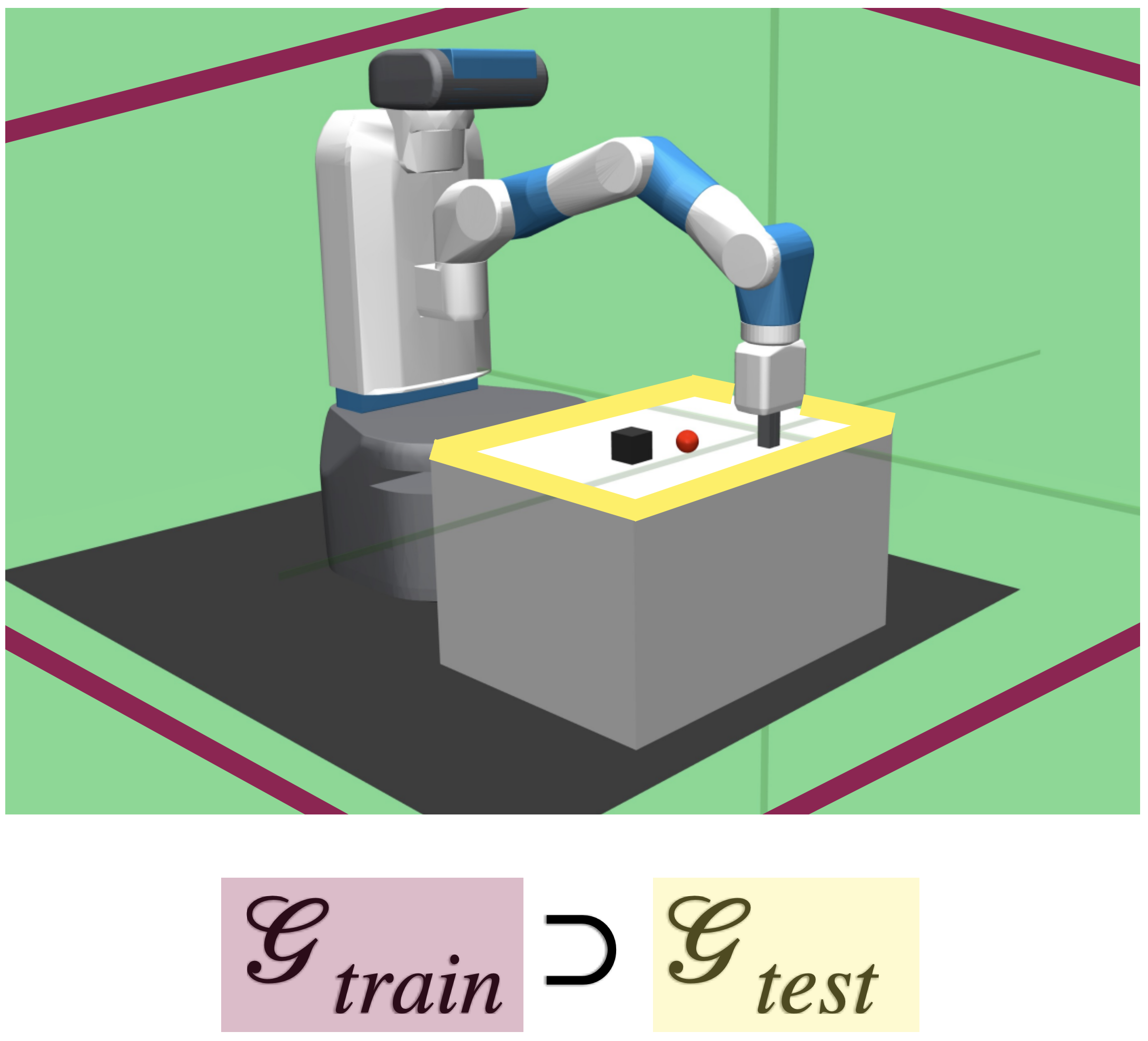} \hspace{0.1in} 
\includegraphics[width=4.16in]{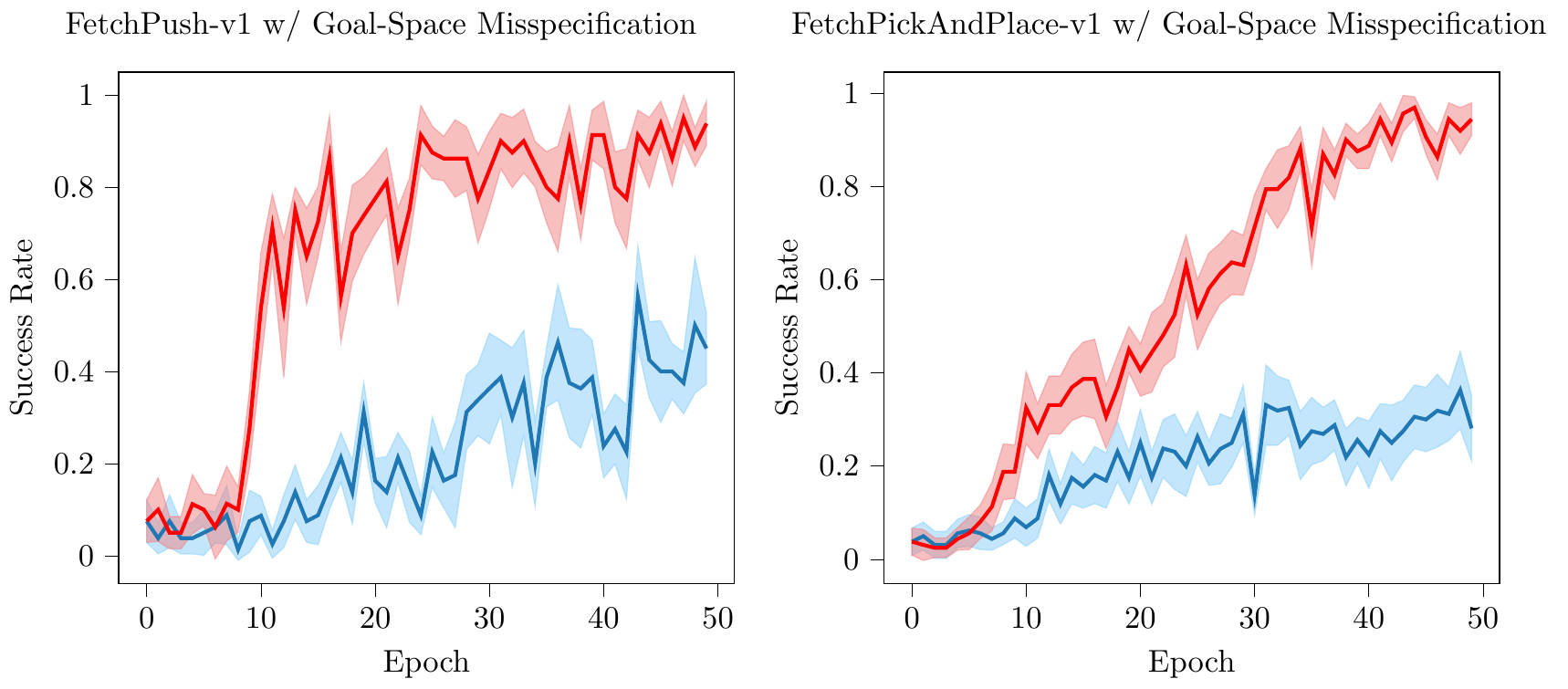}
\vspace{-0.09in}
\caption{Success rate evaluated on $\cG_{\text{test}}$ with the latest policy trained on $\cG_{\text{train}}$. The shaded region represents confidence over 5 random seeds. The adaptive goal sampling scheme improves the learning performance over the uniform sampling of~\HER. This is especially the case in the presence of goal-space misspecification (bottom row), where the training goal space $\cG_{\text{train}}$ (delimited in purple) is larger than the test goal space $\cG_{\text{test}}$ (delimited in yellow).}
\label{fig:deep}
\vspace{-0.02in}
\end{figure*}

\section{OPERATIONALIZING \AdaGoalnormal IN DEEP RL} \label{section_AdaGoal_deepRL}

In this section, we present a way to operationalize the \AdaGoal idea of targeting goals with high ``uncertainty''. We show it can be implemented similar to the deep RL algorithm of \citet{zhang2020automatic}, and we investigate an aspect that was not considered in the latter paper, which pertains to the capability of \AdaGoal to adapt to an unknown target goal set ($\SL$) given a goal set that is possibly misspecified $(\cG)$. 

First, we notice that $\Q$ and $\D$ in Algorithm~\ref{algo} can be learned in practice with a goal-conditioned value-based neural network \citep{schaul2015universal}. Meanwhile, the predictions errors $\E$ can be approximated by the disagreement between an ensemble of goal-conditioned Q-functions. 
Interestingly, this approach has already been investigated in the deep GC-RL algorithm \VDS \citep{zhang2020automatic}, which considers a similar goal-proposal module to prioritize goals that maximize the epistemic uncertainty of the Q-function of the current policy, in order to sample goals at the frontier of the set of goals that the agent is able to reach.

Specifically, we take an ensemble $\{ Q_j\}_{1 \leq j \leq J}$ of $J$ randomly initialized goal-conditioned $Q$-functions and we instantiate $\E(g)$ as the standard deviation of the ensemble's $Q_j$ values conditioned on goal $g$. 
Since computing the maximum over $g \in \cG$ in (\ref{opt_pbm_constrained}) is expensive if the goal space $\cG$ is large, the procedure is replaced by first uniformly sampling a set of candidate goals $\{ g^{(n)}\}_{n=1}^N \subset \cG$, and then selecting a goal $g^{(n)}$ with probability 
\begin{align*}
    q_n &\triangleq \frac{\E(g^{(n)})}{\sum_{n'=1}^N \E(g^{(n')})}, ~ \E(g) \triangleq\!\underset{1 \leq j \leq J}{\textrm{std}}\!\big\{\!\min_{a \in \cA} Q_j(s_0, a , g)\big\}.
\end{align*}
Moreover, approximating $H \approx L$ renders the constraint in (\ref{opt_pbm_constrained_constraint}) always valid so it can be omitted. Hence, this approximation of \AdaGoal exactly recovers the goal sampling scheme of \citet{zhang2020automatic}, which pairs it with Hindsight Experience Replay \citep[\HER,][]{andrychowicz2017hindsight} that performs \textit{uniform} goal sampling. 

We consider the multi-goal environments of FetchPush and FetchPickAndPlace, which are sparse-reward simulated robotic manipulation tasks from OpenAI Gym \citep{plappert2018multi}. We empirically compare the performance of such an adaptive goal selection that prioritizes ``uncertain'' goals and the performance of \HER's uniform goal selection (see Appendix~\ref{app_deepRL_exp} for implementation details). We observe in Figure~\ref{fig:deep} (top row) that the adaptive goal sampling scheme outperforms the uniform one of \HER, which is consistent with the results of \citet{zhang2020automatic}. 

In the above experimental set-up (which is in fact considered by most deep GC-RL works), the goal space $\cG_{\text{train}}$ seen at train time is the same as the goal space $\cG_{\text{test}}$ on which the agent is evaluated at test time, i.e., the white rectangular table. In the language of the previous sections, by relating $\cG_{\text{train}} \leftrightarrow \cG$ and $\cG_{\text{test}} \leftrightarrow \SL$, this means that the environment is considered communicating and $\cG = \cG_L$. However, in some cases, there may be some \emph{misspecification} in the goal space seen during the learning interaction. This may occur if the agent is unaware of the goals of interest, in which case we have that $ \cG_L \subsetneq \cG$, where we recall that $\cG_L$ is a priori unknown. We design an experiment to model this scenario by translating the x-y range of $\cG_{\text{train}}$ by a factor of $\lambda \geq 1$. Specifically, denoting by $(x_0,y_0,z_0)$ the center of the table and letting $r \triangleq 0.15$, we leave $\cG_{\text{test}}$ unchanged yet we expand $\cG_{\text{train}} \supset \cG_{\text{test}}$ as 
\begin{align*}
    \cG_{\text{test}} &\triangleq \big\{ ( x_0 + \mathcal{U}(-r, r), y_0 +  \mathcal{U}(-r, r), z_0) \big\}, \\
    \cG_{\text{train}} &\triangleq \big\{ ( x_0 + \mathcal{U}(-\lambda r, \lambda r), y_0 + \mathcal{U}(-\lambda r, \lambda r), z_0) \big\},
\end{align*}
with $\lambda_{\text{\scalebox{.8}{FetchPush}}} = 10$, $\lambda_{\text{\scalebox{.8}{FetchPickAndPlace}}} = 5$. In this scenario, Figure~\ref{fig:deep} (bottom row) shows that an adaptive goal sampling scheme is particularly pertinent. Intuitively, it enables to discard the set of goals $\cG_{\text{train}} \setminus \cG_{\text{test}}$ that cannot be reached and thus hinder learning when the agent conditions its behavior on them. This empirically corroborates \AdaGoal's (theoretically established) ability to adapt to an unknown target goal set ($\SL$) given a goal set that is possibly misspecified $(\cG)$.




\newpage

\subsection*{Acknowledgement}

We thank Evrard Garcelon, Andrea Tirinzoni and Yann Ollivier for helpful discussion.


\bibliographystyle{plainnat}
\bibliography{bibliography.bib}

\onecolumn
\appendix

\newpage

\part{Appendix}

\parttoc

\renewcommand{\red}[1]{\noindent{\textcolor{black}{#1}}}

\section{SHORTCOMINGS OF ALGORITHMIC DESIGNS OF EXISTING THEORETICAL GC-RL METHODS}\label{app_examples}

Here we corroborate our discussion in Section \ref{sect_intro} that the existing theoretical GC-RL approaches have quite involved algorithmic designs, far from the interpretable alternation of \GS and \PE steps, thus making them poorly amenable to a more practical implementation. We identify three specific algorithmic shortcomings that we describe below. Specifically, \UcbExplore \citep{lim2012autonomous} suffers from Shortcomings 1 and 2, while both \DisCo \citep{tarbouriech2020improved} and \GOSPRL \citep{tarbouriech2021provably} suffer from Shortcoming 3. 
\begin{itemize}
    \item \textit{Shortcoming 1: Cumbersome \PE step.} Once the method fixes a relevant goal state, it has an elaborate policy execution phase, by running numerous executions of the candidate optimistic goal-conditioned policy and checking after each execution whether an empirical performance check is verified to determine whether to continue the policy execution phase or not. 
    
    In contrast, our \AdaGoal-based approach has a simple \PE step, which executes a single rollout of the candidate optimistic goal-conditioned policy for a fixed number of steps before resetting and selecting a (possibly) different, more relevant candidate goal. 
    
    \item \textit{Shortcoming 2: Possibly eliminatory \GS step.} The method establishes a strict separation between goal states that have been identified as reliably reachable by a policy and those that have not. In particular, the goal states belonging to the former category can never be selected again as candidate goals, even if the agent could later discover more promising policies leading to it. 
    
    In contrast, our \AdaGoal-based approach only implicitly targets goal states of intermediate difficulty (without relying on an explicit distinction between goal states that have already been identified as being reliably reachable and the other goal states). This also means that \AdaGoalnormal does not need to formalize a notion of ``fringe'' or ``border'', which can make it more amenable to implementation.
    
    \item \textit{Shortcoming 3: Exhaustive/non-adaptive stopping condition.} The method has a predefined stopping condition, that poorly (or does not) adapts to the samples collected so far. For instance, the method explicitly fixes a minimum number of samples to have collected from \textit{every} state-action pair of interest, or fixes a minimum number of times a candidate goal must be selected. 
    
    In contrast, our \AdaGoal-based approach relies on an \textit{adaptive}, data-driven strategy to select goals and to terminate the algorithm. This may also pave the way for problem-dependent sample complexity guarantees.
\end{itemize}

\section{PROOFS} 
\label{app_proofs}

\begin{proof}[Proof of Equation~\eqref{eq_sandwich_Xk}] Here we take the limit $H \rightarrow + \infty$. Let $g \in \SL$, then Property 1 entails that $\D_k(g) \leq V^{\star}(s_0 \rightarrow g) \leq L$, thus $g \in \Xk$, therefore $\SL \subseteq \Xk$. Now let $g \in \Xk$, then by Property 2 and definition of $\epsilon_k$, we have that $V^{\star}(s_0 \rightarrow g) \leq \D_k(g) + \E_k(g) \leq L + \E_k(g) \leq L + \epsilon_k$, therefore $\Xk \subseteq \SLepsilonk$. 
\end{proof}

\begin{proof}[Proof of Equation~\eqref{eq_kappa_bound}]
    Fix any finite episode index $K < \kappa$, where $\kappa$ denotes the (possibly unbounded) episode index at which \ALGO terminates. By design of \AdaGoal, we have that $\epsilon \leq \E_k(g_k)$ for every $k \leq K$. We assume that $\kappa>1$ otherwise the result is trivially true. Define $\kappa' \triangleq \min\{\kappa-1, K\} \geq 1$. Summing the inequality above yields $\epsilon \cdot \kappa' \leq \sum_{k=1}^{\kappa'} \E_k(g_k) \leq \kappa' \cdot f_{\mathcal{M}}(\kappa')$. This implies that $\epsilon \leq f_{\mathcal{M}}(\kappa')$, in which case $f^{-1}_{\mathcal{M}}(\epsilon) \geq \kappa'$, since $f^{-1}_{\mathcal{M}}$ is strictly decreasing (like $f_{\mathcal{M}}$). Since the last inequality holds for any finite $K < \kappa$, letting $K \rightarrow + \infty$ implies that $\kappa$ is finite and bounded by $f^{-1}_{\mathcal{M}}(\epsilon) + 2$.
\end{proof}

\subsection{Proof of Lemma \ref{lemma_exp_sep}}
We assume throughout that $L \geq 2$ and $\epsilon \in (0, 1]$. On the one hand, Theorem \ref{lemma_UB} proves that \ALGO is $(\epsilon, \delta, L, \cG)$-PAC for \MGE with sample complexity of order $\wt{O}(L^3 S A \epsilon^{-2})$, therefore \MGE is solvable in $\textrm{poly}(S, L, \epsilon^{-1}, A)$ steps (up to poly-log factors). On the other hand, reset-free \MGE is a special case of the cost-free goal-free exploration problem in communicating MDPs studied by \citet{tarbouriech2021provably}, whose algorithm {\small\textsc{GOSPRL}} can solve it in $\textrm{poly}(S, D, \epsilon^{-1}, A)$ steps (up to poly-log factors). Now we prove that there exists an MDP such that any algorithm requires at least $\Omega(D)$ steps to solve the reset-free \MGE problem (i.e., without Assumption~\ref{assumption_reset}), where the MDP's diameter $D$ can be made exponentially larger than $L, S, A, \epsilon^{-1}$. 

To this end, we design in Figure \ref{fig:mdp-separation} a communicating MDP $\mathcal{M}_{a^{\dagger}}$ that does not satisfy Assumption~\ref{assumption_reset}, with $A \geq 4$ actions (including a special action $a^{\dagger} \in \cA$) and $S = 5$ states, where  $\cS \triangleq \{ s_0, g, x, \wt{s}, \ov{s} \} $, and all states apart from $\ov{s}$ are reliably $L$-reachable from $s_0$. 
We define as goal space $\cG \triangleq \{ g \}$. We consider the problem of learning an $\epsilon$-optimal goal-reaching policy with goal state $g$ from the starting state $s_0$, i.e., finding a policy $\pi$ such that $V^{\pi}(s_0 \rightarrow g) \leq V^{\star}(s_0 \rightarrow g) + \epsilon$, which is a sub-problem of the \MGE objective. 

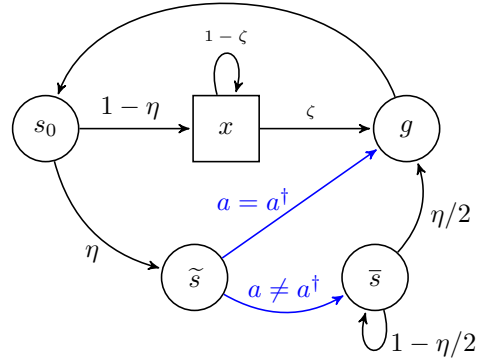
\begin{wrapfigure}[10]{r}{0.42\linewidth}
\flushright
\vspace{-0.4in}
\begin{tikzpicture}[->,>=stealth',shorten >=1pt,auto,node distance=2.8cm, semithick]
\tikzstyle{every state}=[fill=white,text=black]

\node[state]                (S0)                  {$s_0$};
\node[state,rectangle]      (X) [right=1.5cm of S0]    {$x$};
\node[state]                (G) [right=1.5cm of X]    {$g$};
\node[state, fill=white]    (ST) [below right of=S0] {$\wt{s}$};
\node[state, fill=white]    (OS) [below right of=X] {$\ov{s}$};


\path  
    (G)  edge[bend right=70]   (S0)
	(S0)  edge node[xshift=-0.5ex]{$1-\eta$} (X)
	(X)  edge node[xshift=-0.5ex]{{\tiny $\zeta$}} (G)
	(X)  edge [in=20,out=20,loop above,yshift=-1.5ex] node{{\tiny $1 - \zeta$}} (X)  
	(S0)  edge[bend right]   node[below,yshift=-0.5ex]{$\eta$} (ST)
	(ST)  edge[blue] node[yshift=-1.5ex]{$\textcolor{blue}{a = a^{\dagger}}$}  (G)
	(ST)  edge[bend right,blue] node{$\textcolor{blue}{a \neq a^{\dagger}}$}  (OS)
	(OS)  edge [loop below] node[right,xshift=0.5ex,yshift=0.5ex]{$1 - \eta/2$} (OS)  
	(OS)  edge [bend right] node[right,yshift=-0.5ex]{$\eta/2$} (G)  
	;

\end{tikzpicture}
 \vspace{-0.15in}
 \captionof{figure}{Illustration of the MDP instance~$\mathcal{M}_{a^{\dagger}}$.}
    \label{fig:mdp-separation}
\end{wrapfigure}%
Given $\eta \in (0, 1)$ and an unknown action $a^{\dagger} \in \cA$, we define the following transition probabilities for all $a \in \cA$, 
\begin{align*}
    &p(\wt{s} \vert s_0, a) \triangleq \eta, \quad \quad \quad \quad \quad p(x \vert s_0, a) \triangleq 1 - \eta, \\
    &p(g \vert x, a) \triangleq \zeta, \quad \quad \quad \quad \quad ~ p(x \vert x, a) \triangleq 1 - \zeta, \\
    &p(g \vert \wt{s}, a) \triangleq \mathds{1}[a = a^{\dagger}], \quad \quad p(\ov{s} \vert \wt{s}, a) \triangleq \mathds{1}[a \neq a^{\dagger}], \\
    &p(g \vert \ov{s}, a) \triangleq \frac{\eta}{2}, \quad \quad \quad \quad \quad \, p(\ov{s} \vert \ov{s}, a) \triangleq 1 - \frac{\eta}{2}, \\
    &p(s_0 \vert g, a) \triangleq 1,
\end{align*}
where we can set any $\zeta = O(1/L)$ such that $g$ is reliably $L$-reachable from $s_0$ (i.e., $V^{\star}(s_0 \rightarrow g)$). It is easy to see that the MDP's diameter verifies $D = \alpha \eta^{-1}$ for a constant $\alpha > 0$. Finally, we denote by $\mathcal{F}$ the family of MDPs of the form of Figure \ref{fig:mdp-separation} parameterized by $a^{\dagger} \in \{1, \ldots A\}$, i.e., $\mathcal{F} \triangleq \{ \mathcal{M}_{a^{\dagger}} \}_{a^{\dagger} \in \{1, \ldots A\}} $.

We define the event $\mathcal{J}_t$ that state $\wt{s}$ has never been visited by the agent by time $t$ (recalling that it initially starts at state $s_0$), i.e., $\mathcal{J}_t \triangleq \{ n^t(\wt{s}) = 0 \}$ (note that its probability is the same for all MDPs in $\mathcal{F}$). We now fix an MDP $\mathcal{M}_{a^{\dagger}}$ and denote by $\pi^{\star}$ its optimal policy, i.e., $\pi^{\star} \in \argmin_{\pi} V^{\pi}(\cdot \rightarrow g)$, which in particular selects action $a^{\dagger}$ at state $\wt{s}$. First, we consider any deterministic algorithm $\mathfrak{A}$ whose candidate policy $\wh{\pi}$ does not select action $a^{\dagger}$ when it is in state $\wt{s}$. Then it holds that
\begin{align}
     &V^{\wh{\pi}}(\wt{s} \rightarrow g) = 1 + \sum_{y \in \cS} p(y \vert s_0, \wh{\pi}(\wt{s})) V^{\wh{\pi}}(y \rightarrow g) \geq 1 + V^{\wh{\pi}}(\ov{s} \rightarrow g) = 1 + \frac{2}{\eta},  \label{eq_sep_1} \\
     &V^{\wh{\pi}}(\wt{s} \rightarrow g) -  V^{\pi^{\star}}(\wt{s} \rightarrow g) \geq \frac{2}{\eta} , \nonumber \\
    &V^{\pi}(s_0 \rightarrow g) = 1 + \eta V^{\pi}(\wt{s} \rightarrow g) + (1-\eta) V^{\pi}(x \rightarrow g), \quad \forall \pi, \nonumber \\ 
    &V^{\wh{\pi}}(s_0 \rightarrow g) - V^{\pi^{\star}}(s_0 \rightarrow g) = (1-\eta) \big(  \underbrace{ V^{\wh{\pi}}(x \rightarrow g) -  V^{\pi^{\star}}(x \rightarrow g)}_{\geq 0} \big) + \eta \big(   \underbrace{V^{\wh{\pi}}(\wt{s} \rightarrow g) -  V^{\pi^{\star}}(\wt{s} \rightarrow g) }_{\geq 2/\eta}\big) \geq 2 > \epsilon, \label{eq_sep_2}
\end{align}%
thus $\wh{\pi}$ has a sub-optimality gap larger than $\epsilon$. This means that under the event $\mathcal{J}_t$, where the algorithm cannot know which is the favorable action $a^{\dagger}$, it holds that with probability at least $1 - \frac{1}{A} \geq \frac{3}{4}$, any deterministic algorithm's candidate policy $\wh{\pi}$ does not select the action $a^{\dagger}$ at state $\wt{s}$ thus its value function is not $\epsilon$-optimal. Second, we note that we can easily extent to the case where $\mathfrak{A}$ outputs stochastic actions at state $\wt{s}$. Given that $a^{\dagger}$ is unknown, in the best case scenario it can set $\wh{\pi}(\cdot \vert \wt{s}) = 1/A$. Then we can retrace the reasoning above and replace \eqref{eq_sep_1} with $V^{\wh{\pi}}(\wt{s}) \geq 1 + \frac{A-1}{A} \frac{2}{\eta}$, and thus \eqref{eq_sep_2} with $V^{\wh{\pi}}(s_0) - V^{\pi^{\star}}(s_0) \geq 2 \frac{A-1}{A} > 1 \geq \epsilon$ since $A > 2$, which leads to the same result that $\wh{\pi}$ is not $\epsilon$-optimal on at least one of the MDPs in $\mathcal{F}$.

Now we study how the probability of the event $\mathcal{J}_t$ evolves over time $t$, i.e. we bound the time required to visit $\wt{s}$ at least once, which we denote by $\wt{T}$. Recall that $\wt{s}$ can only be reached  with probability $\eta$ from $s_0$ following any action. The random variable $\wt{T}$ can be seen as an upper bound of a random variable distributed geometrically with success probability $\eta$, thus Chernoff's inequality entails that with probability at least $\frac{1}{2}$ we have $\wt{T} \geq \frac{1}{9 \eta}$, i.e., the event $\mathcal{J}_{t}$ for $t = \lfloor 1 / (9 \eta) \rfloor$ holds with probability at least $\frac{1}{2}$.

Putting everything together, there exists an MDP in $\mathcal{F}$ such that with probability at least $\frac{1}{4}$, the number of time steps required to output a candidate policy that is $\epsilon$-optimal for goal state $g$ is at least $\frac{1}{9 \eta} = \Omega(D)$, where $\eta$ can be made arbitrarily small, so in particular $D$ can be exponentially larger than $L$. Hence in this MDP instance, no algorithm can solve \MGE in $\textrm{poly}(L)$ steps with overwhelming probability. While here we considered $S = 5$ for simplicity, note that the result easily holds for any $S \geq 5$ by replacing the state $x$ in the construction above with a set of $S - 4$ states; the only property that must be still verified is that $g$ remains reliably $L$-reachable from $s_0$.

Therefore, for any $L \geq 2, S \geq 5, A \geq 4, \epsilon \in (0, 1]$, there exists an MDP instance and a goal space for which any algorithm requires at least $\Omega(D)$ time steps to solve reset-free \MGE, where the diameter $D$ can be exponentially larger than $L, S, A, \epsilon^{-1}$, which concludes the proof of Lemma \ref{lemma_exp_sep}.


\def \histpi {\bm{\pi}}
\def \recpolicy {\widehat{\pi}}
\def \tsteps {\tau}
\newcommand{\prob}[2][]{ \mathbb{P}_{#1} \left[ #2 \right] }
\newcommand{\MDP}{\mathcal{M}}
\newcommand{\expect}[2][]{ \mathbb{E}_{#1} \left[ #2 \right] }

\subsection{Proof of Lemma \ref{lemma_LB}}

In the following, we will prove in Lemma \ref{BPI_SSP} a lower bound on the expected number of exploration steps to find an $\epsilon$-optimal SSP policy from $s_0$ for a specific goal state $g$ that is reliably $L$-reachable from $s_0$. Such BPI-SSP objective corresponds to our \MGE objective with goal space $\cG = \{ g \}$ and it thus induces a lower bound on the \MGE problem, which will conclude the proof of Lemma \ref{lemma_LB}.

\textit{Notation.} We largely follow the notations and definitions of \cite{domingues2021episodic}. We define an RL algorithm as a history-dependent policy $\histpi$ used to interact with the environment. In the BPI setting, where we eventually stop and recommend a policy, an algorithm is defined as a triple $(\histpi, \tsteps, \wh{\pi}_\tsteps)$ where $\tsteps$ is a stopping time and $\wh{\pi}_\tsteps$ is a Markov policy recommended after $\tsteps$ time steps. We now write more formally our objective. 

\begin{definition}[BPI-SSP] An algorithm $(\histpi, \tsteps, \recpolicy_\tsteps)$ is $(\epsilon, \delta,L)$-PAC for best-policy identification in a stochastic shortest path problem satisfying Assumption~\ref{assumption_reset} (BPI-SSP) if $V^{\star}(s_0) \leq L$ and if the policy $\recpolicy_\tsteps$ returned after $\tsteps$ time steps satisfies, for the initial state $s_0$,
\begin{align*}
    \prob[\histpi, \MDP]{V^{\recpolicy_\tsteps}(s_0) - V^*(s_0)\leq \epsilon } \geq 1-\delta,
\end{align*}
where we denote the goal state by $g$ and the SSP value of any policy $\pi$ at any state $s$ by $V^{\pi}(s) \triangleq \expect[\histpi, \MDP]{\inf \{ i \geq 0: s_{i+1} = g \}\,\vert\, s_1 = s}$, and $V^{\star}(s) \triangleq \min_{\pi} V^{\pi}(s)$.
\end{definition}

\begin{lemma}[BPI-SSP Lower Bound] \label{BPI_SSP} There exist absolute constants $L_0, S_0, A_0, \epsilon_0, \delta_0$ such that, for any $L \geq L_0,\, A\geq A_0 ,\, \epsilon \leq \epsilon_0,\, \delta\leq \delta_0$ and $S_0 \leq S \leq A^{\frac{L}{3}-2}$, and for any algorithm $(\histpi, \tsteps, \recpolicy_\tsteps)$ that is $(\epsilon, \delta,L)$-PAC for BPI-SSP in any finite MDP with $S$ states and $A$ actions, there exists an MDP $\MDP$ with a goal state belonging to $\SL$ and an absolute constant $\beta$ such that
\begin{align*}
    \expect[\histpi, \MDP]{\tsteps}
    \geq
    \beta
    \frac{L^3SA}{\epsilon^2} \log\left(\frac{1}{\delta}\right).
\end{align*}
\end{lemma}

\begin{proof}[Proof of Lemma \ref{BPI_SSP}] ~ 

\vspace{0.05in}

\newcommand{\geqODD}{\mathrel{\stackrel{\makebox[0pt]{\mbox{\normalfont\tiny \textcolor{pearDark}{(D)}}}}{\,\geq\,}}}

\noindent \begin{minipage}{0.53\linewidth}
We first define our family of hard MDPs for $S=4$ states, and the extension to any $S$ states can be done as in \cite{domingues2021episodic} as explained later. Consider the hard MDP illustrated in Figure~\ref{fig:mdp-lower-bound}, where all states incur a cost of $1$ apart from the goal state $s_g$ (a.k.a.\,``good'' state). The agent stays in the initial state $s_0$ with probability $1-q$, and goes to a \emph{decision state} $s_d$ with probability $q$. For any action $a$ taken in $s_d$, the agent reaches $s_g$ with probability $1/2$ and a ``bad'' state $s_b$ with probability $1/2$, except if an action $a^{\star}$ is chosen, that increases to $1/2+\wt{\epsilon}$ the probability of reaching $s_g$. From $s_b$, the agent returns to the initial state $s_0$ with probability $1$. The goal state $s_g$ is absorbing, and the agent stays there unless the reset action is taken, which brings the agent back to $s_0$. Note that the MDP satisfies Assumption~\ref{assumption_reset} (the arrows of the reset action from $s_d$ to $s_0$ and from $s_g$ to $s_0$ are not represented in Figure~\ref{fig:mdp-lower-bound} for visual convenience). Moreover, we define the following parameters
\begin{align*}
    H \triangleq \lceil \frac{L}{2} - 1 \rceil, \quad q \triangleq 1/ H, \quad \wt{\epsilon} \triangleq \frac{\epsilon}{2 (H+1)}.
\end{align*}
\end{minipage}%
\hfill%
\begin{minipage}{0.44\linewidth}
    \flushright
    \begin{tikzpicture}[->,>=stealth',shorten >=1pt,auto,node distance=2.8cm, semithick]
\tikzstyle{every state}=[fill=white,text=black]

\node[state]                (W)                  {$s_0$};
\node[state]                (A) [below=1cm of W]    {$s_d$};
\node[state, fill=white]    (B) [below right of=A] {$s_g$};
\node[state, fill=white]    (C) [below left  of=A] {$s_b$};

\node[state, fill=white,draw=none]    (RG) [right=0.25cm of B,xshift=-1.75ex] {cost = $0$};

\path  
	(W)  edge node{$q$} (A)
	(W)  edge [loop above] node{$1-q$} (W)  
	(A)  edge[blue]   node{$\textcolor{blue}{\frac{1}{2}+\wt{\epsilon}} $} (B)
	(A)  edge[bend right,dashed]   node[below]{$\frac{1}{2}$} (B)
	(A)  edge[bend left,dashed]   node[below]{$\frac{1}{2}$} (C)
	(A)  edge[blue]   node[above left,xshift=1ex]{$\textcolor{blue}{\frac{1}{2}-\wt{\epsilon}}$} (C)
	(B)  edge [loop below] node{1} (B)  
	(C)  edge [bend left] node{1} (W)    
	;

\end{tikzpicture}

    \vspace{-0.05in}
    
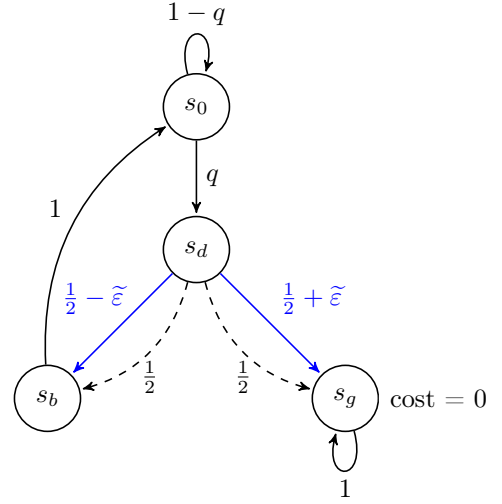
\captionof{figure}{Illustration of our hard MDP.} 
    \label{fig:mdp-lower-bound}
\end{minipage}%

\vspace{0.1in}
Note that this hard MDP instance is inspired from hard MDPs used in prior lower-bound constructions \citep[see e.g.,][]{lattimore2012pac, domingues2021episodic}, albeit with slight modifications. Indeed, a key difference with respect to the discounted MDP setting \citep{lattimore2012pac} or the finite-horizon MDP setting \citep{domingues2021episodic} is that in our case, the agent has access to an anytime reset action (Assumption~\ref{assumption_reset}). This implies that we cannot do as prior works that rely on absorption properties of states in the MDP (e.g., the ``good'' and ``bad'' states $s_b$ and $s_g$) in order to compound errors and add an extra effective horizon term (either $1/(1-\gamma)$ or $H$) in the sample complexity (i.e., to go from quadratic to cubic). The only absorption property we can rely on here is at the initial state $s_0$. It turns out that this will be sufficient in our setting to compound error and go from $L^2$ to $L^3$ dependence. The intuition for this is that the SSP value function generates a cost of $+1$ at each time step until the goal state is reached, which compounds errors more than the usual reward-based value function in the sparse-reward MDP constructions of \citet{lattimore2012pac, domingues2021episodic}.

We consider a family of MDPs of the form of Figure \ref{fig:mdp-lower-bound}, parameterized by $a^{\star} \in \{1, \ldots A\}$, where we denote by $\MDP_{a^{\star}}$ the MDP such that $a^{\star}$ increases the probability by $\wt{\epsilon}$ of reaching the goal state $s_g$ from state $s_d$. For any policy $\pi$ we denote its SSP value (with goal state $s_g$) at state $s$ in $\MDP_{a^{\star}}$ by $V_{a^{\star}}^{\pi}(s)$. 

We also define a reference MDP $\MDP_0$, where $\wt{\epsilon} = 0$, that is, there is no special action increasing the probability of reaching the goal state $s_g$. We denote by $\prob[a^{\star}]{\cdot}$ and $\expect[a^{\star}]{\cdot}$ the probability measure and the expectation in the MDP $\MDP_{a^{\star}}$ by following the algorithm $\histpi$ and by $\prob[0]{\cdot}$ and $\expect[0]{\cdot}$ the corresponding operators in $\MDP_0$.

The optimal value function does not depend on the MDP parameter $a^{\star}$, and for any MDP $\MDP_{a^{\star}}$ we get
\begin{align*}
    & V^{\star}(s_0) = 
    \frac{1}{q}
    +
    \Big(\frac{1}{2}+\wt{\epsilon}\Big)
    +
    \Big(1 - \frac{1}{2} -\wt{\epsilon}\Big)
    \big(1+ V^{\star}(s_0)\big) \\
    & \implies V^{\star}(s_0) =  \Big( \frac{1}{q} + 1 \Big)\frac{1}{1/2+\wt{\epsilon}}.
\end{align*}
Note that by our choice of $q$, it importantly holds that $V^{\star}(s_0) \leq L$, i.e., $s_g \in \SL$.

Meanwhile, the value function of the recommended policy $\recpolicy_\tsteps$ in $\MDP_{a^{\star}}$ is 
\begin{align*}
    V_{a^{\star}}^{\recpolicy_\tsteps}(s_0) =
    \Big( \frac{1}{q} + 1 \Big)\frac{1}{1/2+\wt{\epsilon} \cdot \recpolicy_\tsteps( a^{\star} \vert s_d) }.  
\end{align*}
As a result, 
\begin{align*}
     V_{a^{\star}}^{\recpolicy_\tsteps}(s_0) - V^{\star}(s_0) = \Big( \frac{1}{q} + 1 \Big) \frac{\wt{\epsilon}(1 - \recpolicy_\tsteps( a^{\star} \vert s_d))}{(1/2+\wt{\epsilon})(1/2+\wt{\epsilon} \cdot \recpolicy_\tsteps( a^{\star} \vert s_d))} \leq \underbrace{4 \Big( \frac{1}{q} + 1 \Big) \wt{\epsilon}}_{= 2 \epsilon} \cdot (1 - \recpolicy_\tsteps( a^{\star} \vert s_d)),
\end{align*}
and thus 
\begin{align*}
     V_{a^{\star}}^{\recpolicy_\tsteps}(s_0) - V^{\star}(s_0) < \epsilon \iff \recpolicy_\tsteps( a^{\star} \vert s_d) > \frac{1}{2}.
\end{align*}
We observe that this analysis of the suboptimality gap of $\recpolicy_\tsteps$ in terms of SSP value functions can be mapped to the one of \citet[][Proof of Theorem 7]{domingues2021episodic} for their finite-horizon value functions, despite the different MDP constructions. This means that we can retrace the steps of \citet[][Proof of Theorem 7]{domingues2021episodic}. In the following, we use the notation $\geqODD$ to signify that the inequality stems from following the same corresponding steps of \citet[][Proof of Theorem 7]{domingues2021episodic}. In particular, we similarly define the event 
\begin{align*}
    \mathcal{E}^{\tsteps}_{a^{\star}} \triangleq \Big\{ \recpolicy_\tsteps( a^{\star} \vert s_d) > \frac{1}{2} \Big\},
\end{align*}
and since the algorithm is assumed to be $(\epsilon, \delta, L)$-PAC for BPI-SSP for any MDP, we have 
\begin{align*}
    \mathbb{P}_{a^{\star}}\Big[ \mathcal{E}^{\tsteps}_{a^{\star}} \Big] = \mathbb{P}_{a^{\star}}\Big[ V_{a^{\star}}^{\recpolicy_\tsteps}(s_0) < V^{\star}(s_0) + \epsilon \Big] \geq 1 - \delta. 
\end{align*}
We now proceed with lower bounding the expectation of the sample complexity $\tsteps$ in the reference MDP $\MDP_0$. We define 
\begin{align*}
    N^{\tsteps}_{a^{\star}} \triangleq \sum_{t=1}^{\tsteps} \mathds{1}[S_t = s_d, A_t = a^{\star}], \quad N^{\tsteps} \triangleq \sum_{a^{\star}} N^{\tsteps}_{a^{\star}}.
\end{align*}
Following the steps of \citet[][Proof of Theorem 7]{domingues2021episodic}, we have that 
\begin{align*}
    \mathbb{E}_0\Big[ N^{\tsteps}_{a^{\star}} \Big] 4 \wt{\epsilon}^2 ~\geqODD ~ \mathbb{E}_0\Big[ N^{\tsteps}_{a^{\star}} \Big] \textrm{kl}\Big( \frac{1}{2},  \frac{1}{2} + \wt{\epsilon} \Big)     ~\geqODD~ \Big[ \Big(1 - \mathbb{P}_0 \Big[ \mathcal{E}^{\tsteps}_{a^{\star}}  \Big] \Big) \log\Big(\frac{1}{\delta}\Big) - \log(2) \Big],
\end{align*}
thus 
\begin{align*}
    \mathbb{E}_0\Big[ N^{\tsteps}_{a^{\star}} \Big]  &\geq \frac{1}{4 \wt{\epsilon}^2} \Big[ \Big(1 - \mathbb{P}_0 \Big[ \mathcal{E}^{\tsteps}_{a^{\star}}  \Big] \Big) \log\Big(\frac{1}{\delta}\Big) - \log(2) \Big].
\end{align*}
Summing all over MDP instances, we obtain following \citet[][Proof of Theorem 7]{domingues2021episodic} that
\begin{align*}
    \mathbb{E}_0 \Big[ N^{\tsteps} \Big] &= \sum_{a^{\star}}  \mathbb{E}_0 \Big[ N^{\tsteps}_{a^{\star}} \Big] 
    ~\geqODD~  \frac{A}{8 \wt{\epsilon}^2}\log\Big(\frac{1}{\delta}\Big).
\end{align*}

While the proof of \citet[][Theorem 7]{domingues2021episodic} can stop at this stage, our proof requires an additional step of linking this back to the sample complexity $\tsteps$, since the latter is not defined in terms of number of episodes but in terms of number of time steps. 

For any $i \leq \tsteps$, denote by $W_i(s_0 \rightarrow s_d)$ the random variable of the number of time steps required to reach $s_d$ starting from $s_0$ for the $i$-th time --- note importantly that this quantity is independent of the algorithm and of the MDP parameter $a^{\star}$, and we can write $\mathbb{E}\left[W_i(s_0 \rightarrow s_d) \right] = \mathbb{E}\left[W(s_0 \rightarrow s_d) \right] = \frac{1}{q}$. Then from Wald's equation we have 
\begin{align*}
     \mathbb{E}_0[ \tsteps ] \geq \mathbb{E}_0\left[ \sum_{i=1}^{N^{\tsteps}} W_i(s_0 \rightarrow s_d) \right]
                    = \mathbb{E}\left[W(s_0 \rightarrow s_d) \right] \cdot \mathbb{E}_0\left[N^{\tsteps} \right] &\geq \frac{1}{q} \cdot \alpha \frac{1}{\wt{\epsilon}^2} A \log\left( \frac{1}{\delta}\right)
                    \geq \alpha' \frac{L^3 A}{\epsilon^2} \log\left( \frac{1}{\delta}\right),
\end{align*}
where $\alpha$ and $\alpha'$ are absolute constants.

Finally, as mentioned, the extension to any $S$ to make appear the multiplicative dependence on $S$ can be done by following the steps done in \citet[][Proof of Theorem 7]{domingues2021episodic} (which relies on their Assumption~1, see their Theorem 10 for the relaxed statement). The idea of the construction is to consider not just $1$ decision state $s_d$ but $S-3$ of them, where only one of them possesses the favorable action~$a^{\star}$; intuitively this generates $S A$ actions instead of $A$ \citep[see e.g.,][Section 38.7]{lattimore2020bandit}, thus leading to the additional $S$ factor in the sample complexity.
\end{proof}

\section{DETAILS OF \ALGOnormal AND ANALYSIS} \label{app_TAB}

In this section, we focus on \ALGO. In Section~\ref{notation}, we introduce useful notation. In Section~\ref{subsect_choices}, we define the exact choice of estimates $\D,\,\E,\,\Q$ used by \ALGO (line~\ref{line_estimates_tabular} of Algorithm~\ref{algo}). Then in Section~\ref{app_proof_UB}, we provide the full proof of Theorem \ref{lemma_UB}, by establishing the key steps followed in Section~\ref{sect_theory}. Throughout the analysis we consider that $\cG = \cS$, $\epsilon \in (0, 1]$ and $\delta \in (0,1)$.

\subsection{Notation} \label{notation}

Given a goal state $g \in \cG$, denote by $\mathcal{M}_g$ the unit-cost SSP-MDP which adds a self-loop at $g$ to the original MDP $\mathcal{M}$, and denote by $p_g$ its transition function and $c_g$ its cost function. Formally, let 
    \begin{align*}
    c_{g}(s, a) \triangleq \mathds{1}[s \neq g], \quad p_{g}(s' \vert s,a) \triangleq \left\{
    \begin{array}{ll}
        p(s' \vert s,a) & \mbox{if\,} s \neq g \\
        \mathds{1}[s'=g] & \mbox{if\,} s = g.
    \end{array}
\right.
    \end{align*}
For any (possibly non-stationary) policy $\pi = (\pi_h)_{h \geq 1}$, let $V_g^{\pi}$ be its SSP value function (i.e., expected cost-to-go) in $\mathcal{M}_g$, i.e., 
    \begin{align*}
        V_g^{\pi}(s_0) \triangleq \mathbb{E}\left[ \sum_{h=1}^{+\infty} c_g(s_h, a_h) ~\big\vert~ s_1 = s_0, \pi, \mathcal{M}_g \right],
    \end{align*}
where $a_h \triangleq \pi_h(s_h)$ and $s_{h+1} \sim P_g(s_h, a_h)$. Let $\pi_g^{\star} \in \argmin_{\pi} V_g^{\pi}$ and $V_g^{\star} \triangleq V_g^{\pi_g^{\star}}$. We now define the set of finite-horizon goal-conditioned models. 

\begin{definition}[Finite-Horizon Goal-Conditioned Models] \label{def_multistate}
    Fix a horizon $H \geq 1$. For any goal state $g \in \cG$, denote by $\overline{\mathcal{M}}_{g,H}$ the finite-horizon model that corresponds to starting from state $s_0$ and interacting for $H$ steps with the original MDP $\mathcal{M}$ in which state~$g$ is made absorbing. $\overline{\mathcal{M}}_{g,H}$ admits as cost function $\overline{c}_g \triangleq c_g$ and as transition function $\overline{p}_g \triangleq p_g$.
\end{definition}

\begin{remark} Note that for a given goal state $g \in \cG$, the cost function $\overline{c}_{g}$ is known to the agent, while the MDP transitions $\overline{p}_{g}$ are unknown with some (known) goal-specific changes w.r.t.\,the original MDP $\mathcal{M}$ (namely, the self-loop at $g$). An alternative way of framing the problem is that there is one single MDP with state space $\mathcal{S} \times \mathcal{G}$, i.e., with state variable $(s,g)$.
\end{remark}

For a finite-horizon policy $\pi \in \Pi_H$, denote by $\ov{V}_{g,h}^{\pi}$ its finite-horizon value function at step $1 \leq h \leq H$ in the finite-horizon instance $\overline{\mathcal{M}}_{g,H}$, i.e., 
    \begin{align*}
        \ov{V}_{g,h}^{\pi}(s_0) \triangleq \mathbb{E}\left[ \sum_{h'=h}^H \overline{c}_g(s_{h'}, a_{h'}) ~\big\vert~ s_1 = s_0, \pi, \overline{\mathcal{M}}_{g,H}  \right].
    \end{align*}
We define the corresponding optimal value function as $\ov{V}_{g,h}^{\star} \triangleq \min_{\pi} \ov{V}_{g,h}^{\pi}$. Observe that $\ov{V}_{g,1}^{\star}(s_0) = \DHstar(g)$ (notation used in Properties 1 and 2 of Section~\ref{sect_overview}). 

Let $(s_i, a_i, s_{i+1})$ be the state, action and next state observed by an algorithm at time step~$i$. Let $n^{k}(s,a) \triangleq \sum_{i=1}^{k H} \mathds{1}[(s_i,a_i) = (s,a)]$ be the number of times state-action pair $(s,a)$ was visited in the first $k$ episodes and $n^{k}(s,a,s') \triangleq \sum_{i=1}^{k H} \mathds{1}[(s_i,a_i,s_{i+1}) = (s,a,s')]$. We define the empirical transitions as $\wh{p}^{\,k}(s' | s,a) \triangleq n^k(s,a,s') / n^k(s,a)$ if $n^k(s,a)>0$, and $\wh{p}^{\,k}(s' | s,a) \triangleq 1/S$ otherwise. Also, $p X(s, a) \triangleq \mathbb{E}_{s' \sim p(\cdot | s, a)} \left[X(s')\right]$ denotes the expectation operator w.r.t.\,the transition probabilities $p$ and $\pi_h Y(s) \triangleq  Y(s,\pi_h(s))$ denotes the composition with policy~$\pi$ at step $h$, so that $p \pi_h Z(s,a) \triangleq \mathbb{E}_{s' \sim p(\cdot | s, a)} \left[Z(s', \pi_h(s'))\right]$. Finally, we denote the clip function by $\clip(x,y,z) \triangleq \max(\min(x,z),y)$.

Finally, in the analysis we denote by $p_{g,h}^k(s,a)$ the probability of reaching state-action pair $(s,a)$ at step $h$ under policy $\pi_g^{k}$ in the true MDP. We also define the pseudo-counts as $\ov{n}^k(s,a) \triangleq \sum_{h=1}^H \sum_{l=1}^k p_{g_l,h}^l(s,a)$, where $g_l \in \cS$ denotes the goal state selected by \ALGO at the beginning of algorithmic episode $l$. 

\subsection{Algorithmic Choices of $\D,\,\E,\,\Q$ for \ALGO} \label{subsect_choices}

\newcommand{\QQQzero}{{\scalebox{0.90}{$\utQ_{g,h}^{0}$}}}
\newcommand{\VVV}{{\scalebox{0.90}{$\utV_{g,h+1}^k$}}}
\newcommand{\VVVhone}{{\scalebox{0.90}{$\utV_{g,1}$}}}
\newcommand{\QQQ}{{\scalebox{0.90}{$\utQ_{g,h}^{k}$}}}
\newcommand{\JJJ}{{\scalebox{0.90}{$\ov{V}_{g,1}^{\pi^{k+1}}(s_0)$}}}

We generalize to our goal-conditioned scenario the estimates used by \BPIUCBVI \citep{menard2021fast}, a recent algorithm designed for Best-Policy Identification (BPI) in finite-horizon non-stationary MDPs. First, we build the optimistic goal-conditioned Q-values and value functions in the finite-horizon models $\overline{\mathcal{M}}_{g,H}$ for $g \in \cS$ and $h \leq H$ as follows,  $\QQQzero(s,a)\triangleq \mathds{1}[s \neq g]$,
\begin{align*}
  \utQ_{g,h}^{k}(s,a) &\triangleq \clip\!\!\Bigg(\! \mathds{1}[s \neq g]  - 3\sqrt{\Var_{\hp^{\,k}}(\utV_{g,h+1}^k)(s,a) \frac{\betastar(n^k(s,a),\delta)}{n^k(s,a)}} - 14H^2 \frac{\beta(n^k(s,a),\delta)}{n^k(s,a)} \\ &\qquad~\quad -\frac{1}{H}\hp^{\,k} (\ltV_{g,h+1}^k - \utV_{g,h+1}^{k})(s,a)+ \hp^{\,k} \utV_{g,h+1}^{k}(s,a), ~0, ~H\!\!\Bigg),\\
  \utV_{g,h}^k(s) &\triangleq \min_{a\in\cA}\utQ_{g,h}^k(s,a),\quad \utV_{g,h}^k(g) \triangleq 0, \quad \utV_{g,H+1}^{k}(s) \triangleq0,
\end{align*}%
where we define the variance of $\VVV$ with respect to $\wh{p}^{\,k}(\cdot|s,a)$ as \begin{small}
$\Var_{\wh{p}^{\,k}}(\VVV)(s,a) \!\triangleq\!\! \sum_{s'}  \wh{p}^{\,k}(s'|s,a) \big(\VVV(s') -\wh{p}^k\VVV(s,a) \big)^{\!2}$,
\end{small}%
where $\beta(n,\delta) = \wt{O}(S\log(n/\delta))$ and $\beta^{\star}(n,\delta) = \wt{O}(\log(n/\delta))$ are some exploration thresholds and $\ltV_g^k$ is a pessimistic finite-horizon goal-conditioned value function; see Appendix \ref{app_proof_UB} for the complete definitions. Let $\pi_{g,h}^{k+1}$ be the greedy policy with respect to the lower bounds $\QQQ$. We recursively define the functions $\ov{U}_g^{k}$ for $g \in \cS$ and $h \leq H$ as follows, $\ov{U}_{g,h}^{\,0}(s,a)\triangleq H \mathds{1}[s \neq g]$,
\begin{align*}
  \ov{U}_{g,h}^{\,k}(s,a) &\triangleq \clip\!\!\Bigg(\!  6\sqrt{\Var_{\hp^k}(\utV_{g,h+1}^k)(s,a) \frac{\betastar(n^k(s,a),\delta)}{n^k(s,a)}} + 36 H^2 \frac{\beta(n^k(s,a),\delta)}{n^k(s,a)} \\ 
  &\qquad~\quad + \left(1+\frac{3}{H}\right)\hp^{\,k} \pi_{g,h+1}^{k+1} \ov{U}_{g,h+1}^k(s,a),0,H\!\Bigg), \\
  \ov{U}_{g,h}^{\,k}(g,a) &\triangleq 0, \quad \ov{U}_{g,H+1}^{\,k}(s,a) \triangleq 0,
\end{align*}%
We are now ready to define the distance and error estimates of \ALGO (Algorithm~\ref{algo}) as follows
\begin{align}
    \Q_{k,h}(s,a,g) &\triangleq \utQ_{g,h}^{k-1}(s,a), \label{eq_Qt} \\
    \D_k(g) &\triangleq \wt{\ov V}^{k-1}_{g,1}(s_0), \label{eq_Dt} \\
    \E_k(g) &\triangleq \pi_{g,1}^{k} \ov{U}_{g,1}^{k-1}(s_0) + \frac{\red{8}\epsilon}{\red{9}}. \label{eq_Et}
\end{align}

\subsection{Proof of Theorem \ref{lemma_UB}} \label{app_proof_UB}

In this section, we provide the full proof of Theorem \ref{lemma_UB}, by establishing the key steps followed in Section~\ref{sect_theory}. Appendices~\ref{proof_keystep1}, \ref{proof_keystep2} and \ref{proof_keystep3}, which prove ``key steps \ding{172}, \ding{173}, \ding{174}'', focus on more ``standard'' technical tools (e.g., high-probability events, variance-aware concentration inequalities); in particular building on the analysis of \citet{menard2021fast} on the sample complexity of BPI in finite-horizon MDPs and extending it to our goal-conditioned scenario. Then Appendix~\ref{proof_keystep4} proves ``key step \ding{175}'', which focuses on the technical novelty of the \AdaGoal goal selection scheme that is specific to the multi-goal exploration setting.

\newcommand{\cross}[1]{\textcolor{blue}{\cancelto{}{#1}}}
\renewcommand{\cross}[1]{}

We begin by stating a simple property that we will rely on throughout the analysis. 
\begin{lemma} \label{simple_lemma}
    For any state-action pair $(s,a) \in \cS \times \cA$ and goal state $g \in \cS$, consider any vector $Y \in \mathbb{R}^S$ such that $Y(g) = 0$, then $p Y(s,a) = p_g Y(s,a)$, where we recall that 
 \begin{align*}
    p_{g}(s' \vert s,a) \triangleq \left\{
    \begin{array}{ll}
        p(s' \vert s,a) & \mbox{if } s \neq g \\
        \mathds{1}[s'=g] & \mbox{if } s = g.
    \end{array}
\right.
\end{align*}
\end{lemma}
\begin{proof} It is easy to see that
\begin{align*}
    p Y(s,a) &= \sum_{s' \in \cS \setminus \{ g \}} p(s' \vert s,a) Y(s') + p(g \vert s,a) \underbrace{Y(g)}_{= 0} \\ &= \sum_{s' \in \cS \setminus \{ g \}} p_g(s' \vert s,a) Y(s') + p_g(g \vert s,a) \underbrace{Y(g)}_{= 0} \\ &= p_g Y(s,a).
\end{align*}
\end{proof}
Thanks to the above observation, our analysis will not require to handle goal-conditioned (true or empirical) transition probabilities, and will only need to deal with the (true or empirical) transition probabilities of the original MDP $\mathcal{M}$.

\subsubsection{Proof of ``key step \ding{172}''} \label{proof_keystep1}

\paragraph{Concentration events.} 
Here we define the high-probability event $\cU$ on which we condition our statements. We follow the notation of \citet[][Appendix A]{menard2021fast} and define the three following favorable events: $\cE$ the event where the empirical transition probabilities are close to the true ones, $\cE^\cnt$ the event where the pseudo-counts are close to their expectation, and $\cE^\star$ where the empirical means of the optimal goal-conditioned value functions are close to the true ones. Denoting by $\KL$ the Kullback-Leibler divergence, we set 
\begin{align*}
 \cE &\triangleq \left\{\forall k \in \N, \forall (s,a)\in\cS\times\cA:\ \KL\left(\widehat{p}^{\,k}(\cdot | s,a), p(\cdot | s,a)\right)\leq \frac{\beta(n^k(s,a),\delta)}{n^k(s,a)}\right\}\CommaBin \\
 \cE^{\cnt} &\triangleq  \left\{ \forall k \in \N, \forall (s ,a)\in\cS\times\cA:\ n^k(s,a) \geq \frac{1}{2}\bar n^k(s,a)-\beta^{\cnt}(\delta)  \right\}\CommaBin \\
 \cE^\star &\triangleq \Bigg\{\forall k \in \N, \forall h \in [H], \forall (s,a)\in\cS\times\cA, \forall g \in \cS:\\
  & \quad\quad\quad \big|(\hp^{\,k}-p)\ov{V}^{\star}_{g,h+1}(s,a)\big|\leq\min\left(\!H, \sqrt{2\Var_{p}(\ov{V}^{\star}_{g,h+1})(s,a)\frac{\betastar(n^k(s,a),\delta) }{n^k(s,a)}}+3 H \frac{\betastar(n^k(s,a),\delta) }{n^k(s,a)}\right)\Bigg\}\cdot
\end{align*}
We define the intersection of these events as
\begin{align}\label{eq_event_hp}
    \cU \triangleq \cE \cap \cE^{\cnt}\cap \cE^\star.
\end{align}
We prove that for the right choice of the functions~$\beta$ the above event holds with high probability.
\begin{lemma}
\label{lem:proba_master_event}
For the following choices of functions $\beta,$
\begin{align*}
  \beta(n,\delta) &\triangleq   \log(3S^2AH/\delta) + S\log \left(8e(n+1)\right),\\
  \beta^\cnt(\delta) &\triangleq \log\left(3S^2AH/\delta\right), \\
  \betastar(n,\delta) &\triangleq \log(3S^2AH/\delta) + \log\left(8e(n+1)\right),
\end{align*}
it holds that $\P(\cU)\geq 1-\delta$.
\end{lemma}
\begin{proof}
The only difference with respect to the concentration inequalities of \citet[][Appendix A]{menard2021fast} is that we need to take a union bound over the goal states $g \in \cS$ when concentrating our optimal goal-conditioned value functions. We thus set $\delta \leftarrow \delta / S$ in the choices of functions $\beta$ compared to \citet{menard2021fast}. As a result, by \citet[][Theorem 3 \& 4 \& 5]{menard2021fast} we have that $\P(\cE)\geq 1-\frac{\delta}{3}$, $\P(\cE^{\cnt})\geq 1-\frac{\delta}{3}$ and $\P(\cE^\star)\geq 1-\frac{\delta}{3}$, respectively. Applying a union to the above three inequalities, we conclude that $\P(\cU)\geq 1-\delta$.
\end{proof}

We recall the definitions of the functions $\ov{U}_g^{k}$, $\utQ_{g,h}^{k}$ and  $\utV_{g,h}^k$ in Appendix~\ref{subsect_choices}. They rely on the pessimistic finite-horizon goal-conditioned values $\ltV_g^k$ defined as
\begin{align*}
  \ltQ_{g,h}^{k}(s,a) &\triangleq \clip\!\!\Bigg(\! \mathds{1}[s \neq g] +  3\sqrt{\Var_{\hp^{\,k}}(\utV_{g,h+1}^k)(s,a) \frac{\betastar(n^k(s,a),\delta)}{n^k(s,a)}} + 14H^2 \frac{\beta(n^k(s,a),\delta)}{n^k(s,a)}\\ 
  &\qquad \qquad +\frac{1}{H}\hp^{\,k} (\ltV_{g,h+1}^k - \utV_{g,h+1}^{k})(s,a)+ \hp^{\,k} \ltV_{g,h+1}^{k}(s,a), ~0, ~H\!\!\Bigg),\\
  \ltV_{g,h}^k(s) \triangleq &\min_{a\in\cA}\ltQ_{g,h}^k(s,a),\quad \ltV_{g,h}^k(g) \triangleq 0, \quad \ltV_{g,H+1}^{k}(s) \triangleq0.
\end{align*}
Finally, we define the following quantities
\begin{align*}
  \rQ_{g,h}^{k}(s,a) &\triangleq \max\!\!\Bigg(\! \mathds{1}[s \neq g] + p \rV_{g,h+1}(s,a) , ~ \clip\!\!\Bigg(\! \mathds{1}[s \neq g] + 3\sqrt{\Var_{\hp^{\,k}}(\utV_{g,h+1}^k)(s,a) \frac{\betastar(n^k(s,a),\delta)}{n^k(s,a)}}\\ 
  &\qquad \qquad + 14H^2 \frac{\beta(n^k(s,a),\delta)}{n^k(s,a)} +\frac{1}{H}\hp^{\,k} (\ltV_{g,h+1}^k - \utV_{g,h+1}^{k})(s,a)+ \hp^{\,k} \utV_{g,h+1}^{k}(s,a), ~0, ~H\!\!\Bigg) ~\!\!\Bigg),\\
  \rV_{g,h}^k(s) &\triangleq \pi_{g,h}^{k+1} \rQ_{g,h}^{k}(s,a), \\
  \rV_{g,h}^k(g) &\triangleq 0, \\ 
  \rV_{g,H+1}^{k}(s) &\triangleq0.
\end{align*}

\textcolor{black}{We have the following property, which is the equivalent of \citet[][Lemma 6]{menard2021fast} and is proved likewise.}

\begin{lemma} \label{lemma_mathring_bound}
On the event $\cU$, for all $(s,a,g,h) \in \cS \times \cA \times \cS \times [H]$ and for every episode $k$, it holds that
\begin{align*}
  \rQ_{g,h}^k(s,a) &\geq \max\left(\ltQ_{g,h}^k(s,a), \ov{Q}_{g,h}^{\pi_g^{k+1}}(s,a)\right),\\
  \rV_{g,h}^k(s) &\geq \max\left(\ltV_{g,h}^k(s), \ov{V}_{g,h}^{\pi_g^{k+1}}(s) \right).
\end{align*}
\end{lemma}

We now derive ``key step \ding{172}'' by establishing that Properties 1 and 2 hold. Specifically, we show that \textbf{(i)} the functions $\VVVhone$ are optimistic estimates of the optimal goal-conditioned finite-horizon value functions and \textbf{(ii)} the functions $\ov{U}_{g,1}$ serve as valid upper bounds to the goal-conditioned finite-horizon gaps, as shown below.

\begin{lemma} \label{lemma_value_gap}
On the event $\cU$, it holds that for every episode $k$ and goal $g \in \cS$,
\begin{align*}
    \ov V^{\pi_{g,1}^{k+1}}_{g,1}(s_0) -  \ov V^{\star}_{g,1}(s_0) &\leq  \ov V^{\pi_{g,1}^{k+1}}_{g,1}(s_0) - \wt{\ov V}^k_{g,1}(s_0)  \\ 
    &\leq  \pi_{g,1}^{k+1} \ov{U}_{g,1}^k(s_0).
\end{align*}
\end{lemma}

\begin{proof}
On the event $\cU$, using Lemma \ref{lemma_mathring_bound}, we upper bound the goal-conditioned gap at episode~$t$ as
\begin{align*}
    \ov V^{\pi_{g,1}^{k+1}}_{g,1}(s_0) -  \ov V^{\star}_{g,1}(s_0) \leq  \ov V^{\pi_{g,1}^{k+1}}_{g,1}(s_0) - \wt{\ov V}^k_{g,1}(s_0) \leq \rV_{g,1}^k(s_0) - \wt{\ov V}^k_{g,1}(s_0).
\end{align*}
Next, \textcolor{black}{following the same reasoning as in \citet[][Proof of Lemma 2]{menard2021fast}}, we obtain by induction on $h$ that for all state-action pairs $(s,a)$ and goal states $g$, 
\begin{align} \label{eq:G_larger_gap_bounds}
    \rQ^k_{g,h}(s,a) - \wt{\ov Q}^k_{g,h}(s,a) \leq \ov{U}_{g,h}^k(s,a).
\end{align}
In particular for the initial layer $h=1$ and initial state $s = s_0$, we get that 
\begin{align*}
     \rV_{g,1}^k(s_0) - \wt{\ov V}^k_{g,1}(s_0) = \pi_{g,1}^{k+1} (\rQ^k_{g,1} - \wt{\ov Q}^k_{g,1})(s_0) \leq \pi_{g,1}^{k+1} \ov{U}_{g,1}^k(s_0).
\end{align*}

\end{proof}

\subsubsection{Proof of ``key step \ding{173}''} \label{proof_keystep2}

\begin{lemma} \label{lemma_bound_G}
On the event $\cU$, for every goal state $g\in \cS$ and episode $k$, it holds that
\begin{align*}
  \pi_{g,1}^{k+1} \ov{U}_{g,1}^k(s_0) \leq 24e^{13} H  \sqrt{  \sum_{h=1}^H \sum_{s,a} p_{g,h}^{k+1}(s,a) \frac{\betastar(\bn^k(s,a),\delta)}{\bn^k(s,a)\vee 1} } + 336 e^{13} H^2 \sum_{s,a} \left[ \sum_{h=1}^H  p_{g,h}^{k+1}(s,a) \frac{\beta(\bn^k(s,a),\delta)}{\bn^k(s,a)\vee 1  } \right] \wedge 1,
\end{align*}
where we recall that $p_{g,h}^{k+1}(s,a)$ denotes the probability of reaching $(s,a)$ at step $h$ under policy $\pi_g^{k+1}$.
\end{lemma}

\begin{proof}
Similar to \citet[][Steps 1 \& 2 in proof of Theorem 2]{menard2021fast}, we begin by upper-bounding $\ov{U}_{g,h}^k(s,a)$ for all $(s,a,h,g,k)$. If $n^k(s,a) > 0$, by definition of $\ovU_{g,h}^k$ we have that
\begin{align}
  \ovU_{g,h}^{\,k}(s,a) \leq  6\sqrt{\Var_{\hp^{\,k}}(\utV_{g,h+1}^k)(s,a) \frac{\betastar(n^k(s,a),\delta)}{n^k(s,a)}}+ 36 H^2 \frac{\beta(n^k(s,a),\delta)}{n^k(s,a)} + \left(1+\frac{3}{H}\right)\hp^{\,k} \pi_{g,h+1}^{k+1} \ovU_{g,h+1}^k(s,a).
  \label{eq:def_G_without_min}
\end{align}
We now replace the empirical transition probabilities with the true ones. Using the Bernstein-type technical inequality of \citet[][Lemma 10]{menard2021fast} and that $0 \leq \ovU_{g,h}^k \leq H$, we get
\begin{align*}
  (\hp^{\,k}-p) \pi_{g,h+1}^{k+1} \ovU_{g,h+1}^k(s,a) &\leq \sqrt{2\Var_{p}(\pi_{g,h+1}^{k+1} \ovU_{g,h+1}^k)(s,a) \frac{\beta(n^k(s,a),\delta)}{n^k(s,a)}} + \frac{2}{3}H \frac{\beta(n^k(s,a),\delta)}{n^k(s,a)}\\
  &\leq \frac{1}{H} p \pi_{g,h+1}^{k+1} \ovU_{g,h+1}^k(s,a) +  3 H^2\frac{\beta(n^k(s,a),\delta)}{n^k(s,a)}\CommaBin
\end{align*}
where in the last line we used $\Var_{p}(\pi_{h+1}^{k+1} \ovU_{g,h+1}^k)(s,a) \leq H \pi_{g,h+1}^{k+1} \ovU_{g,h+1}^k(s,a)$ and $\sqrt{xy}\leq x+y$ for all $x,y\geq 0$. We then replace the variance of the upper confidence bound under the empirical transition probabilities by the variance of the optimal value function under the true transition probabilities. Using the technical lemmas of \citet[][Lemma 11 \& 12]{menard2021fast} that control the deviation in variances w.r.t.\,the choice of transition probabilities, we obtain that
\begin{align*}
\Var_{\hp^{\,k}}(\utV_{g,h+1}^k)(s,a) &\leq 2\Var_{p}(\utV_{g,h+1}^k)(s,a) + 4 H^2 \frac{\beta(n^k(s,a),\delta)}{n^k(s,a)}\\
&\leq 4 \Var_{p}(\ovV_{g,h+1}^{\pi_g^{k+1}})(s,a) + 4 H p (\utV_{g,h+1}^k-\ovV_{g,h+1}^{\pi_g^{k+1}})(s,a)+4H^2 \frac{\beta(n^k(s,a),\delta)}{n^k(s,a)}\\
&\leq 4 \Var_{p}(\ovV_{g,h+1}^{\pi_g^{k+1}})(s,a) + 4 H p \pi_{g,h+1}^{k+1}\ovU_{g,h+1}^{k}(s,a)+4H^2 \frac{\beta(n^k(s,a),\delta)}{n^k(s,a)}\CommaBin
\end{align*}
where we used \eqref{eq:G_larger_gap_bounds} in the last inequality. Next, using $\sqrt{x+y}\leq \sqrt{x}+\sqrt{y}$, $\sqrt{xy}\leq x+y,$ and $\betastar(n, \delta)\leq \beta(n,\delta)$ leads to
\begin{align*}
\sqrt{\Var_{\hp^{\,k}}(\utV_{g,h+1}^k)(s,a) \frac{\betastar(n^k(s,a),\delta)}{n^k(s,a)}} &\leq 2\sqrt{ \Var_{p}(\ovV_{g,h+1}^{\pi_g^{k+1}})(s,a) \frac{\betastar(n^k(s,a),\delta)}{n^k(s,a)}}+ (2 H +4H^2) \frac{\beta(n^k(s,a),\delta)}{n^k(s,a)}\\
&\qquad+ \frac{1}{H} p \pi_{g,h+1}^{k+1}\ovU_{g,h+1}^{k}(s,a)\\
& \leq 2\sqrt{ \Var_{p}(\ovV_{g,h+1}^{\pi_g^{k+1}})(s,a) \frac{\betastar(n^k(s,a),\delta)}{n^k(s,a)}}+ 6H^2 \frac{\beta(n^k(s,a),\delta)}{n^k(s,a)}\\
&\qquad+ \frac{1}{H} p \pi_{g,h+1}^{k+1} \ovU_{g,h+1}^k(s,a).
\end{align*}
Combining these two inequalities with~\eqref{eq:def_G_without_min} yields
\begin{align*}
  \ovU_{g,h}^{\,k}(s,a) &\leq 12\sqrt{ \Var_{p}(\ovV_{g,h+1}^{\pi_g^{k+1}})(s,a) \frac{\betastar(n^k(s,a),\delta)}{n^k(s,a)}}+ 36H^2 \frac{\beta(n^k(s,a),\delta)}{n^k(s,a)}\\
  &\qquad+ \frac{6}{H} p \pi_{g,h+1}^{k+1} \ovU_{g,h+1}^k(s,a) + 36 H^2 \frac{\beta(n^k(s,a),\delta)}{n^k(s,a)}\\
  &\qquad+\left(1+\frac{3}{H}\right) \frac{1}{H} p \pi_{g,h+1}^{k+1} \ovU_{g,h+1}^k + \left(1+\frac{3}{H}\right) 3 H^2\frac{\beta(n^k(s,a),\delta)}{n^k(s,a)}\\
  &\qquad+  \left(1+\frac{3}{H}\right)p \pi_{g,h+1}^{k+1} \ovU_{g,h+1}^k(s,a)\\
  &\leq 12 \sqrt{ \Var_{p}(\ovV_{g,h+1}^{\pi_g^{k+1}})(s,a) \frac{\betastar(n^k(s,a),\delta)}{n^k(s,a)}} + 84 H^2 \frac{\beta(n^k(s,a),\delta)}{n^k(s,a)}\\
  &\qquad +  \left(1+\frac{13}{H}\right)p \pi_{g,h+1}^{k+1} \ovU_{g,h+1}^k(s,a).
\end{align*}
Since by construction, $\ovU_{g,h}^{\,k}(s,a)\leq H,$ we have that for all $n^k(s,a)\geq 0,$
\begin{align*}
  \ovU_{g,h}^{\,k}(s,a) & \leq 12 \sqrt{ \Var_{p}(\ovV_{g,h+1}^{\pi_g^{k+1}})(s,a) \left(\frac{\betastar(n^k(s,a),\delta)}{n^k(s,a)} \wedge 1 \right)} + 84 H^2 \left(\frac{\beta(n^k(s,a),\delta)}{n^k(s,a)}\wedge 1\right)\\
  &\qquad +  \left(1+\frac{13}{H}\right)p \pi_{g,h+1}^{k+1} \ovU_{g,h+1}^k(s,a).
\end{align*}
Unfolding the previous inequality and using $(1+13/H)^H\leq e^{13}$ we get
\begin{align*}
  \pi_{g,1}^{k+1} \ovU_{g,1}^k(s_0) & \leq 12e^{13} \sum_{h=1}^H \sum_{s,a} p_{g,h}^{k+1}(s,a) \sqrt{ \Var_{p}(\ovV_{g,h+1}^{\pi_g^{k+1}})(s,a) \left(\frac{\betastar(n^k(s,a),\delta)}{n^k(s,a)} \wedge 1 \right)}\\
   &\qquad + 84  e^{13} H^2 \sum_{h=1}^H \sum_{s,a} p_{g,h}^{k+1}(s,a) \left(\frac{\beta(n^k(s,a),\delta)}{n^k(s,a)}\wedge 1\right).
\end{align*}
Using that $\pi_{g,1}^{k+1} \ovU_{g,1}^k(s_0) \leq H$, we can clip the above bound as follows 
\begin{align} \label{eq_clip_U}
  \pi_{g,1}^{k+1} \ovU_{g,1}^k(s_0) & \leq 12e^{13} \sum_{h=1}^H \sum_{s,a} p_{g,h}^{k+1}(s,a) \sqrt{ \Var_{p}(\ovV_{g,h+1}^{\pi_g^{k+1}})(s,a) \left(\frac{\betastar(n^k(s,a),\delta)}{n^k(s,a)} \wedge 1 \right)} \nonumber\\
   &\qquad + 84  e^{13} H^2 \sum_{s,a} \left[ \sum_{h=1}^H p_{g,h}^{k+1}(s,a) \left(\frac{\beta(n^k(s,a),\delta)}{n^k(s,a)}\wedge 1\right) \right] \wedge 1. 
\end{align}
From the technical lemma of \citet[][Lemma 8]{menard2021fast} that relates counts to pseudo-counts, 
\begin{align*}
    \frac{\beta(n^k(s,a),\delta)}{n^k(s,a)} \wedge 1 \leq 4 \frac{\beta(\bn^k(s,a),\delta)}{\bn^k(s,a) \vee 1},
\end{align*}
thus we can replace the counts by the pseudo-counts in \eqref{eq_clip_U} as
\begin{align}
  \pi_{g,1}^{k+1} \ovU_{g,1}^k(s_0) & \leq 24e^{13} \sum_{h=1}^H \sum_{s,a} p_{g,h}^{k+1}(s,a) \sqrt{ \Var_{p}(\ovV_{g,h+1}^{\pi_g^{k+1}})(s,a) \frac{\betastar(\bn^k(s,a),\delta)}{\bn^k(s,a)\vee 1} }\nonumber\\
   &\qquad + 336 e^{13} H^2 \sum_{s,a} \left[ \sum_{h=1}^H  p_{g,h}^{k+1}(s,a) \frac{\beta(\bn^k(s,a),\delta)}{\bn^k(s,a)\vee 1  } \right] \wedge 1.  \label{eq:upper_bound_G1}
\end{align}
We now apply the law of total variance (see e.g.,\,\citealp{azar2017minimax} or \citealp[][Lemma 7]{menard2021fast}) in order to further upper-bound the first sum in~\eqref{eq:upper_bound_G1}. In particular, by Cauchy-Schwarz inequality, we obtain
\begin{align*}
  \sum_{h=1}^H \sum_{s,a} p_{g,h}^{k+1}(s,a) &\sqrt{ \Var_{p}(\ovV_{g,h+1}^{\pi_g^{k+1}})(s,a) \frac{\betastar(\bn^k(s,a),\delta)}{\bn^k(s,a)\vee 1} }  \\
  &\leq
  \sqrt{  \sum_{h=1}^H \sum_{s,a} p_{g,h}^{k+1}(s,a) \Var_{p}(\ovV_{g,h+1}^{\pi_g^{k+1}})(s,a) } \sqrt{  \sum_{h=1}^H \sum_{s,a} p_{g,h}^{k+1}(s,a) \frac{\betastar(\bn^k(s,a),\delta)}{\bn^k(s,a)\vee 1} }\\
  &\leq \sqrt{ \mathbb{E}_{\pi_g^{k+1}}\!\left[ \left(\sum_{h=1}^H \mathds{1}[s_h \neq g] - \ovV_{g,1}^{\pi_g^{k+1}}(s_0)\right)^2 \right] } \sqrt{  \sum_{h=1}^H \sum_{s,a} p_{g,h}^{k+1}(s,a) \frac{\betastar(\bn^k(s,a),\delta)}{\bn^k(s,a)\vee 1} }\\
  &\leq H  \sqrt{  \sum_{h=1}^H \sum_{s,a} p_{g,h}^{k+1}(s,a) \frac{\betastar(\bn^k(s,a),\delta)}{\bn^k(s,a)\vee 1} }\cdot
\end{align*}
Plugging this in~\eqref{eq:upper_bound_G1} concludes the proof of Lemma \ref{lemma_bound_G}.
\end{proof}

We are now ready to derive ``key step \ding{173}'' which controls the cumulative gap bounds, see Equation~\ref{eq_PE}. 

\begin{lemma}\label{bound_sum_gaps}
On the event $\cU$, for any number of episodes $K \geq 1$, it holds that
    \begin{align*}
\sum_{k=0}^{K-1} \pi_{g_{k+1},1}^{k+1} \ovU_{g_{k+1},1}^k(s_0) &\leq  48 e^{13} \sqrt{K}\sqrt{H^2 S A \log(H K +1) \betastar(K,\delta)} + 1344 e^{13} H^2 S A \beta(K,\delta) \log(H K+1) \\
  &\quad + 48 e^{13} H^2 S A \sqrt{ \betastar(K,\delta)}.
    \end{align*}
\end{lemma}

\begin{proof}
Plugging in the bound of Lemma \ref{lemma_bound_G} yields
\begin{align*}
  &\sum_{k=0}^{K-1} \pi_{g_{k+1},1}^{k+1} \ovU_{g_{k+1},1}^k(s_0) \\
  &\leq 24e^{13} H  \sum_{k=0}^{K-1} \sqrt{  \sum_{h=1}^H \sum_{s,a} p_{g_{k+1},h}^{k+1}(s,a) \frac{\betastar(\bn^k(s,a),\delta)}{\bn^k(s,a)\vee 1} } + 336 e^{13} H^2  \sum_{k=0}^{K-1}  \sum_{s,a} \left[ \sum_{h=1}^H  p_{g,h}^{k+1}(s,a) \frac{\beta(\bn^k(s,a),\delta)}{\bn^k(s,a)\vee 1  } \right] \wedge 1\\
  &\leq 24e^{13} H  \sqrt{ \betastar(H K,\delta) } \sum_{k=0}^{K-1} \sqrt{  \sum_{s,a}  \frac{\bn^{k+1}(s,a)-\bn^k(s,a)}{\bn^k(s,a)\vee 1} } + 336 e^{13} H^2 \beta(H K,\delta) \sum_{s,a}  \sum_{k=0}^{K-1} \left[ \frac{\bn^{k+1}(s,a)-\bn^k(s,a)}{\bn^k(s,a)\vee 1} \right] \wedge 1,
\end{align*}
where we used that $\beta(.,\delta)$ and $\betastar(.,\delta)$ are increasing. We define $\cJ \triangleq \{ k \in [0, K-1]: \bn^k(s,a) < \bn^{k+1}(s,a)-\bn^k(s,a) -1 \}$. Applying Lemma \ref{lem:pigeonhole} gives that
\begin{align*}
    &\sum_{k \in \cJ} \sqrt{ \sum_{s,a}   \frac{\bn^{k+1}(s,a)-\bn^k(s,a)}{\bn^k(s,a)\vee 1} } \leq  \sum_{s,a}   \underbrace{\sum_{k \in \cJ} \sqrt{ \frac{\bn^{k+1}(s,a)-\bn^k(s,a)}{\bn^k(s,a)\vee 1} }}_{\leq 2 H} \leq 2 S A H, \\
    &\sum_{k \notin \cJ} \sqrt{  \sum_{s,a}  \frac{\bn^{k+1}(s,a)-\bn^k(s,a)}{\bn^k(s,a)\vee 1} } \leq \sqrt{K} \sqrt{  \sum_{s,a} \underbrace{ \sum_{k \notin \cJ} \frac{\bn^{k+1}(s,a)-\bn^k(s,a)}{\bn^k(s,a)\vee 1} }_{\leq 4 \log(H K + 1)} } \leq 2 \sqrt{K S A \log(H K + 1)}, \\
    &\sum_{s,a} \underbrace{ \sum_{k=0}^{K-1}   \left[ \frac{\bn^{k+1}(s,a)-\bn^k(s,a)}{\bn^k(s,a)\vee 1} \right] \wedge 1 }_{\leq 4 \log(H K + 1)} \leq 4 S A \log(H K + 1).
\end{align*}
Putting everything together yields Lemma \ref{bound_sum_gaps}. \textcolor{black}{Note that as opposed to \citet{menard2021fast}, we are in the setting of stationary transition probabilities (and cost functions), which is why we are able to shave a factor $H$ in the main order term of the bound of the cumulative gap bounds (also recall that their sample complexity bound is in terms of exploration episodes and not exploration steps as ours).} 
\end{proof}

\begin{lemma}[Technical lemma] 
	\label{lem:pigeonhole}
	 For $T\in\N^\star$ and $(u_t)_{t\in\N^\star},$ for any sequence where $u_t\in[0,H]$ for some constant $H > 0$ and $U_t \triangleq \sum_{l=1}^t u_\ell$, let $\Omega \triangleq \{ t \in [0, T]: U_t < u_{t+1} -1\}$ and $\omega \triangleq \max\{ t \in \Omega\}$. Then it holds that
	 \begin{align*}
		&\sum_{t \in \Omega} \sqrt{\frac{u_{t+1}}{U_t\vee 1}} \leq 2 H, \\
		&\sum_{t \notin \Omega} \frac{u_{t+1}}{U_t\vee 1} \leq 4\log(U_{T+1}+1), \\
		&\sum_{t=0}^T \left[ \frac{u_{t+1}}{U_t\vee 1} \right] \wedge 1 \leq 4\log(U_{T+1}+1).
	\end{align*}
\end{lemma}
\begin{proof}
First, note that for any $t \in \Omega$, $\frac{u_{t+1}}{U_t \vee 1} \geq 1$, therefore
\begin{align*}
		\sum_{t \in \Omega} \sqrt{\frac{u_{t+1}}{U_t\vee 1}} \leq \sum_{t \in \Omega} \frac{u_{t+1}}{U_t\vee 1} \leq \sum_{t \in \Omega} u_{t+1} \leq \sum_{t=0}^{\omega} u_{t+1} = U_{\omega + 1} = U_{\omega} + u_{\omega} \leq u_{\omega+1} - 1 + u_{\omega} \leq 2 H.
\end{align*}
Second, if $t \notin \Omega$, then $2 U_t + 2 \geq U_{t+1} + 1$, therefore 
\begin{align*}
    \frac{u_{t+1}}{U_t\vee 1} \leq 4 \frac{u_{t+1} }{2U_t + 2} \leq 4 \frac{U_{t+1}-U_{t}}{U_{t+1} + 1},
\end{align*}
which yields that
\begin{align*}
	\sum_{t \notin \Omega} \frac{u_{t+1}}{U_t\vee 1} \leq 4\sum_{t \notin \Omega} \frac{U_{t+1}-U_{t}}{U_{t+1} + 1} \leq 4\sum_{t=0}^T \int_{U_t}^{U_{t+1}} \frac{1}{x+1} \mathrm{d}x \leq 4\log(U_{T+1}+1).
	\end{align*}
Third, combining the two cases above and further noticing that $4 \frac{U_{t+1}-U_{t}}{U_{t+1} + 1} =  \frac{4u_{t+1}}{U_t + u_{t+1} + 1} \geq  \frac{4u_{t+1}}{2 u_{t+1}} \geq 1$ for all $t\in \Omega$, it holds that 
\begin{align*}
    \left[ \frac{u_{t+1}}{U_t\vee 1} \right] \wedge 1 \leq 4 \frac{U_{t+1}-U_{t}}{U_{t+1} + 1},
\end{align*}
thus 
\begin{align*}
    \sum_{t=0}^T \left[ \frac{u_{t+1}}{U_t\vee 1} \right] \wedge 1 \leq 4 \sum_{t=0}^T \frac{U_{t+1}-U_{t}}{U_{t+1} + 1} \leq 4\sum_{t=0}^T \int_{U_t}^{U_{t+1}} \frac{1}{x+1} \mathrm{d}x \leq 4\log(U_{T+1}+1).
\end{align*}
\end{proof}

\subsubsection{Proof of ``key step \ding{174}''} \label{proof_keystep3}

\begin{lemma} \label{lemmaAPP_algo_sample_complexity_for_MGE}
The \MGE sample complexity $\tau$ of algorithm \ALGO can be bounded with probability at least $1-\delta$ by\footnote{We say that $f = O(g)$ if there exists an absolute constant $\iota > 0$ (independent of the MDP instance) such that $f \leq \iota g$.} 
\begin{align*}
 \tau = O\Bigg( \,&\frac{L^3 S A}{\epsilon^2} \cdot \log^3\Big(\frac{L S A }{\epsilon \delta}\Big) \cdot \log^3\Big(\frac{L}{\epsilon}\Big) + \frac{L^3 S^2 A}{\epsilon} \cdot \log^3\Big(\frac{L S A }{\epsilon \delta}\Big) \cdot \log^3\Big(\frac{L}{\epsilon}\Big) \,\Bigg). 
\end{align*}
\end{lemma}

\begin{proof}
    
We assume that the event $\cU$ holds and fix a (finite) episode $K < \kappa$, where $\kappa$ denotes the (possibly unbounded) episode index at which \ALGO terminates. For any $k \leq K$, denote by $g_k$ the goal selected by \ALGO at the beginning of episode $k$, then by design of the stopping rule~\eqref{stopping_rule} and by choice of prediction error~$\E_k$~\eqref{eq_Et}, it holds that 
\[
 \epsilon \leq \pi_{g_{k},1}^{k} \ovU_{g_{k},1}^{k-1}(s_0) + \frac{8 \epsilon}{9}\,.
\]
By summing the previous inequality for all $k \leq K$ and plugging in the bound of Lemma \ref{bound_sum_gaps}, we get that
\begin{align*}
  \frac{\epsilon}{9} K &\leq 48 e^{13} \sqrt{K}\sqrt{H^2 S A \log(H K +1) \betastar(K,\delta)} + 1344 e^{13} H^2 S A \beta(K,\delta) \log(H K+1) \\
  &\quad + 48 e^{13} H^2 S A \sqrt{ \betastar(K,\delta)}.
\end{align*}
We assume that $\kappa>1$ otherwise the result is trivially true. Defining $\kappa' \triangleq \min\{\kappa-1, K\}$ and using the definition of $\beta$ and $\betastar$ given in Lemma \ref{lem:proba_master_event}, we get the following functional inequality in $\kappa'$
\begin{align*}
    \epsilon \kappa' &\leq x_1 \sqrt{ \kappa' H^2 S A \log(H \kappa') \big(\log(3S^2AH/\delta) + \log(16e\kappa')\big) } 
    + x_2 H^2 S^2 A \log(3S^2AH/\delta)  \log^2(H \kappa'),
\end{align*}
for some absolute constants $x_1, x_2$. There remains to invert the above inequality to obtain an upper bound on $\kappa'$. We use the auxiliary inequality of Lemma \ref{lemma_aux_ineq} instantiated with scalars $B = x_1 \sqrt{H^2 SA \log(3S^2AH/\delta)}/\epsilon$, $C = x_2 H^2 S^2 A \log(3S^2AH/\delta) / \epsilon$ and $\alpha = 16 e H$. This yields that 
\begin{align} \label{eq_bound_kappap}
    \kappa' \leq O\left( \frac{H^2 SA}{\epsilon^2}  \log^3 \left(\frac{H S A}{ \epsilon \delta }\right)  + \frac{H^2 S^2 A}{\epsilon} \log^3\left( \frac{H S A}{ \delta \epsilon} \right) \right). 
\end{align}
Since \eqref{eq_bound_kappap} holds for $\kappa' = \min\{\kappa-1, K\}$ for any finite $K < \kappa$, letting $K \rightarrow + \infty$ implies that $\kappa$ is finite and bounded as in \eqref{eq_bound_kappap}.

The last step is to relate the above bound on the number of algorithmic episodes $\kappa$ to the \MGE sample complexity of \ALGO denoted by $\tau$. Since the algorithmic episodes are of length $H$ and separated by a one-step execution of the reset action $\areset$, it holds that $\tau \leq (H+1) \kappa$. We finally plug in the choice of horizon $H \triangleq \lceil 5(L+\red{2})\log\big(\red{10}(L+\red{2})/\epsilon\big)/\log(2) \rceil$ to conclude the proof of Lemma \ref{lemmaAPP_algo_sample_complexity_for_MGE}. 
\end{proof}

\begin{lemma}[An auxiliary inequality] \label{lemma_aux_ineq}
For any positive scalars $B,C \geq 1$ and $\alpha\geq e$, it holds for any $X \geq 2$ that 
\begin{align*}
    X \leq B \sqrt{X} \log(\alpha X) + C \log^2(\alpha X) \implies X \leq O\big( B^2 \log^2(\alpha B) + C \log^2( \alpha C)  \big).
\end{align*}
\end{lemma}
\begin{proof}
    On the one hand, assume that $X \leq B \sqrt{X} \log(\alpha X)$, then $\frac{X}{2} \leq - \frac{X}{2} + B \sqrt{X} \log(\alpha X)$. From the technical lemma of \citet[][Lemma 8]{kazerouni2017conservative}, $ - \frac{X}{2} + B \sqrt{X} \log(\alpha X) \leq \frac{32 B^2}{9} \left[ \log( 4 B \sqrt{\alpha}e ) \right]^2$, thus $X \leq \frac{64 B^2}{9} \left[ \log( 4 B \sqrt{\alpha}e ) \right]^2$. 
    On the other hand, assume that $X \leq C \log^2(\alpha X)$. Using that $\log(x) \leq x^{\beta} / \beta $ for all $x \geq 0$, $\beta > 0$, we get $X \leq C (8 \alpha^{1/8} X^{1/8})^2 \leq 64 C \alpha^{1/4} X^{1/4}$, thus $X \leq (64 C)^{4/3} \alpha^{1/3}$, thus $X \leq C \log^2( 64 \alpha^{4/3} C^{4/3})$.
    Now, assume that $ X \leq B \sqrt{X} \log(\alpha X) + C \log^2(\alpha X)$. Then $X \leq 2 \max\{ B \sqrt{X} \log(\alpha X),  C \log^2(\alpha X) \}$. From above we can bound each term separately, which concludes the proof.
\end{proof}


\subsubsection{Proof of ``key step \ding{175}''} \label{proof_keystep4}

We finally establish ``key step \ding{175}'', which focuses on the technical novelty of the \AdaGoal goal selection scheme that is specific to the multi-goal exploration setting. 

Recall for any $H\in \mathbb{N}^*$ the definition of the finite-horizon  MDP $\bcM_{g,H} = \{H,S,A,\overline{p}_g,\overline{c}_g\}$, where we recall that $\ov{p}_g = p_g$ and $\ov{c}_g = c_g$. We denote by $\bpistar_{g,H}$ the optimal policy in $\bcM_{g,H}$ as well as $\bV_{g,H,h}(s)$ the optimal value function starting from state $s$ at step $h$. We also define $\bp_{g,H}^{\bpi_{g,H}}(s_H\neq g | s_1 = s)$ the probability of reaching state $g$ starting from state $s$ with the policy $\pi\in\Pi_H$ in the MDP $\bcM_{g,H}$. When it is clear from the context we drop the dependence on the horizon $H$ in the previous notations. 

The following lemma controls the probability of not reaching a goal in $\SLepsilon$ with the optimal policy in the finite-horizon reduction MDP.

\begin{lemma}
\label{lem:control_terminal_prob}
For $g\in \SLepsilon$, for all $H \geq 2(L+2)$, for all $s \in \cS$,
\begin{align*}
    \bp_{g,H}^{\bpistar_{g,H}}(s_H\neq g | s_1 = s) &\leq e^{-\log(2) H/\big(4(L+2)\big)}\,.
\end{align*}
\end{lemma}

\begin{proof}
By induction it holds
\begin{align*}
    \bp_{g,H}^{\bpistar_{g,H}}(s_H\neq g | s_1 = s) &= \sum_{s'\neq g} \bp_{g,H}^{\bpistar_{g,H}}(s_{H-M+1}= s' | s_1 = s) \bp_{g,H}^{\bpistar_{g,H}}(s_H\neq g | s_{H-M+1} = s')\\
    &\leq \bp_{g,H}^{\bpistar_{g,H}}(s_{H-M+1} \neq  g | s_1 = s) \max_{s'} \bp_{g,H}^{\bpistar_{g,H}}(s_H\neq g | s_{H-M+1} = s')\\
    &\leq \prod_{j=0}^{\lfloor H/M \rfloor} \max_{s'} \bp_{g,H}^{\bpistar_{g,H}}(s_{H-j M}\neq g | s_{H-(j+1)M+1} = s')\,.
\end{align*}
Then thanks to the Markov inequality and the optimal Bellman equations solved by $\bpistar_{g,H}$ we obtain
\begin{align*}
   \bp_{g,H}^{\bpistar_{g,H}}(s_{H-j M}\neq g | s_{H-(j+1)M+1} = s') &\leq  \bp_{g,H}^{\bpistar_{g,H}}(s_{H}\neq g | s_{H-(j+1)M+1} = s')\\
   &\leq \frac{\bVstar_{g,H,H-(j+1)+1}(s')}{M}\\
   &= \frac{\bVstar_{g,(j+1)M,1}(s')}{M} \\
   &\leq  \frac{\bVstar_{g,(j+1)M,1}(s')}{M}\leq \frac{\Vstar_{g}(s')}{M} \leq \frac{L+2}{M}\,,
\end{align*}
where the last inequality uses the existence of a resetting action (Assumption~\ref{assumption_reset}) and the fact that $g \in \SLepsilon$ with $\epsilon \leq 1$. Choosing $M = 2(L+2)$ allows us to conclude
\[
 \bp_{g,H}^{\bpistar_{g,H}}(s_H\neq g | s_1 = s) \leq e^{-\lfloor H/M \rfloor \log(2)} \leq e^{-\log(2) H/\big(4(L+2)\big)}\,,
\]
where in the last inequality we used $\lfloor x \rfloor \geq x/2$ for $x\geq 1$. 
\end{proof}

We define the class of non-stationary, infinite-horizon policies that perform the reset action whenever the goal state is not reached after $H$ steps. 

\begin{definition}[Resetting policies]
    For any $\pi$, we denote by $\pi^{\vert H}$ the non-stationary policy that, until the goal is reached, successively executes the actions prescribed by $\pi$ for $H$ steps and takes action $\areset$, i.e., at time step $i$ and state $s$ it executes the following action:
    \begin{align*}
        \pi^{\vert H}(a \vert s, i) \triangleq \left\{
    \begin{array}{ll}
        \areset & \mbox{if } i \equiv 0\ (\textrm{mod}\ H+1), \\
        \pi(a \vert s, i) & \mbox{otherwise.}
    \end{array}
    \right.
    \end{align*}
We denote by $\Pi^{\vert H}$ the set of such resetting policies.
\end{definition}

We now establish two key lemmas. First, we show that, equipped with a near-optimal policy for the finite-horizon model $\bcM_{g,H}$, expanding it into an infinite-horizon policy via the reset provides a near-optimal goal-reaching policy in the original MDP $\mathcal{M}_g$ \textit{as long as the goal state $g$ belongs to $\SOL$ and the horizon $H$ is large enough}.

\begin{lemma}\label{lem:link_FH_SSP}
For $g\in \SLepsilon$ and $H \geq 5(L+\red{2})\log\big(\red{10}(L+\red{2})/\epsilon\big)/\log(2)$, 
it holds that 
\[
\bV^{\star}_{g,1}(s_0) \leq V^{\star}_{g}(s_0) \leq \bV^{\star}_{g,1}(s_0) + \frac{\epsilon}{\red{9}}\,,
\]
and if a policy $\tpi$ is $\epsilon\red{/9}$-optimal in $\bcM_{g,H}$ then 
\[
V_{g}^{\tpi^{|H}}(s_0) \leq V^{\star}_{g}(s_0) + \red{\epsilon}\,.
\]
\end{lemma}

\begin{proof}

We have trivially $\bVstar_{g,H,1}(s_0) \leq \Vstar_g(s_0) $. Thanks to Lemma~\ref{lem:control_terminal_prob} it holds
\begin{equation}
\label{eq:upper_bound_qH_1}
\bq_{g,H}^\star \triangleq \bp_{g,H}^{\bpistar_{g,H}}(s_{H+1}\neq g|s_1=s_0) \leq \bp_{g,H}^{\bpistar_{g,H}}(s_{H}\neq g|s_1=s_0) \leq  \frac{\epsilon}{10(L+2)}\leq \frac{1}{30}\,.
\end{equation}
Thanks to~\eqref{eq:upper_bound_qH_1} and the definition of a resetting policy, we can conclude that
\begin{align*}
    \Vstar_g (s_0)\leq V_g^{{\bpistar_{g,H}}^{|H}}(s_0) &= \bVstar_{g,H,1}(s_0) + \bq_{g,H}^\star(1+V_g^{{\bpistar_{g,H}}^{|H}}(s_0))\\
    &= \bVstar_{g,H,1}(s_0) +\frac{\bq_{g,H}^\star}{1-\bq_{g,H}^\star} (1+\bVstar_{g,H,1}(s_0))\\
    &\leq \bVstar_{g,H,1}(s_0) + \frac{30}{29} \frac{\epsilon}{10(L+2)} (1 + L + \epsilon) \\
    &\leq \bVstar_{g,H,1}(s_0) + \frac{\epsilon}{9}\,.
\end{align*}
Thus it holds that
\begin{equation}
\label{eq:bound_finite_horizon}
    \bVstar_{g,H,1}(s_0) \leq \Vstar_{g}(s_0) \leq \bVstar_{g,H,1}(s_0) + \frac{\epsilon}{9}\,.
\end{equation} 
For the second part of the lemma, first note that 
\begin{align*}
    \bV_{g,H,1}^{\tpi}(s_0) = \sum_{h=1}^H \bp_{g,H}^{\tpi}(s_h\neq g|s_1=s_0)  \geq \bV_{g,H-L,1}^{\tpi}(s_0) + L \bp_{g,H}^{\tpi}(s_H \neq g|s_1=s_0)\,,
\end{align*}
where in the inequality we used that $\bp_{g,H}^{\tpi}(s_h\neq g|s_1=s_0) \geq \bp_{g,H}^{\tpi}(s_H\neq g|s_1=s_0)$. Using successively the fact that $\tpi$ is $\frac{\epsilon}{9}$-optimal in $\bcM_{g,H}$, the inequality above and~\eqref{eq:bound_finite_horizon} we obtain
\begin{align*}
\Vstar_g(s_0)+\frac{\epsilon}{9} &\geq \bVstar_{g,H,1}(s_0) + \frac{\epsilon}{9} \\
&\geq \bV_{g,H,1}^{\tpi}(s_0)\\
&\geq \bV_{g,H-L,1}^{\tpi}(s_0) + L \bp_{g,H}^{\tpi}(s_{H} \neq g|s_0) \\
&\geq \Vstar_g(s_0)-\frac{\epsilon}{9} +L \bp_{g,H}^{\tpi}(s_{H} \neq g|s_0) \,.
\end{align*}
The previous sequence of inequalities entails that $\bp_{g,H}^{\tpi}(s_{H} \neq g) \leq (2\epsilon)/(9 L)$. Now we can upper bound the value of the resetting extension of $\tpi$. Indeed, for $\tq \triangleq \bp_{g,H}^{\tpi}(s_{H+1} \neq g|s_1=s_0) \leq \bp_{g,H}^{\tpi}(s_{H} \neq g|s_1=s_0)$ we have using that $\tpi$ is $\frac{\epsilon}{9}$-optimal in $\bcM_{g,H}$ with $g \in \SLepsilon$ that
\begin{align*}
 V_g^{\tpi^{\vert H}}(s_0)  &= \bV_{g,H,1}^{\tpi}(s_0) + \frac{\tq}{1-\tq}\big(1+\bV_{g,H,1}^{\tpi}(s_0)\big) \\
 &\leq \bVstar_{g,H,1}(s_0) + \frac{\epsilon}{9}+ \frac{2\epsilon}{9 L} \frac{1}{1 - 2/9} \left( 1 + L + \epsilon + \frac{\epsilon}{9}\right) \\
 &\leq  \Vstar_{g}(s_0) + \epsilon\,.
\end{align*}

\end{proof}

The second key lemma that we prove is that any goal state that meets the constraint \eqref{opt_pbm_constrained_constraint} with small enough prediction error \eqref{opt_pbm_constrained} must belong to $\SLepsilon$.

\newtheorem*{T100}{Restatement of Lemma \ref{lemma_EandD_gives_g_SLeps}}
\begin{T100}
\textit{\lemEandDgivesgSLeps}
\end{T100}

\begin{proof}
    Consider that the event $\mathcal{U}$ defined in \eqref{eq_event_hp} holds. Consider a goal state $g$ such that $\D_k(g) \leq L$ and $\E_k(g) \leq \epsilon$ at an episode $k \geq 1$. Then 
    \begin{align}
    \ov V^{\star}_{g,H,1}(s_0) &\myineqi \wt{\ov V}^k_{g,H,1}(s_0) + \pi_{g,1}^{k+1} \ov{U}_{g,1}^k(s_0) \nonumber \\
    &\myeqii \D_k(g) + \E_k(g) - \frac{8 \epsilon}{9} \nonumber \\
    &\myineqiii L + \frac{\epsilon}{9}, \label{eq_useful}
\end{align}
where (i) comes from Lemma \ref{lemma_value_gap}, (ii) stems from the choice of $\D$ and $\E$ estimates and (iii) comes from the conditions on $g$. Following the steps of the proof of Lemma \ref{lem:control_terminal_prob}, we have that
\begin{align*}
    \bp_{g,H}^{\bpistar_{g,H}}(s_H\neq g | s_1 = s) &\leq \prod_{j=0}^{\lfloor H/M \rfloor} \max_{s'} \bp_{g,H}^{\bpistar_{g,H}}(s_{H-j M}\neq g | s_{H-(j+1)M+1} = s'),
\end{align*}
and that
\begin{align*}
   \bp_{g,H}^{\bpistar_{g,H}}(s_{H-j M}\neq g | s_{H-(j+1)M+1} = s') &\leq  \frac{\bVstar_{g,(j+1)M,1}(s')}{M}\leq \frac{\bVstar_{g,H,1}(s')}{M} \leq \frac{1 + \bVstar_{g,H,1}(s_0)}{M} \leq \frac{1 + L + \epsilon/9}{M},
\end{align*}
where the before last inequality uses the existence of a resetting action (Assumption~\ref{assumption_reset}) and the last inequality uses \eqref{eq_useful}. Choosing $M = 2(L+2)$ gives
\begin{align} \label{eq_AAA}
 \bp_{g,H}^{\bpistar_{g,H}}(s_H\neq g | s_1 = s) \leq e^{-\log(2) H/\big(4(L+2)\big)}.
\end{align}
We now follow the steps of the proof of Lemma \ref{lem:link_FH_SSP}. Thanks to \eqref{eq_AAA} and the choice of $H$ it holds
\begin{equation}
\label{eq:upper_bound_qH_2}
\bq_{g,H}^\star \triangleq \bp_{g,H}^{\bpistar_{g,H}}(s_{H+1}\neq g|s_1=s_0) \leq \bp_{g,H}^{\bpistar_{g,H}}(s_{H}\neq g|s_1=s_0) \leq  \frac{\epsilon}{10(L+2)}\leq \frac{1}{30}\,.
\end{equation}
Thanks to~\eqref{eq:upper_bound_qH_2} and the definition of a resetting policy, we obtain that
\begin{align*}
    \Vstar_g (s_0)\leq V_g^{{\bpistar_{g,H}}^{|H}}(s_0) &= \bVstar_{g,H,1}(s_0) + \bq_{g,H}^\star(1+V_g^{{\bpistar_{g,H}}^{|H}}(s_0))\\
    &= \bVstar_{g,H,1}(s_0) +\frac{\bq_{g,H}^\star}{1-\bq_{g,H}^\star} (1+\bVstar_{g,H,1}(s_0))\\
    &\myineqi L + \frac{\epsilon}{9} + \frac{30}{29} \frac{\epsilon}{10(L+2)} \big(1 + L + \frac{\epsilon}{9}\big) \\
    &\leq L + \frac{2 \epsilon}{9},
\end{align*}
where (i) uses \eqref{eq_useful}. Therefore we have that $g \in \SLepsilon$, which concludes the proof.
\end{proof}

We now have all the tools to prove that when \ALGO terminates, it fulfills the \MGE objective of Definition~\ref{def_MGE}.

\begin{lemma} \label{lemmaAPP_algo_PAC_for_MGE}
If the algorithm \ALGO stops, it is $(\epsilon, \delta, L)$-PAC for \MGE.
\end{lemma}

\begin{proof}
Consider that the event $\cU$ defined in \eqref{eq_event_hp} holds, and that the algorithm \ALGO has stopped at episode $\kappa$. Recall that we define $\D_{\kappa}(g) \triangleq \wt{\ov V}^{\kappa}_{g,1}(s_0)$ and $\E_{\kappa}(g) \triangleq \pi_{g,1}^{{\kappa}+1} \ov{U}_{g,1}^{\kappa}(s_0) + \frac{8\epsilon}{9}$. Denoting $\X_{\kappa} \triangleq \{ g \in \cS : \D_{\kappa}(g) \leq L\}$, the stopping rule (Equation~\ref{stopping_rule}) implies that $\max_{g \in \X_{\kappa}} \E_{\kappa}(g) \leq \epsilon$. We now prove that $\SL \subseteq \cX_{\kappa} \subseteq \SLepsilon$. On the one hand, it holds that
\begin{align*}
    \D_{\kappa}(g) = \wt{\ov V}^{\kappa}_{g,1}(s_0) \leq \ov V^{\star}_{g,1}(s_0) \leq V^{\star}_{g}(s_0), 
\end{align*}
which ensures that $\SL \subseteq \cX_{\kappa}$. On the other hand, consider that $g \in \cX_{\kappa}$, then $\D_{\kappa}(g) \leq L$ and $\E_{\kappa}(g) \leq \epsilon$, which implies that $g \in \SLepsilon$ from Lemma \ref{lemma_EandD_gives_g_SLeps}, therefore $\cX_{\kappa} \subseteq \SLepsilon$.

We now prove that the candidate policies of \ALGO are near-optimal goal-reaching policies. Consider any $g \in \cX_{\kappa}$. Combining the result of Lemma \ref{lemma_value_gap} and Equation~\ref{stopping_rule}, we obtain that 
\begin{align*}
    \ov V^{\pi_{g,1}^{\kappa + 1}}_{g,1}(s_0) -  \ov V^{\star}_{g,1}(s_0) \leq \pi_{g,1}^{\kappa+1} \ov{U}_{g,1}^{\kappa}(s_0) \leq \E_{\kappa}(g) - \frac{8\epsilon}{9} \leq \frac{\epsilon}{9},
\end{align*}
thus the policy $\pi_{g}^{\kappa + 1}$ is $\frac{\epsilon}{9}$-optimal in $\bcM_{g,H}$. As a result, denoting by $\wh{\pi}_g \triangleq (\pi^{\kappa + 1}_g)^{|H}$ the candidate policy of \ALGO, we have from Lemma \ref{lem:link_FH_SSP} that
\[
V_{g}^{\wh{\pi}_g}(s_0) \leq V^{\star}_{g}(s_0) + \epsilon\,,
\]
i.e., $\wh{\pi}_g$ is $\epsilon$-optimal for the original SSP objective. Putting everything together, we have that
\begin{align*}
      \mathbb{P}\left( \left\{ \SL \subseteq \cX_{\kappa} \subseteq \SLepsilon  \right\} ~ \cap ~ \left\{ \forall g \in \cX_{\kappa}, ~ V^{\wh{\pi}_g}(s_0 \rightarrow g) - V^{\star}(s_0 \rightarrow g) \leq \epsilon \right\} \right) \geq \mathbb{P}(\cU) \geq 1 - \delta.
\end{align*}
which ensures that \ALGO is $(\epsilon, \delta, L)$-PAC for \MGE. 
\end{proof}


\subsection{Putting everything together}

\newtheorem*{T1}{Restatement of Theorem \ref{lemma_UB}}
\begin{T1}
\textit{\thmMGE}
\end{T1} 

\begin{proof}
    The result comes from combining Lemmas \ref{lemmaAPP_algo_sample_complexity_for_MGE} and \ref{lemmaAPP_algo_PAC_for_MGE}.
\end{proof}

\newpage

\section{DETAILS OF \ALGOLMnormal AND ANALYSIS} \label{app_LFA}

In this section, we provide details on the \ALGOLM algorithm and the guarantee of Theorem \ref{SC_algolm} which bounds its \MGE sample complexity in linear mixture MDPs. Recall that since the state space $\cS$ may be large, we consider that the known goal space is in all generality a subset of it, i.e., $\cG \subseteq \cS$, where $G \triangleq \abs{\cG}$ denotes the cardinality of the goal space. 

First of all, we extend the linear mixture definition (Definition~\ref{def_linearmixture}) to handle our multi-goal setting. For any goal $g \in \cG$, we define
\begin{align*}
    p_g(s'|s,a) \triangleq \la \phi_g(s' \vert s,a), \theta_g^{\star} \ra,  \qquad \textrm{with} \qquad \theta_g^{\star} \triangleq \begin{pmatrix}
           \theta^{\star} \\
           1 
         \end{pmatrix} \in \mathbb{R}^{d+1}, \qquad 
         \phi_g(s' \vert s,a) \triangleq \begin{pmatrix}
         \mathds{1}[s \neq g]  \phi(s' \vert s,a) \\
          \mathds{1}[s = g] \mathds{1}[s' = g] 
         \end{pmatrix}  \in \mathbb{R}^{d+1}.
\end{align*}
We see that by construction, 
 \begin{align*}
    p_g(s'|s,a) = \left\{
    \begin{array}{ll}
        p(s' \vert s,a) & \mbox{if } s \neq g \\
        \mathds{1}[s'=g] & \mbox{if } s = g.
    \end{array}
\right.
    \end{align*}

\subsection{Overview of \ALGOLMnormal and Choice of $\E,\,\D,\,Q$ in line \ref{line_estimates_LFA} of Algorithm \ref{algo}}

Here we focus on the specificities of \ALGOLM in the linear mixture MDP setting (refer to Section~\ref{sect_overview} for the description of the algorithmic structure that is common to \ALGO), i.e., we explain how to define the estimates $\D,\,\E,\,\Q$ in line \ref{line_estimates_LFA} of Algorithm~\ref{algo}. At a high level, \ALGOLM uses two regression-based goal-conditioned estimators of the unknown parameter vector $\btheta_{g}^{\star}$ of each goal $g \in \cG$:
\begin{itemize}
    \item \textit{Value-targeted estimator.} The first estimator minimizes a ridge regression problem with the target being the past value functions. This is similar to the \UCRLVTR algorithm for linear mixture MDPs \citep{ayoub2020model} and follow-up work \citep[e.g.,][]{zhou2021nearly,zhang2021reward}. This step is used to compute the distance estimates $\D$ (and $\Q$) for \AdaGoal.
    \item \textit{Error-targeted estimator.} The second estimator is novel and minimizes a ridge regression problem with the target being past ``error functions'', that are computable upper bounds on the goal-conditioned gaps. This step is used to compute the prediction errors $\E$ for \AdaGoal.
\end{itemize}

\noindent $\triangleright$~\textit{Value-targeted estimator.} First, \ALGOLM builds a goal-conditioned estimator $\btheta_{g}$ for the unknown parameter vector $\btheta_{g}^{\star}$ of each goal $g \in \cG$, as well as a goal-conditioned covariance matrix $\bSigma_{g}$ of the feature mappings, which characterizes the uncertainty of the estimator $\btheta_{g}$. Similar to \UCRLVTR, $\btheta_{g}$ is computed as the minimizer to a ridge regression problem with the target being the past value functions, i.e.,
\begin{align*}
    \btheta_{g,k+1} \leftarrow \argmin_{\btheta \in \mathbb{R}^{d+1}} \lambda \| \btheta\|^2 + \sum_{k'=1}^{k} \sum_{h=1}^H \left( \big\la \btheta, \bphi_{V_{g,k',h}}(s_h^{k'}, a_h^{k'}) \big\ra - V_{g,k',h}(s_{h+1}^{k'}) \right),
\end{align*}
which has a closed-form solution given in \eqref{eq_closed_form_1}. Leveraging $\btheta_{g}$ and subtracting an exploration bonus term, \ALGOLM builds optimistic goal-conditioned estimators $Q_{g,k,h}(\cdot, \cdot)$ \eqref{eq_LFA_Q} and $V_{g,k,h}(\cdot)$ \eqref{eq_LFA_V} for the optimal action-value and value functions $\ov{Q}^{\star}_{g,h}(\cdot, \cdot)$ and $\ov{V}^{\star}_{g,h}(\cdot)$. The associated goal-conditioned policy is the greedy policy of the calculated optimistic $Q$-values \eqref{eq_LFA_pi}.

\vspace{0.1in}
\noindent $\triangleright$~\textit{Error-targeted estimator.} The main addition compared to existing works on linear mixture MDPs is that \ALGOLM also builds goal-conditioned errors denoted by $U_{g,k,h}$ \eqref{eq_LFA_U} that upper bound the goal-conditioned gaps (defined as the difference between the value function of the current greedy policy and the optimistic value estimates). They rely on an additional estimator $\mathring{\btheta}_{g,k}$ and covariance matrix $\mathring{\bSigma}_{g,k}$ based on the errors $\{U_{g,k',h} \}_{k' \leq k-1,h}$, instead of the values $\{V_{g,k',h} \}_{k' \leq k-1,h}$ as considered before. Specifically, $\mathring{\btheta}_{g}$ minimizes the ridge regression problem with contexts $\bphi_{U_{g,k',h}}(s_h^{k'}, a_h^{k'})$ and targets $U_{g,k',h}(s_{h+1}^{k'})$, i.e., 
\begin{align*}
    \mathring{\btheta}_{g,k+1} \leftarrow \argmin_{\btheta \in \mathbb{R}^{d+1}} \lambda \| \btheta\|^2 + \sum_{k'=1}^{k} \sum_{h=1}^H \left( \big\la \btheta, \bphi_{U_{g,k',h}}(s_h^{k'}, a_h^{k'}) \big\ra - U_{g,k',h}(s_{h+1}^{k'}) \right),
\end{align*}
which has a closed-form solution given in \eqref{eq_closed_form_2}.

\vspace{0.1in}
\noindent $\triangleright$~\textit{Algorithmic notation and updates.}

Let $\pnorm$ be an upper bound of the $\ell_2$-norm of $\btheta^{\star}$ (see Definition~\ref{def_linearmixture}) and set as regularization parameter $\lambda \triangleq 1 / (\pnorm+1)^2$. Also define the confidence radius 
\begin{align} \label{eq_conf_radius_LM}
\beta_k \triangleq H \sqrt{d \log(3(1+k H^3 (B+1)^2) / \delta)} + 1.
\end{align}
At the first episode indexed by $k=1$, we initialize for every goal $g \in \cG$ and $h \in [H]$ the following quantities $$\bSigma_{g,1,h}, \mathring{\bSigma}_{g,1,h} \triangleq \lambda\Ib, \quad \bbb_{g,1,h}, \mathring{\bbb}_{g,1,h}  \triangleq \zero, \quad \btheta_{g,1}, \mathring{\btheta}_{g,1,h} \triangleq\zero, \quad \vvalue_{g,1,H+1}(\cdot) \triangleq 0, \quad U_{g,1,H+1}(\cdot) \triangleq 0.$$
We now explain how the various estimates are updated during an episode $k$ with goal state denoted by $g_k$. Over the trajectory of episode $k$, given the current state visited at step $h$ denoted by $s_h^k$, the executed action is denoted by $a_h^k  \triangleq \pi_{g_k,h}^k(s_h^k)$ and the next state is denoted by $s_{h+1}^k$. Then for every goal $g \in \cG$ and for $h = 1, \ldots, H$, we set
\begin{align*}
    \bSigma_{g,k,h+1} &\triangleq \bSigma_{g,k,h} + \bphi_{{\vvalue}_{g,k,h}}(s_h^k,a_h^k)\bphi_{{\vvalue}_{g,k,h}}(s_h^k,a_h^k)^\top, \\
    \bbb_{g,k,h+1} &\triangleq \bbb_{g,k,h} + \bphi_{{\vvalue}_{g,k,h}}(s_h^k,a_h^k) \vvalue_{g,k,h}(s_{h+1}^k), \\
    \mathring{\bSigma}_{g,k,h+1} &\triangleq \mathring{\bSigma}_{g,k,h} + \bphi_{{U}_{g,k,h}}(s_h^k,a_h^k)\bphi_{{U}_{g,k,h}}(s_h^k,a_h^k)^\top, \\
    \mathring{\bbb}_{g,k,h+1} &\triangleq \mathring{\bbb}_{g,k,h} + \bphi_{{U}_{g,k,h}}(s_h^k,a_h^k) U_{g,k,h}(s_{h+1}^k),
\end{align*}
and for every goal $g \in \cG$, we set 
\begin{align}
    \bSigma_{g,k+1,1} &\triangleq \bSigma_{g,k,H+1}, \quad \bbb_{g,k+1,1}\triangleq \bbb_{g,k,H+1}, \quad \btheta_{g,k+1} \triangleq \bSigma^{-1}_{g,k+1,1} \bbb_{g,k+1,1}, \label{eq_closed_form_1} \\
    \mathring{\bSigma}_{g,k+1,1} &\triangleq \mathring{\bSigma}_{g,k,H+1}, \quad \mathring{\bbb}_{g,k+1,1}\triangleq \mathring{\bbb}_{g,k,H+1}, \quad \mathring{\btheta}_{g,k+1} \triangleq \mathring{\bSigma}^{-1}_{g,k+1,1} \mathring{\bbb}_{g,k+1,1}. \label{eq_closed_form_2}
\end{align}
We proceed by recursively defining for every episode $k$, goal  $g \in \cG$ and $h = H, \ldots, 1$,
\begin{align}
    \qvalue_{g,k,h}(\cdot, \cdot)&\triangleq \clip\Big(\mathds{1}[\cdot \neq g] + \big\la \btheta_{g,k}, \bphi_{\vvalue_{g,k,h+1}}(\cdot, \cdot) \big\ra - \beta_k \Big\| \bSigma_{g,k,1}^{-1/2} \bphi_{\vvalue_{g,k,h+1}}(\cdot, \cdot)\Big\|_2, \, 0, \, H\Big), \label{eq_LFA_Q} \\
    \pi_{g,h}^{k}(\cdot) &\triangleq \argmin_{a \in \cA}\qvalue_{g,k,h}(\cdot, a), \label{eq_LFA_pi} \\
    \vvalue_{g,k,h}(\cdot) &\triangleq \min_{a \in \cA}\qvalue_{g,k,h}(\cdot, a), \label{eq_LFA_V} \\
    U_{g,k,h}(\cdot) &\triangleq \clip\Big( 2 \beta_k
    \Big\|\bSigma_{g,k,1}^{-1/2}\bphi_{\vvalue_{g,k,h+1}}(s, \pi_{g,h}^k(s))\Big\|_2
    + \big\la \bphi_{U_{g,k,h+1}}(s, \pi_{g,h}^k(s)), \mathring{\btheta}_{g,k} \big\ra \nonumber \\
    &\qquad \qquad + \beta_k \Big\|\mathring{\bSigma}_{g,k,1}^{-1/2} \bphi_{U_{g,k,h+1}}(s, \pi_{g,h}^k(s))\Big\|_2, \, 0, \, H\Big). \label{eq_LFA_U}
\end{align}

\vspace{0.1in}
\noindent $\triangleright$~\textit{Choice of estimates $\D,\,\E,\,\Q$ of \ALGOLM (line \ref{line_estimates_LFA} of Algorithm~\ref{algo}):}
\begin{align}
    \Q_{k,h}(s,a,g) &\triangleq \qvalue_{g,k,h}(s,a), \label{eq_Qt_LFA} \\
    \D_k(g) &\triangleq \vvalue_{g,k,1}(s_0), \label{eq_Dt_LFA} \\
    \E_k(g) &\triangleq U_{g,k,1}(s_0) 
    + \frac{\red{8}\epsilon}{\red{9}}. \label{eq_Et_LFA}
\end{align}

\subsection{Proof sketch of Theorem \ref{SC_algolm}}

The analysis of \ALGOLM follows the same key steps considered in Section \ref{sect_theory} for the analysis of \ALGO. We now sketch the \ALGOLM equivalent of the various key steps. 

First, note that similar to Lemma \ref{simple_lemma} in the tabular case, for any state-action pair $(s,a) \in \cS \times \cA$, goal state $g \in \cG$ and vector $Y \in \mathbb{R}^S$ such that $Y(g) = 0$, it holds that $[p Y](s,a) = [p_g Y](s,a)$.

We now build the high-probability events. By using the standard self-normalized concentration inequality for vector-valued martingales of \citet[][Theorem 2]{abbasi2011improved}, it holds that with probability at least $1-\delta/3$, for any $k \geq 1$ and $g \in \cG$, $\btheta_{g}^{\star}$ lies in the ellipsoid 
\begin{align*}
    \mathcal{C}_{g,k} \triangleq \left\{ \btheta \in  \mathbb{R}^{d+1} : \Big\|\bSigma_{g,k,1}^{1/2} (\btheta_{g,k} -\btheta )\Big\|_2 \leq \beta_k \right\}.
\end{align*}
The proof of the above statement follows the steps of e.g., \citet[][Lemma A.2]{zhang2021reward}, the only slight difference being that we take an additional union bound over all goals $g \in \cG$, hence the presence of $G$ in the confidence radius \eqref{eq_conf_radius_LM}. Furthermore, following the exact same steps as above and by definition of $\mathring{\bSigma}$ and $\mathring{\btheta}$, it holds that with probability at least $1-\delta/3$, for any $k \geq 1$ and $g \in \cG$, $\btheta_{g}^{\star}$ lies in the ellipsoid 
\begin{align*}
    \mathcal{C}'_{g,k} \triangleq \left\{ \btheta \in  \mathbb{R}^{d+1} : \Big\|\mathring{\bSigma}_{g,k,1}^{1/2} (\mathring{\btheta}_{g,k} -\btheta )\Big\|_2 \leq \beta_k \right\}.
\end{align*}
Here the confidence radius is the same as the one in $\mathcal{C}_{g,k}$ since it is chosen to be proportional to the magnitude of the $U_{g,k,h+1}(\cdot)$ function, which lies in $[0, H]$, as does the value function $V_{g,k,h+1}(\cdot)$. In what follows, we assume that the two high-probability events considered above hold, i.e., that the following event holds (it does so with probability at least $1-2\delta/3$)
\begin{align} \label{eq_lfa_event_hp}
    \left\{ \forall k \geq 1, \forall g \in \cG, ~ \btheta_{g}^{\star} \in  \mathcal{C}_{g,k} \cap \mathcal{C}'_{g,k} \right\}.
\end{align}

\paragraph{$\triangleright$~\textit{Key step \ding{172}: Optimism and gap bounds.}} The optimism property is standard: following e.g., \citet[][Lemma A.1]{zhang2021reward} (see also \citealp[][Lemma C.4]{zhou2021nearly}), it holds that $Q_{g,k,h}(s,a) \leq \ov{Q}_{g,h}^{\star}(s,a)$ and $V_{g,k,h}(s) \leq \ovV_{g,h}^{\star}(s)$ for any $(s, a, g) \in \cS \times \cA \times \cG$, $h \in [H]$, $k \geq 1$. 

We now depart from a usual regret minimization analysis and examine our errors $U$, proving that they upper bound the goal-conditioned gaps, formally defined as 
\begin{align*}
    W_{g,k,h}(s) \triangleq \vvalue_{g,h}^{\pi_g^k}(s) - \vvalue_{g,k,h}(s).
\end{align*}
We now prove by induction that $W_{g,k,h}(s) \leq U_{g,k,h}(s)$. The property holds at $H+1$ since $W_{g,k,h}(s) = 0 = U_{g,k,h}(s)$. Assume that $W_{g,k,h+1}(s) \leq U_{g,k,h+1}(s)$, then we start by noticing, similar to \citet[Equation~C.10]{zhou2021nearly}; \citet[][Lemma A.1]{zhang2021reward}, that  
\begin{align*}
    W_{g,k,h}(s) \leq 2 \beta_k
    \Big\|\bSigma_{g,k,1}^{-1/2}\bphi_{\vvalue_{g,k,h+1}}(s, \pi_{g,h}^k(s))\Big\|_2 + \underbrace{[p\vvalue_{g,h+1}^{\pi^k}](s, \pi_{g,h}^k(s)) - [p\vvalue_{g,k,h+1}](s, \pi_{g,h}^k(s))}_{\triangleq X}.
\end{align*}
We bound $X$ as follows 
\begin{align*}
    X &=  [p W_{g,k,h+1}](s, \pi_{g,h}^k(s)) \\
    &\myineqi   [p U_{g,k,h+1}](s, \pi_{g,h}^k(s)) \\
    &= \big\la \bphi_{U_{g,k,h+1}}(s, \pi_{g,h}^k(s)), \btheta_{g}^{\star} \big\ra \\
    &= \big\la \bphi_{U_{g,k,h+1}}(s, \pi_{g,h}^k(s)), \mathring{\btheta}_{g,k} \big\ra + \big\la \bphi_{U_{g,k,h+1}}(s, \pi_{g,h}^k(s)), \btheta_{g}^{\star} - \mathring{\btheta}_{g,k} \big\ra \\
    &\myineqii \big\la \bphi_{U_{g,k,h+1}}(s, \pi_{g,h}^k(s)), \mathring{\btheta}_{g,k} \big\ra 
    + \Big\|\mathring{\bSigma}_{g,k,1}^{1/2} (\mathring{\btheta}_{g,k} -\btheta_{g}^{\star} )\Big\|_2 \Big\| \mathring{\bSigma}_{g,k,1}^{-1/2} \bphi_{U_{g,k,h+1}}(s, \pi_{g,h}^k(s))\Big\|_2 \\
    &\myineqiii \big\la \bphi_{U_{g,k,h+1}}(s, \pi_{g,h}^k(s)), \mathring{\btheta}_{g,k} \big\ra + \beta_k \Big\|\mathring{\bSigma}_{g,k,1}^{-1/2} \bphi_{U_{g,k,h+1}}(s, \pi_{g,h}^k(s))\Big\|_2,
\end{align*} 
where (i) comes from the induction hypothesis and because $p$ is a monotone operator w.r.t.\,the partial ordering of functions, (ii) is by Cauchy-Schwarz, (iii) holds by event \eqref{eq_lfa_event_hp}. Finally, using that $W_{g,k,h}(s) \in [0, H]$ and by definition of $U_{g,k,h}(s)$, we conclude that $W_{g,k,h}(s) \leq U_{g,k,h}(s)$.

\paragraph{$\triangleright$~\textit{Key step \ding{173}: Bounding the cumulative gap bounds.}} We now bound $\sum_{k=1}^K U_{g_k,k,1}(s_0)$. It holds that
\begin{align*}
    &U_{g,k,h}(s_{k,h}) - U_{g,k,h+1}(s_{k,h+1}) \\
    &\leq 2 \beta_k \min\Big\{ 1,
    \Big\|\bSigma_{g,k,1}^{-1/2}\bphi_{\vvalue_{g,k,h+1}}(s_{k,h}, \pi_{g,h}^k(s_{k,h}))\Big\|_2 \Big\} + \beta_k \min\Big\{ 1,
    \Big\|\mathring{\bSigma}_{g,k,1}^{-1/2}\bphi_{U_{g,k,h+1}}(s_{k,h}, \pi_{g,h}^k(s_{k,h}))\Big\|_2 \Big\} \\ 
    &\qquad + \min\Big\{ \underbrace{\big\la \bphi_{U_{g,k,h+1}}(s_{k,h}, \pi_{g,h}^k(s_{k,h})), \mathring{\btheta}_{g,k} \big\ra - U_{g_k,k,h+1}(s_{k,h+1})}_{\triangleq Y}, H \Big\},
\end{align*}
where 
\begin{align*}
    Y &\leq \big\vert \big\la \bphi_{U_{g,k,h+1}}(s_{k,h}, \pi_{g,h}^k(s_{k,h})), \btheta_{g}^{\star} - \mathring{\btheta}_{g,k} \big\ra \big\vert + \big\la \bphi_{U_{g,k,h+1}}(s_{k,h}, \pi_{g,h}^k(s_{k,h})), \btheta^{\star}_{g} \big\ra - U_{g,k,h+1}(s_{k,h+1}) \\ 
    &\leq \Big\|\mathring{\bSigma}_{g,k,1}^{1/2} (\mathring{\btheta}_{g,k} -\btheta_{g}^{\star} )\Big\|_2 \Big\| \mathring{\bSigma}_{g,k,1}^{-1/2} \bphi_{U_{g,k,h+1}}(s_{k,h}, \pi_{g,h}^k(s_{k,h}))\Big\|_2 + [p_g U_{g,k,h+1}](s_{k,h}, \pi_{g,h}^k(s_{k,h})) - U_{g,k,h+1}(s_{k,h+1}) \\
    &\leq \beta_k \Big\|\mathring{\bSigma}_{g,k,1}^{-1/2} \bphi_{U_{g,k,h+1}}(s_{k,h}, \pi_{g,h}^k(s_{k,h}))\Big\|_2 +  [p_g U_{g,k,h+1}](s_{k,h}, \pi_{g,h}^k(s_{k,h}))  - U_{g,k,h+1}(s_{k,h+1}).
\end{align*}
Therefore we get by telescopic sum 
\begin{align*}
    \sum_{k=1}^K U_{g_k,k,1}(s_0) 
                           &= \sum_{k=1}^K \sum_{h=1}^H \big( U_{g_k,k,h}(s_{k,h}) - U_{g_k,k,h+1}(s_{k,h+1})  \big) \\
                           &\leq 2 \beta_K \underbrace{ \sum_{k=1}^K \sum_{h=1}^H  \min\Big\{ 1,
    \Big\|\bSigma_{g_k,k,1}^{-1/2}\bphi_{\vvalue_{g_k,k,h+1}}(s_{k,h}, a_{k,h})\Big\|_2 \Big\}}_{\triangleq Z_1} \\
    &\quad + 2 \beta_K \underbrace{ \sum_{k=1}^K \sum_{h=1}^H \min\Big\{ 1,
    \Big\|\mathring{\bSigma}_{g_k,k,1}^{-1/2}\bphi_{U_{g_k,k,h+1}}(s_{k,h}, a_{k,h})\Big\|_2 \Big\}}_{\triangleq Z_2} \\
    &\quad + \underbrace{ \sum_{k=1}^K \sum_{h=1}^H [p_{g_k} U_{g_k,k,h+1}](s_{k,h}, a_{k,h})  - U_{g_k,k,h+1}(s_{k,h+1})}_{\triangleq Z_3}.
\end{align*}
We bound $Z_1$ and $Z_2$ using Cauchy-Schwarz and the elliptical potential lemma from linear bandits \citep[][Lemma 11]{abbasi2011improved}, see e.g., \citet[][Proof of Lemma A.3]{zhang2021reward}. This yields
\begin{align*}
    Z_1 &\leq \sqrt{K H} \sqrt{\sum_{k=1}^K \sum_{h=1}^H  \min\Big\{ 1, \Big\|\bSigma_{g_k,k,1}^{-1/2}\bphi_{\vvalue_{g_k,k,h+1}}(s_{k,h}, a_{k,h})\Big\|_2^2   \Big\} } \\
    &\leq \sqrt{2} \sqrt{K H} \sqrt{\sum_{k=1}^K \sum_{h=1}^H  \min\Big\{ 1, \Big\|\bSigma_{g_k,k,h}^{-1/2}\bphi_{\vvalue_{g_k,k,h+1}}(s_{k,h}, a_{k,h})\Big\|_2^2   \Big\} } + 2 H d \log(1 + k H^3 / \lambda) \\
    &\leq \sqrt{2 K H d \log(1 + K H^3 / (d \lambda)} + 2 H d \log(1 + k H^3 / \lambda),
\end{align*}
and likewise for $Z_2$. The term $Z_3$ can be bounded by the Azuma-Hoeffding inequality since its summands form a martingale difference sequence, thus with probability at least $1-\delta/3$, it holds that $Z_3 \leq H \sqrt{2 H K \log(3 / \delta)}$.

Putting everything together and using that $\beta_K = \wt{O}( H \sqrt{d})$, we obtain 
\begin{align*}
     \sum_{k=1}^K U_{g_k,k,1}(s_0) = \wt{O}\left( d H^{3/2} \sqrt{K} + H^2 d^{3/2} \right).
\end{align*}

\paragraph{$\triangleright$~\textit{Key step \ding{174}: Bounding the sample complexity.}} We follow the reasoning given in Section~\ref{sect_theory}. By construction of the stopping rule \eqref{stopping_rule}, the algorithm terminates at an episode $\kappa$ that verifies
\begin{align*}
    \epsilon \cdot (\kappa -1) \leq \sum_{k=1}^{\kappa-1} \E_k(g_k) = \frac{8 \epsilon}{9} \cdot (\kappa -1) + \wt{O}\left( d H^{3/2} \sqrt{\kappa} + H^2 d^{3/2} \right).
\end{align*}
Solving this functional inequality in $\kappa$ yields 
\begin{align*}
    \kappa = \wt{O}\left( \frac{H^3 d^2}{\epsilon^2} + \frac{H^2 d^{3/2}}{\epsilon} \right).
\end{align*}
Using that the sample complexity is bounded by $\kappa (H+1)$ and that $H = \wt{O}(L)$, we conclude the proof.

\paragraph{$\triangleright$~\textit{Key step \ding{175}: Connecting to the original \MGE objective.}} The proof of this step is identical to the one of \ALGO in Appendix~\ref{proof_keystep4}.

\newpage

\begin{figure*}[t!]
\centering
\includegraphics[width=0.3\linewidth,trim
=20 20 20 5,clip]{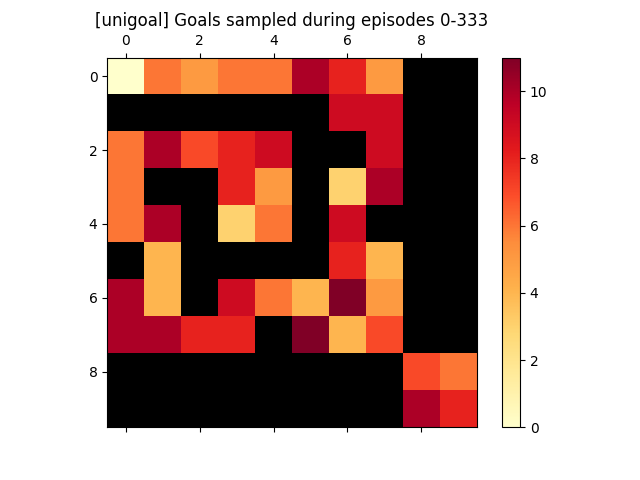} \hspace{-0.15in}
\includegraphics[width=0.3\linewidth,trim =20 20 20 5,clip]{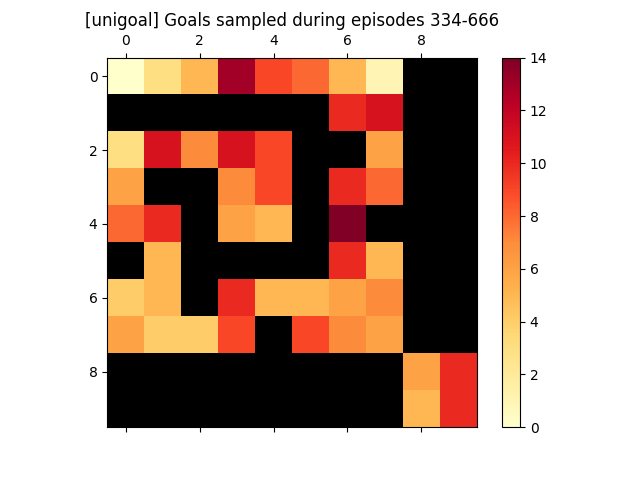} \hspace{-0.15in}
\includegraphics[width=0.3\linewidth,trim =20 20 20 5,clip]{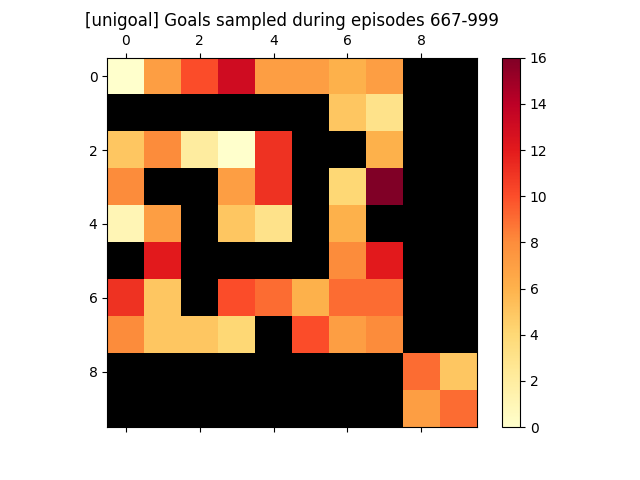} \\
\vspace{0.03in}
\includegraphics[width=0.3\linewidth,trim =20 20 20 5,clip]{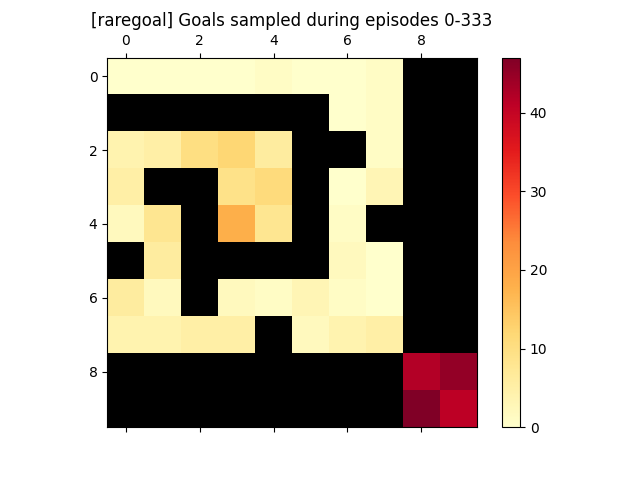} \hspace{-0.15in}
\includegraphics[width=0.3\linewidth,trim =20 20 20 5,clip]{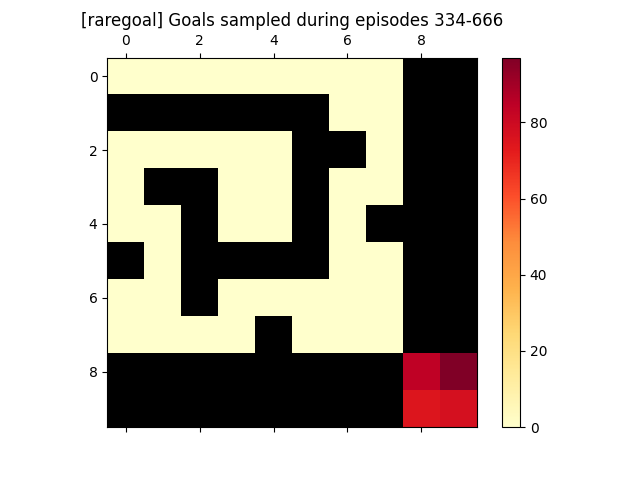} \hspace{-0.15in}
\includegraphics[width=0.3\linewidth,trim =20 20 20 5,clip]{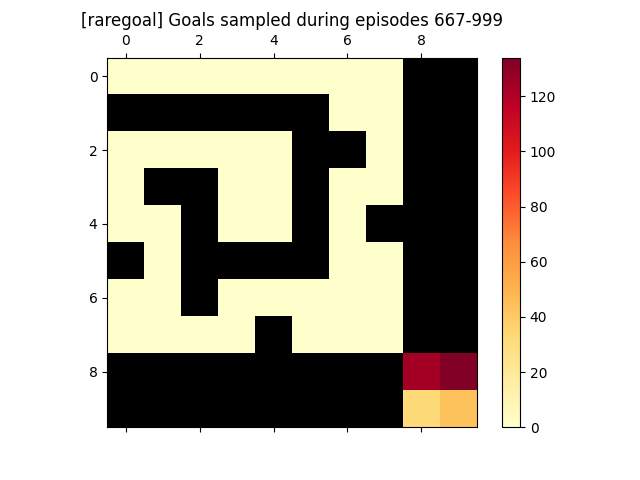} \\
\vspace{0.03in}
\includegraphics[width=0.3\linewidth,trim =20 20 20 5,clip]{fig/adagoal_timestamp0.png} \hspace{-0.15in}
\includegraphics[width=0.3\linewidth,trim =20 20 20 5,clip]{fig/adagoal_timestamp1.png} \hspace{-0.15in}
\includegraphics[width=0.3\linewidth,trim =20 20 20 5,clip]{fig/adagoal_timestamp2.png} \\

\vspace{-0.08in}
\caption{Goal sampling frequency of \UniGoalUCBVIfns \textit{(top row)}, \RareGoalUCBVIfns \textit{(middle row)} and \ALGOfns \textit{(bottom row)} over $1000$ episodes, split over episodes $0-333$ \textit{(left column)}, episodes $334-666$ \textit{(middle column)} and episodes $667-999$ \textit{(right column)}. Episodes are of length $H=50$, the environment is a grid-world with $S=52$ states, starting state $s_0 = (0, 0)$ (i.e., the top left state), $A=5$ actions (the $4$ cardinal actions plus $\areset$). The black walls act as reflectors, i.e., if the action leads against the wall, the agent stays in the current position with probability~$1$. An action fails with probability $p_f = 0.1$, in which case the agent follows (uniformly) one of the other directions. The $4$ states of the bottom right room can only be accessed from $s_0$ by any cardinal action with probability $\eta = 0.001$, thus they are extremely hard to reliably reach as their associated $V^{\star}(s_0 \rightarrow \cdot)$ is very large (scaling with $\eta^{-1}$). We select $L=40$ for \AdaGoalfns, and $\alpha=0.1$ for \RareGoalfns. For the three methods we follow the practice of \citet[][Section 4]{menard2021ucb} and use their proposed simplified form for the exploration bonuses. The experiment is based on the rlberry framework \citep{domingues2021rlberry}.}
\label{fig:gridworld_goals_sampled}
\end{figure*}

\section{ABLATION OF THE GOAL SELECTION SCHEME \& PROOF OF CONCEPT EXPERIMENT} 
\label{app_exp}

In this section, we single out the role of the adaptive goal selection scheme of \AdaGoal, i.e., step \textcircled{i} in Algorithm~\ref{algo}. For simplicity we focus on the tabular case and consider that $\cG = \cS$. Keeping the remainder of the \ALGO algorithm fixed, we compare it to two other ad-hoc goal sampling alternatives:
\begin{itemize}[leftmargin=.2in,topsep=-1.5pt,itemsep=1pt,partopsep=0pt, parsep=0pt]
    \item \UniGoal: the goal state is sampled uniformly in $\cS \setminus \{s_0\}$, i.e., with probability
    \begin{align*}
    p^{\uni}(g) \triangleq \big(S - 1\big)^{-1};
     \end{align*}
    \item \RareGoal: the goal state is sampled proportionally to its rarity, i.e., with probability
    \begin{align*}
        p^{\rare}_{\alpha}(g) \triangleq \frac{(n^k_{\alpha}(g))^{-1}}{\sum_{s \in \cS \setminus \{s_0\}} (n^k_{\alpha}(s))^{-1}},
    \end{align*}
    where $n^k(s) \triangleq \sum_{t=1}^{k H} \mathds{1}[s_t = s]$ denotes the number of times state $s$ was visited in the first $k$ episodes, and $n^k_{\alpha}(s) \triangleq \max\{n^k(s), \alpha\}$ for $\alpha \in (0, 1]$.
\end{itemize}
We can find equivalents of these two goal selection schemes in existing goal-conditioned deep RL methods. The case of a uniform goal sampling distribution prescribed by the environment (i.e., \UniGoal) is the most common, see e.g., \citet{schaul2015universal, andrychowicz2017hindsight}. Meanwhile, the goal sampling scheme of Skew-Fit \citep{pong2020skew}, a recent state-of-the-art algorithm for deep GC-RL, gives higher sampling weight to rarer goal states, where rarity is measured by a learned generative model. In the tabular case, a goal state's rarity can be characterized by the inverse of its visitation count, which corresponds to \RareGoal.

On the one hand, it is straightforward to show that \UniGoal achieves a sample complexity of at most $\wt{O}(L^3 S^2 A \epsilon^{-2})$. Intuitively, it pays for an extra $S$ since it may sample goals that are too easy or too hard, in either case they are not very useful for the agent to improve its learning (and there may be in the worst case $S-2$ of such non-informative goal states). 

On the other hand, we can see that by design \RareGoal relies on the communicating assumption and may require $\text{poly}(S, A, D, \epsilon^{-1})$ samples to learn an $\epsilon$-optimal goal-conditioned policy on $\SL$. Here the dependence on the diameter $D$ is somewhat problematic. Indeed, imagine there exist a set of states $\mathcal{S}_{\text{hard}}$ such that $1 \ll V^{\star}(s_0 \rightarrow s) \leq D$ for $s \in \mathcal{S}_{\text{hard}}$ (i.e., very hard to reach states, e.g., by chance due to environment stochasticity). Then throughout the learning process, \RareGoal will strive to reach the states in $\mathcal{S}_{\text{hard}}$ and select them as goals, which leads to unsuccessful episodes and a possible waste of samples.
Consequently, when goals have varying reachability (e.g., if the environment is highly stochastic), \RareGoal suffers from an issue of \textit{goal prioritizing}, i.e., too-hard-to-reach states are given too much goal sampling importance. 

Finally, we empirically complement our discussion above on the sequence of goals selected by \UniGoal, \RareGoal and \AdaGoal. We design a simple two-room grid-world with a very small probability of reaching the second room, and illustrate in Figure \ref{fig:gridworld_goals_sampled} the goal sampling frequency of \UniGoalUCBVI, \RareGoalUCBVI and \ALGO. We see that over the course of the learning interaction, as opposed to the designs of \RareGoal and \UniGoal, our \AdaGoal strategy is able to successfully discard the states from the bottom right room, which have a negligible probability of being reached. In addition, \AdaGoal is able to target as goals the states in the first room that are furthest away from $s_0$, i.e., those at the center of the ``spiral'', which effectively correspond to the fringe of what the agent can reliably reach. 
\captionsetup[figure]{labelfont=small,font=small}

\section{IMPLEMENTATION DETAILS OF SECTION \ref{section_AdaGoal_deepRL}} \label{app_deepRL_exp}

Here we provide the implementation details of our experiments reported in Section~\ref{section_AdaGoal_deepRL}. Our implementation and hyperparameters of \HER \citep{andrychowicz2017hindsight} are based on the PyTorch open-source codebase of \href{https://github.com/TianhongDai/hindsight-experience-replay}{https://github.com/TianhongDai/hindsight-experience-replay}, which follows the official implementation of \HER. As explained in Section~\ref{section_AdaGoal_deepRL}, we approximate \AdaGoal by computing the disagreement (i.e., standard deviation) of an ensemble of $J$ goal-conditioned Q-functions and selecting a goal proportionally to it among $N$ uniformly sampled goals. Then the policy is conditioned on this goal and executed for an episode of length $H=50$. This recovers the \VDS algorithm of \citet{zhang2020automatic}. We thus follow the implementation details given in the latter paper. In particular, the value ensemble (for goal selection) is treated as a separate module from the policy optimization (for goal-conditioned policy execution). In each training epoch, each Q-function in the ensemble performs Bellman updates with independently sampled mini-batches, and the policy is updated with DDPG \citep{lillicrap2015continuous}. Each Q-function in the ensemble is trained with its target network, with learning rate 1e-3, polyak coefficient 0.95, buffer size 1e6, and batch size 1000. Finally, we set $J=3$ and $N=1000$.

\end{document}